\documentclass{article}
\pdfoutput=1
\usepackage[final]{neurips_2021}

\usepackage[T1]{fontenc}    % use 8-bit T1 fonts
\usepackage{url}            % simple URL typesetting
\usepackage{booktabs}       % professional-quality tables
\usepackage{amsfonts}       % blackboard math symbols
\usepackage{nicefrac}       % compact symbols for 1/2, etc.
\usepackage{microtype}      % microtypography
\usepackage{xcolor}         % colors
\usepackage[pdftex]{graphicx}
\usepackage{algorithm, algorithmic}
\usepackage{caption}
\usepackage{enumitem}
\usepackage{amsmath}
\usepackage{amssymb}
\usepackage{dsfont}
\usepackage{subfigure}
\usepackage{wrapfig}
\usepackage{amsfonts}
\usepackage{booktabs}
\usepackage{mathrsfs}
\usepackage{bbm}
\usepackage{sidecap, caption}
\usepackage{hyperref}
\usepackage{bm}
\usepackage{bbding}
\usepackage{pifont}
\usepackage[normalem]{ulem}
\usepackage{xcolor}
\usepackage{pdfpages}
\usepackage{multirow}
\definecolor{svmdotblue}{RGB}{107,109,251}
\definecolor{svmdotgreen}{RGB}{0, 191, 191}
\definecolor{curveori}{RGB}{114,0,218}
\definecolor{curvemixup}{RGB}{0,176,240}
\definecolor{curvefactor}{RGB}{58,197,105}
\definecolor{curveunimix}{RGB}{255,0,0}
\useunder{\uline}{\ul}{}

\sidecaptionvpos{figure}{c}
\newtheorem{definition}{Definition}
\newtheorem{corollary}{Corollary}
\newtheorem{proof}{Proof}[section]
\newtheorem{basicsetting}{Setting}
\newtheorem{thm}{\bf Theorem}[section]
\newtheorem{lemma}{\bf Lemma}[section]
\newcommand{\cmark}{\ding{51}}%
\newcommand{\xmark}{\ding{55}}%
\newcommand{\topcaption}{%
    \setlength{\abovecaptionskip}{0pt}%
    \setlength{\belowcaptionskip}{0pt}%
    \caption}

\title{Towards Calibrated Model for Long-Tailed Visual Recognition from Prior Perspective}

\author{%
  Zhengzhuo Xu$^{1*}$, Zenghao Chai$^{1}$\thanks{Equal contribution.},~~Chun Yuan$^{1,2}$\thanks{Corresponding author.}\\
  $^{1}$Shenzhen International Graduate School, Tsinghua University\\
  $^{2}$Peng Cheng Laboratory\\
  \texttt{xzz20@mails.tsinghua.edu.cn},~~\texttt{zenghaochai@gmail.com},\\
  \texttt{yuanc@sz.tsinghua.edu.cn} \\
 }

\begin{document}
\maketitle
\begin{abstract}
    Real-world data universally confronts a severe class-imbalance problem and exhibits a \textit{long-tailed} distribution, i.e., most labels are associated with limited instances. The na\"ive models supervised by such datasets would prefer dominant labels, encounter a serious generalization challenge and become poorly calibrated. We propose two novel methods from the \textit{prior} perspective to alleviate this dilemma. First, we deduce a balance-oriented data augmentation named Uniform Mixup (UniMix) to promote \textit{mixup} in long-tailed scenarios, which adopts advanced mixing factor and sampler in favor of the minority. Second, motivated by the Bayesian theory, we figure out the Bayes Bias (Bayias), an inherent bias caused by the inconsistency of \textit{prior}, and compensate it as a modification on standard cross-entropy loss. We further prove that both the proposed methods ensure the classification \textit{calibration} theoretically and empirically. Extensive experiments verify that our strategies contribute to a better-calibrated model and their combination achieves state-of-the-art performance on CIFAR-LT, ImageNet-LT, and iNaturalist 2018.
\end{abstract}
\section{Introduction}\label{Sec.Introduction}

Balanced and large-scaled datasets \cite{Ijcv/ImageNet, Eccv/COCO} have promoted deep neural networks to achieve remarkable success in many visual tasks \cite{Cvpr/Bag-tricks-cls, Nips/Faster-RCNN, Iccv/Mask-RCNN}. However, real-world data typically exhibits a \textit{long-tailed} (LT) distribution \cite{Ijcv/Long-tailed-dataset, Cvpr/OLTR, Cvpr/INaturalist, Cvpr/LVIS}, and collecting a minority category (\textbf{tail}) sample always leads to more occurrences of common classes (\textbf{head}) \cite{Nips/Causal-LT, Corr/DevilinTail}, resulting in most labels associated with limited instances. The paucity of samples may cause insufficient feature learning on the tail classes \cite{Cvpr/BBN,Cvpr/CB,Cvpr/M2m,Cvpr/OLTR}, and such data imbalance will bias the model towards dominant labels \cite{Cvpr/EQL,Nips/Causal-LT,Cvpr/GroupSoftmax}. Hence, the generalization of minority categories is an enormous challenge.

The intuitive approaches such as directly over-sampling the tail \cite{Jair/SMOTE, Nn/OverOrUnderSample, Iccv/Gene-SMOTE, Eccv/Over-samlping1, Icml/Over-sampling2} or under-sampling the head \cite{Icic/Board_SMOTE, Nn/OverOrUnderSample, Tkde/Under-sampling1} will cause serious robustness problems. \textit{mixup} \cite{Iclr/mixup} and its extensions \cite{Icml/Manifold_Mixup, Iccv/CutMix, Eccv/remix} are effective feature improvement methods and contribute to a well-calibrated model in balanced datasets \cite{Nips/On_Mixup_Training, Corr/CalibrationForMixup}, i.e., \textit{the predicted confidence indicates actual accuracy likelihood} \cite{Icml/Calibration-NN, Nips/On_Mixup_Training}. However, \textit{mixup} is inadequately calibrated in an imbalanced LT scenario (Fig.\ref{Fig.confacc}). In this paper, we raise a conception called \textit{$\xi$-Aug} to analyze \textit{mixup} and figure out that it tends to generate more head-head pairs, resulting in unsatisfactory generalization of the tail. Therefore, we propose Uniform Mixup (UniMix), which adopts a tail-favored \textit{mixing factor} related to label \textit{prior} and a \textit{inverse sampling strategy} to encourage more head-tail pairs occurrence for better generalization and \textit{calibration}.

Previous works adjust the logits \textit{weight} \cite{Tnn/CSCE, Cvpr/CB, Cvpr/EQL, Corr/CDT} or \textit{margin} \cite{Nips/LDAM, Aaai/logit_adjustment} on standard \textit{Softmax} cross-entropy (CE) loss to tackle the bias towards dominant labels. We analyze the inconstancy of label \textit{prior}, which varies in LT train set and balanced test set, and pinpoint an inherent bias named Bayes Bias (Bayias). Based on the Bayesian theory, the \textit{posterior} is proportional to \textit{prior} times \textit{likelihood}. Hence, it's necessary to adjust the \textit{posterior} on train set by compensating different \textit{prior} for each class, which can serve as an additional \textit{margin} on CE. We further demonstrate that the Bayias-compensated CE ensures classification \textit{calibration} and propose a unified learning manner to combine Bayias with UniMix towards a better-calibrated model (see in Fig.\ref{Fig.confacc}). Furthermore, we suggest that bad calibrated approaches are counterproductive with each other, which provides a heuristic way to analyze the combined results of different feature improvement and loss modification methods (see in Tab.\ref{Tab.Ablation}).

In summary, our contributions are: 1) We raise the concept of \textit{$\xi$-Aug} to theoretically explain the reason of \textit{mixup}'s miscalibration in LT scenarios and propose Unimix (Sec.\ref{Sec.UniMix}) composed of novel mixing and sampling strategies to construct a more class-balanced virtual dataset. 2) We propose the Bayias (Sec.\ref{sec:bayias}) to compensate the bias incurred by different label \textit{prior}, which can be unified with UniMix by a training manner for better classification \textit{calibration}. 3) We conduct sufficient experiments to demonstrate that our method trains a well-calibrated model and achieves state-of-the-art results on CIFAR-10-LT, CIFAR-100-LT, ImageNet-LT, and iNaturalist 2018.

\sidecaptionvpos{figure}{c}
\begin{SCfigure*}[10][t]
% \vspace{-0.5cm}
\caption[width=0.65\textwidth]{Joint density plots of accuracy vs. confidence to measure the \textit{calibration} of classifiers on CIFAR-100-LT-100 during training. A well-calibrated classifier's density will lay around the red dot line $y=x$, indicating prediction score reflects the actual likelihood of accuracy. \textit{mixup} manages to regularize classifier on balanced datasets. However, both \textit{mixup} and its extensions tend to be overconfident in LT scenarios. Our UniMix reconstructs a more balanced dataset and Bayias-compensated CE erases \textit{prior} bias to ensure better \textit{calibration}. Without loss of accuracy, either of proposed methods trains the same classifier more calibrated and their combination achieves the best. How to measure \textit{calibration} and more visualization results are available in Appendix \ref{Apdx:caliexpadd}.}
\includegraphics[width=0.35\linewidth]{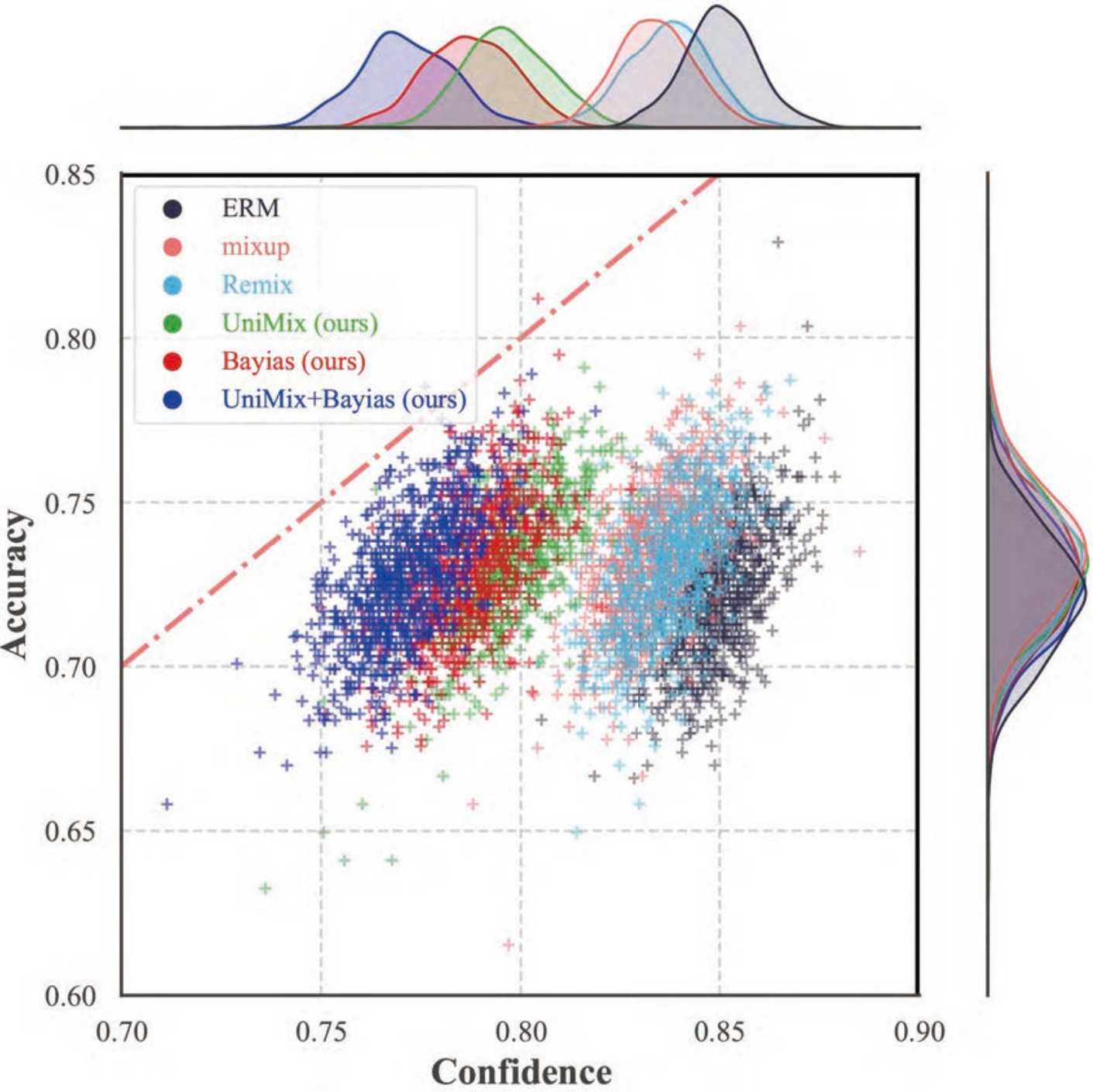}  
\label{Fig.confacc}
\vspace{-11pt}
\end{SCfigure*}

\section{Analysis of \textit{mixup}}  \label{limitmixup}

The core of supervised image classification is to find a $\theta$ parameterized mapping $\mathcal{F}_\theta:X\in \mathds{R}^{c\times h\times w} \mapsto Y \in \mathds{R}^{C\times 1}$ to estimate the empirical Dirac delta distribution $\mathds{P}_{\delta}(x,y) = \frac{1}{N}\sum\nolimits_{i=1}^N \delta(x_i,y_i)$ of $N$ instances $x\in \mathcal{X}$ and labels $y \in \mathcal{Y}$. The learning progress by minimizing Eq.\ref{Eq.ERM} is known as Empirical Risk Minimization (ERM), where $\mathcal{L}(Y=y_i,\mathcal{F}_\theta(X=x_i))$ is $x_i$'s conditional risk.
\begin{equation}
% \small
\label{Eq.ERM}
    \begin{aligned}
        R_{\delta}(\mathcal{F}_\theta) = \int_{x\in \mathcal{X}} \mathcal{L}\left(Y=y,\mathcal{F}_{\theta}(X=x)\right) d\mathds{P}_{\delta}(x,y) 
        = \frac{1}{N} \sum\nolimits_{i = 1}^N \mathcal{L}\left(Y=y_i,\mathcal{F}_{\theta}(X=x_i)\right)
        % R_{\delta}(\mathcal{F}_\theta) = \int_{x\in \mathcal{X}} \mathcal{L}\left(\mathcal{F}_{\theta}(X=x),Y=y\right) d\mathds{P}_{\delta}(x,y) 
        % = \frac{1}{N} \sum\nolimits_{i = 1}^N \mathcal{L}\left(\mathcal{F}_{\theta}(X=x_i),Y=y_i\right)
    \end{aligned}
\end{equation}
To overcome the over-fitting caused by insufficient training of $N$ samples, \textit{mixup} utilizes Eq.\ref{Eq.mixup} to extend the feature space to its vicinity based on Vicinal Risk Minimization (VRM) \cite{Nips/VRM}.
\begin{equation}
% \small
\label{Eq.mixup}
\begin{aligned}
    \widetilde{x}=\xi \cdot x_i + (1-\xi) \cdot x_j \qquad \widetilde{y}=\xi \cdot y_i + (1-\xi) \cdot y_j
\end{aligned}
\end{equation}
where $\xi \sim Beta(\alpha,\alpha),\alpha \in[0,1]$, the sample pair $(x_i,y_i),(x_j,y_j)$ is drawn from training dataset $\mathcal{D}_{train}$ randomly. Hence, Eq.\ref{Eq.mixup} converts $\mathds{P}_{\delta}(X,Y)$ into empirical \textit{vicinal distribution} $\mathds{P}_{\nu}(\widetilde{x}, \widetilde{y}) = \frac{1}{N}\sum\nolimits_{i=1}^N \nu(\widetilde{x},\widetilde{y}|x_i,y_i)$, where $\nu(\cdot)$ describes the manner of finding virtual pairs $(\widetilde{x},\widetilde{y})$ in the vicinity of arbitrary sample $(x_i,y_i)$. Then, we construct a new dataset $\mathcal{D}_{\nu}:=\{(\widetilde{x}_k, \widetilde{y}_k)\}_{k=1}^M$ via Eq.\ref{Eq.mixup} and minimize the empirical vicinal risk by Vicinal Risk Minimization (VRM):
\begin{equation}
% \small
\label{Eq.VRM}
R_{\nu}(\mathcal{F}_\theta) = \int_{\widetilde{x}\in \widetilde{\mathcal{X}}} \mathcal{L}\left(\widetilde{{Y}}=\widetilde{y},\mathcal{F}_{\theta}(\widetilde{{X}}=\widetilde{x})\right) d\mathds{P}_{\nu}(\widetilde{x},\widetilde{y})
 = \frac{1}{M} \sum\nolimits_{i = 1}^M \mathcal{L}\left(\widetilde{{Y}}=\widetilde{y}_i,\mathcal{F}_{\theta}(\widetilde{{X}}=\widetilde{x}_i)\right)
\end{equation}
\textit{mixup} is proven to be effective on balanced dataset due to its improvement of \textit{calibration} \cite{Nips/On_Mixup_Training, Icml/Calibration-NN}, but it is unsatisfactory in LT scenarios (see in Tab.\ref{Tab.Ablation}). In Fig.\ref{Fig.confacc}, \textit{mixup} fails to train a calibrated model, which surpasses baseline (ERM) a little in accuracy and seldom contributes to \textit{calibration} (far from $y=x$). To analyze the insufficiency of \textit{mixup}, the definition of \textit{$\xi$-Aug} is raised.
\begin{definition}
\label{Def.Aug}
$\xi$-Aug.\quad The virtual sample $(\widetilde{x}_{i,j},\widetilde{y}_{i,j})$ generated by Eq.\ref{Eq.mixup} with mixing factor $\xi$ is defined as a $\xi$-Aug sample, which is a robust sample of class $y_i$ $($class $y_j)$ iff $\xi \geq 0.5(\xi < 0.5)$ that contributes to class $y_i$ (class $y_j$) in model's feature learning.
\end{definition}
In LT scenarios, we reasonably assume the instance number $n$ of each class is \textit{exponential} with parameter $\lambda$ \cite{Cvpr/CB} if indices are descending sorted by $n_{y_i}$, where $y_i \in [1,C]$ and $C$ is the total class number. Generally, the imbalance factor is defined as $\rho= n_{y_1} / n_{y_C}$ to measure how skewed the LT dataset is. It is easy to draw $\lambda = \ln \rho / (C-1)$. Hence, we can describe the LT dataset as Eq.\ref{Eq.P(X)}:
\begin{equation}
% \small
\mathds{P}(Y=y_i) = \frac{{\iint_{x_i \in \mathcal{X},y_j \in \mathcal{Y}} {\mathds{1}(X = x_i,Y =y_i) dx_i dy_j}}}{{\iint_{{x_i} \in \mathcal{X},{y_j} \in \mathcal{Y}} {\mathds{1}(X = {x_i},Y={y_j})dx_i dy_j}}}= \frac{\lambda}{e^{-\lambda}-e^{-\lambda C}} {e ^{-\lambda y_i}},y_i \in [1,C]
\label{Eq.P(X)}
\end{equation}
Then, we derive the following corollary to illustrate the limitation of na\"ive \textit{mixup} strategy.
\begin{corollary}
\label{Col.1}
    When $\bm{\xi \sim Beta(\alpha,\alpha)},\alpha\in [0,1]$, the newly mixed dataset $\mathcal{D}_\nu$ composed of $\xi$-Aug samples $(\widetilde{x}_{i,j}, \widetilde{y}_{i,j})$ follows the same long-tailed distribution as the origin dataset $\mathcal{D}_{train}$, where $(x_i,y_i)$ and $(x_j,y_j)$ are \textbf{randomly} sampled from $\mathcal{D}_{train}$. (See detail derivation in Appendix \ref{Apdx:Coro1}.)
 \begin{equation}
% \small
\label{Eq.ProbMixup}
\begin{aligned}
    \mathds{P}_{\textit{mixup}}(Y^* = {y_i}) &= \mathds{P}^2(Y = {y_i}) + \mathds{P}(Y = {y_i})\iint_{y_i \ne y_j} {Beta(\alpha ,\alpha )}\mathds{P}(Y = {y_j})d\xi dy_j \\
    &= \frac{\lambda}{e^{-\lambda}-e^{-\lambda C}} {e ^{-\lambda y_i}} , y_i \in [1,C]
\end{aligned}
\end{equation}
\end{corollary}
In \textit{mixup}, the probability of any $(\widetilde{x}_{i,j},\widetilde{y}_{i,j})$ belongs to class $y_i$ or class $y_j$ is strictly determined by $\xi$ and $\mathds{E}(\xi)\equiv0.5$. Furthermore, both $(x_i,y_i)$ and $(x_j,y_j)$ are randomly sampled and concentrated on the head instead of tail, resulting in that the head classes get more \textit{$\xi$-Aug} samples than the tail ones.

\section{Methodology}

\subsection{UniMix: balance-oriented feature improvement} \label{Sec.UniMix}
\textit{mixup} and its extensions tend to generate head-majority pseudo data, which leads to the deficiency on the tail feature learning and results in a bad-calibrated model. To obtain a more balanced dataset $\mathcal{D}_\nu$, we propose the UniMix Factor $\xi^*_{i,j}$ related to the \textit{prior} probability of each category and a novel UniMix Sampler to obtain sample pairs. Our motivation is to generate comparable \textit{$\xi$-Aug} samples of each class for better generalization and \textit{calibration}.

\paragraph{UniMix Factor.} \label{subsubsection:balmixfactor}
Specifically, the \textit{prior} in imbalanced train set and balanced test set of class $y_i$ is defined as $\mathds{P}_{train}(Y=y_i)\triangleq \pi_{y_i}$, and $\mathds{P}_{test}(Y=y_i)\equiv 1/C$, respectively. We design the UniMix Factor $\xi_{i,j}^*$ for each virtual sample $\widetilde{x}_{i,j}$ instead of a fixed $\xi$ in \textit{mixup}. Consider adjusting $\xi$ with the class \textit{prior} probability $\pi_{y_i},\pi_{y_j}$. It is intuitive that a proper factor $\xi_{i,j}=\pi_{y_j}/(\pi_{y_i} + \pi_{y_j})$ ensures $\widetilde{x}_{i,j}$ to be a \textit{$\xi$-Aug} sample of class $y_j$ if $\pi_{y_i} \geq \pi_{y_j}$, i.e., class $y_i$ occupies more instances than class $y_j$.

However, $\xi_{i,j}$ is uniquely determined by $\pi_{y_i},\pi_{y_j}$. To improve the robustness and generalization, original $Beta(\alpha,\alpha)$ is adjusted to obtain UniMix Factor $\xi_{i,j}^*$. Notice that $\xi$ is close to $0$ or $1$ and symmetric at $0.5$, we transform it to maximize the probability of $\xi_{i,j} = \pi_{y_j}/(\pi_{y_i} + \pi_{y_j})$ and its vicinity. Specifically, if note $\xi \sim Beta(\alpha,\alpha)$ as $f(\xi;\alpha,\alpha)$, we define $\xi_{i,j}^* \sim \mathscr{U}(\pi_{y_i},\pi_{y_j},\alpha,\alpha)$ as:
\begin{equation}
% \small
   \xi_{i,j}^* \sim \mathscr{U}(\pi_{y_i},\pi_{y_j},\alpha,\alpha) = 
        \left\{
            \begin{aligned}
            &f(\xi_{i,j}^*-\frac{\pi_{y_j}}{\pi_{y_i}+\pi_{y_j}}+1;\alpha,\alpha), & \xi_{i,j}^* \in [0, \frac{\pi_{y_j}}{\pi_{y_i}+\pi_{y_j}}); \\
            &f(\xi_{i,j}^*-\frac{\pi_{y_j}}{\pi_{y_i}+\pi_{y_j}};\alpha,\alpha), & \xi_{i,j}^* \in [\frac{\pi_{y_j}}{\pi_{y_i}+\pi_{y_j}},1]
            \end{aligned}
        \right.
\label{Eq.UnimixFactor}
\end{equation}
Rethink Eq.\ref{Eq.mixup} with $\xi_{i,j}^*$ described as Eq.\ref{Eq.UnimixFactor}:
\begin{equation}
% \small
    \begin{aligned}
    \widetilde{x}_{i,j} = \xi_{i,j}^* \cdot x_i + (1-\xi_{i,j}^*) \cdot x_j \qquad \widetilde{y}_{i,j} = \xi_{i,j}^* \cdot y_i + (1-\xi_{i,j}^*) \cdot y_j
    \end{aligned}
\label{Eq.UniMix}
\end{equation}
We have the following corollary to show how $\xi^*_{i,j}$ ameliorates the imbalance of $\mathcal{D}_{train}$:
% With , we have the following corollary that proves to improve imbalance dataset $\mathcal{D}_{train}$:
\begin{corollary}
\label{Col.2}
When $\bm{\xi_{i,j}^* \sim\mathscr{U}(\pi_{y_i},\pi_{y_j},\alpha,\alpha)},\alpha\in [0,1]$, the newly mixed dataset $\mathcal{D}_\nu$ composed of $\xi$-Aug samples $(\widetilde{x}_{i,j}, \widetilde{y}_{i,j})$ follows a middle-majority distribution (see Fig.\ref{Fig.curve-200-100}), where $(x _i,y_i)$ and $(x_j,y_j)$ are both \textbf{randomly} sampled from $\mathcal{D}_{train}$. (See detail derivation in Appendix \ref{Apdx:Coro2}.)
\begin{equation}
% \small
\label{Eq.ProbIM}
\begin{aligned}
    \mathds{P}_{\textit{mixup}}^*(Y^* = {y_i}) &= \mathds{P}(Y = {y_i})\int_{y_j < {y_i}} {\mathds{1}\left(\int {\xi^*_{i,j} \mathscr{U}(\pi_i,\pi_j,\alpha ,\alpha )d\xi^*_{i,j} \geq 0.5 }\right) \mathds{P}(Y = {y_j})dy_j} \\
    & = \frac{\lambda }{{{{\left( {{e^{ - \lambda }} - {e^{ - \lambda C}}} \right)}^2}}}\left( {{e^{ - \lambda \left( {{y_i} + 1} \right)}} - {e^{ - 2\lambda {y_i}}}} \right) , y_i \in [1,C]
\end{aligned}
\end{equation}
\end{corollary}
\paragraph{UniMix Sampler.} \label{subsubsection:unimixsampler}
% \subsubsection{UniMix Sampler} \label{subsubsection:unimixsampler}
% , which contributes to a well-calibrated model
\begin{wrapfigure}{}{0.39\textwidth}
% \vspace{-2.5cm}
\centering
\includegraphics[width=0.39\textwidth]{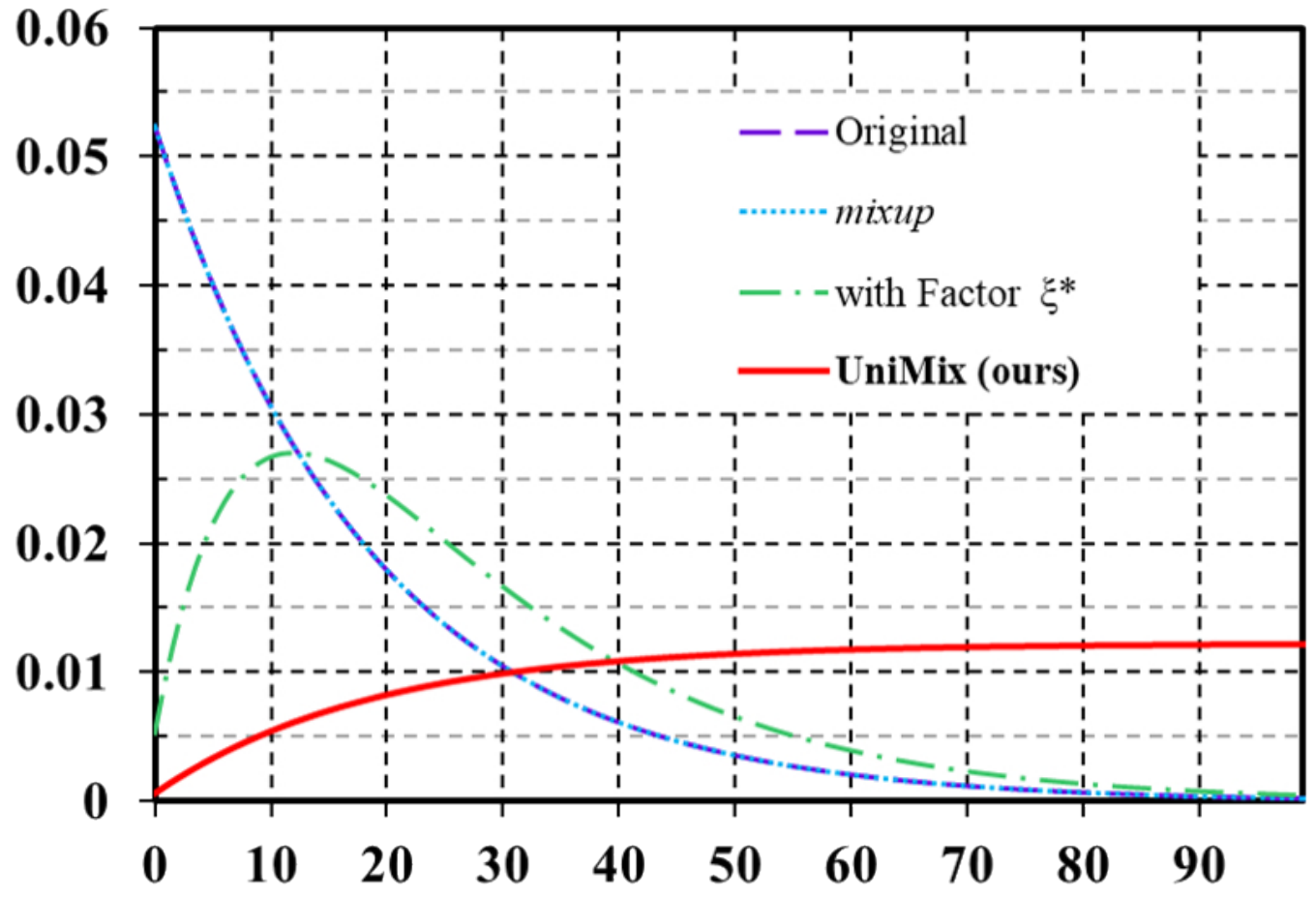}
% \caption{\footnotesize \textcolor{red}{here to do aaa aaaa aaa a sssss description dddd adfaf a sssss description dddd adfaf a sssss description dddd adfaf a sssss description dddd adfaf a sssss description dddd adfaf a sssss description dddd adfaf}}  of \textit{$\xi$-Aug} samples
\topcaption{\footnotesize Visualization of \textit{$\xi$-Aug} samples distribution ($C=100,\rho=200$) in Corollary\ref{Col.1},\ref{Col.2},\ref{Col.3}. $x$-axis: class indices. $y$-axis: probability of each class. \textit{mixup} (\textcolor{curvemixup}{blue}) exhibits the same LT distribution as origin (\textcolor{curveori}{purple}). $\xi^*$ (\textcolor{curvefactor}{green}) alleviates such situation and the full pipeline ($\tau=-1$) (\textcolor{curveunimix}{red}) constructs a more uniform distributed dataset. See more results in Appendix \ref{Apdx:addcurve}.}
\label{Fig.curve-200-100}
\vspace{-15pt}
\end{wrapfigure}

UniMix Factor facilitates \textit{$\xi$-Aug} samples more balance-distributed over all classes. However, most samples are still \textit{$\xi$-Aug} for the head or middle (see Fig.\ref{Fig.curve-200-100}(\textcolor{curvefactor}{green})). Actually, the constraint that pair $x_i, x_j$ drawn from the head and tail respectively is preferred, which dominantly generates \textit{$\xi$-Aug} samples for tail classes with $\xi_{i,j}^*$. To this end, we consider sample $x_j$ from $\mathcal{D}_{train}$ with probability inverse to the label \textit{prior}:
\begin{equation}
% \small
    {\mathds{P}_{inv}}(Y = {y_i}) = \frac{{\mathds{P}^\tau(Y = {y_i})}}{{\int_{{y_j} \in \mathcal{Y}} {\mathds{P}^\tau(Y = {y_j})dy_j} }}
    \label{Eq.reversesample} 
\end{equation}
When $\tau=1$, UniMix Sampler is equivalent to a random sampler. $\tau < 1$ indicates that $x_j$ has higher probability drawn from tail class. Note that $x_i$ is still randomly sampled from $\mathcal{D}_{train}$, i.e., it's most likely drawn from the majority class. The virtual sample $\widetilde{x}_{i,j}$ obtained in this manner is mainly a \textit{$\xi$-Aug} sample of the tail composite with $x_i$ from the head. Hence Corollary\ref{Col.3} is derived:
\begin{corollary}
\label{Col.3}
When $\bm{\xi_{i,j}^* \sim \mathscr{U}(\pi_{y_i},\pi_{y_j},\alpha,\alpha)},\alpha\in [0,1]$, the newly mixed dataset $\mathcal{D}_\nu$ composed of $\xi$-Aug samples $(\widetilde{x}_{i,j}, \widetilde{y}_{i,j})$ follows a tail-majority distribution (see Fig.\ref{Fig.curve-200-100}), where $(x _i,y_i)$ is \textbf{randomly} and $(x_j,y_j)$ is \textbf{inversely} sampled from $\mathcal{D}_{train}$, respectively. (See detail derivation in Appendix \ref{Apdx:Coro3}.)
\begin{equation}
% \small
    \label{Eq.P_ours}
    \begin{aligned}
        {\mathds{P}_{UniMix}}({Y^*} = {y_i}) &= \mathds{P}(Y = {y_i})\int_{y_j < {y_i}} { \mathds{1}\left(\int {\xi^*_{i,j} \mathscr{U}(\pi_i,\pi_j,\alpha ,\alpha )d\xi^*_{i,j} \geq 0.5 }\right) \mathds{P}_{inv}(Y=y_j)dy_j}  \\
        &= \frac{\lambda }{{\left( {{e^{ - \lambda }} - {e^{ - \lambda C}}} \right)\left( {{e^{ - \lambda \tau C}} - {e^{ - \lambda \tau}}} \right)}}\left( {{e^{ - \lambda {y_i}\left( {\tau  + 1} \right)}} - {e^{ - \lambda \left( {\tau  + {y_i}} \right)}}} \right), y_i \in [1,C]
    \end{aligned}
\end{equation}
\end{corollary}
With the proposed UniMix Factor and UniMix Sampler, we get the complete UniMix manner, which constructs a uniform \textit{$\xi$-Aug} samples distribution for VRM and greatly facilitates model's \textit{calibration} (See Fig.\ref{Fig.curve-200-100} (\textcolor{curveunimix}{red}) \& \ref{Fig.confacc}). We construct $\mathcal{D}_{\nu}:=\{(\widetilde{x}_k, \widetilde{y}_k)\}_{k=1}^M$ where $\{\widetilde{x}_k,\widetilde{y}_k\}$ is $(\widetilde{x}_{i,j},\widetilde{y}_{i,j})$ generated by $(x_i,y_i)$ and $(x_j,y_j)$. We conduct training via Eq.\ref{Eq.VRM} and the loss via VRM is available as:
\begin{equation}
% \small
\label{Eq.CalVRMLoss}
    \mathcal{L}(\widetilde{y}_k,\mathcal{F}_{\theta}(\widetilde{x}_k)) = \xi_{i,j}^*\mathcal{L}(y_i,\mathcal{F}_\theta(\widetilde{x}_{i,j}))+ (1-\xi_{i,j}^*)\mathcal{L}(y_j,\mathcal{F}_\theta(\widetilde{x}_{i,j}))
    % \mathcal{L}(f_{\theta}(\widetilde{x}_t),\widetilde{y}_t) = \xi_{i,j}^*\  \mathcal{L}(\widetilde{x}_{i,j},y_i) + (1-\xi_{i,j}^*)\  \mathcal{L}(\widetilde{x}_{i,j},y_j)
\end{equation}

\subsection{Bayias: an inherent bias in LT} \label{sec:bayias}
The bias between LT set and balanced set is ineluctable and numerous studies \cite{Cvpr/CB, Nips/Causal-LT, Nips/Rethinking-labels-LT} have demonstrated its existence. To eliminate the systematic bias that classifier tends to predict the head, we reconsider the parameters training process. Generally, a classifier can be modeled as:
\begin{equation}
% \small
\label{Eq.OrigalOptim}
    \begin{aligned}
    \hat{y}=\mathop{\arg\max}_{y_i \in \mathcal{Y}}\frac{e^{\sum\nolimits_{{d_i} \in D} [{{{({W^T})_{{y_i}}^{({d_i})}\mathcal{F}{{(x;\theta )}^{(d_i)}}] + {b_{{y_i}}}}}} }}{{\sum\nolimits_{{y_j} \in \mathcal{Y}} e^{\sum\nolimits_{{d_i} \in D} [{{{({W^T})_{{y_j}}^{({d_i})}\mathcal{F}{{(x;\theta )}^{(d_i)}}] + {b_{{y_j}}}}}} } }} \triangleq \mathop{\arg\max}_{y_i \in \mathcal{Y}}\frac{{{e^{\psi {{(x;\theta ,W,b)}_{{y_i}}}}}}}{{\sum\nolimits_{{y_j} \in \mathcal{Y}} {{e^{\psi {{(x;\theta ,W,b)}_{{y_j}}}}}} }}
    \end{aligned}
\end{equation}
where $\hat{y}$ indicts the predicted label, and $\mathcal{F}(x;\theta)\in \mathds{R}^{D\times1}$ is the $D$-dimension feature extracted by the backbone with parameter $\theta$. $W \in \mathds{R}^{D\times C}$ represents the parameter matrix of the classifier.

Previous works \cite{Cvpr/CB, Nips/Causal-LT} have demonstrated that it is not suitable for imbalance learning if one ignores such bias. In LT scenarios, the instances number in each class of the train set varies greatly, which means the corresponding \textit{prior} probability $\mathds{P}_{train}(Y=y)$ is highly skewed whereas the distribution on the test set $\mathds{P}_{test}(Y=y)$ is uniform.

% CNN based on gradient back propagation can be view as Bayesian discriminate models.

According to Bayesian theory, \textit{posterior} is proportional to \textit{prior} times \textit{likelihood}. The supervised training process of $\psi(x;\theta,W,b)$ in Eq.\ref{Eq.OrigalOptim} can regard as the estimation of \textit{likelihood}, which is equivalent to get \textit{posterior} for inference in balanced dataset. Considering the difference of \textit{prior} during training and testing, we have the following theorem (See detail derivation in Appendix \ref{Apdx:baybiasexist}):
\begin{thm}
    For classification, let $\psi(x;\theta,W,b)$ be a hypothesis class of neural networks of input $X=x$, the classification with Softmax should contain the influence of prior, i.e., the predicted label during training should be:

    \begin{equation}
    % \small
    \label{Eq.thmbias}
    \begin{aligned}
    \hat{y}=\mathop{\arg\max}_{y_i \in \mathcal{Y}}\frac{{{e^{\psi {{(x;\theta ,W,b)}_{{y_i}}}+\bm{\log(\pi_{y_i})+\log (C)}}}}}{{\sum\nolimits_{{y_j} \in \mathcal{Y}} {{e^{\psi {{(x;\theta ,W,b)}_{{y_j}}}+\bm{\log(\pi_{y_j})+\log (C)}}}}}}
    \end{aligned}
\end{equation}
\end{thm}
In balanced datasets, all classes share the same \textit{prior}. Hence, the supervised model $\psi(x;\theta,W,b)$ could use the estimated \textit{likelihood} $\mathds{P}(X=x|Y=y)$ of train set to correctly obtain \textit{posterior} $\mathds{P}(Y=y|X=x)$ in test set. However, in LT datasets where $\mathds{P}_{train}(Y=y_i)=\pi_{y_i}$ and $\mathds{P}_{test}(Y=y_i)\equiv 1/C$, \textit{prior} cannot be regard as a constant over all classes any more. Due to the difference on \textit{prior}, the learned parameters $\theta,W,b\triangleq \Theta$ will yield class-level bias, i.e., the optimization direction is no longer as described in Eq.\ref{Eq.OrigalOptim}. Thus, the bias incurred by \textit{prior} should compensate at first. To correctness the bias for inferring, the offset term that model in LT dataset to compensate is:
\begin{equation}
% \small
\label{Eq.biascompensate1}
    \mathscr{B}_y = \log(\pi_y) + \log(C)
\end{equation}
Furthermore, the proposed Bayias $\mathscr{B}_y$ enables predicted probability reflecting the actual correctness likelihood, expressed as Theorem \ref{Thm.Calibration1}. (See detail derivation in Appendix \ref{apdx:baybiascaliprove}.)
\begin{thm}
\label{Thm.Calibration1}
    $\mathscr{B}_y$-compensated cross-entropy loss in Eq.\ref{Eq.b-ce} ensures classification \textit{calibration}. 
\begin{equation}
% \small
    \label{Eq.b-ce}
    \begin{aligned}
    \mathcal{L}_{\mathscr{B}}(y_i,\psi(x;\Theta)) = \log \left[ 1 + \sum\nolimits_{y_k \neq y_i} e^{(\mathscr{B}_{y_k}-\mathscr{B}_{y_i})} \cdot e^{ \psi(x;\Theta)_{y_k} - \psi(x;\Theta)_{y_i}}\right]
    \end{aligned}
\end{equation}
\end{thm}
Here, the optimization direction during training will convert to $\psi(X;\theta,W,b)+\mathscr{B}_y$. In particular, if the train set is balanced, $\mathds{P}_{train}(Y=y) \triangleq \pi_y \equiv 1/C$, then $\mathscr{B}_y = \log(1/C) + \log(C) \equiv 0$, which means the Eq.\ref{Eq.OrigalOptim} is a balanced case of Eq.\ref{Eq.thmbias}. We further raise that $\mathscr{B}_y$ is critical to the classification \textit{calibration} in Theorem \ref{Thm.Calibration1}. The pairwise loss in Eq.\ref{Eq.b-ce} will guide model to avoid over-fitting the tail or under-fitting the head with better generalization, which contributes to a better calibrated model. 

Compared with logit adjustment \cite{Aaai/logit_adjustment}, which is also a \textit{margin} modification, it is necessary to make a clear statement about the concrete difference from three points. 1) Logit adjustment is motivated by Balanced Error Rate (BER), while the Bayias compensated CE loss is inspired by the Bayesian theorem. We focus more on the model performance on the real-world data distribution. 2) As motioned above, our loss is consistent with standard CE loss when the train set label prior is the same as real test label distribution. 3) Our loss can tackle the imbalanced test set situation as well by simply setting the margin as $\mathscr{B}_y = \log(\pi_y) + \log(\pi'_y)$, where the $\pi'_y$ represents the test label distribution. The experiment evidence can be found in Appendix Tab.\ref{Tab.testimb}.

\subsection{Towards calibrated model with UniMix and Bayias}
\setlength{\textfloatsep}{11pt}
\begin{algorithm}[ht!]
    \caption{Integrated training manner towards calibrated model.}
    \label{Alg:trainingmanner}
	\renewcommand{\algorithmicrequire}{\textbf{Input:}}
	\renewcommand{\algorithmicensure}{\textbf{Output:}}
	\begin{algorithmic}[1]
		\REQUIRE $\mathcal{D}_{train}$, Batch Size $\mathcal{N}$, Stop Steps $T_1,T_2$, Random Sampler $\mathcal{R}$, UniMix Sampler $\mathcal{R}^*$
		\ENSURE Optimized $\Theta^*$, i.e., feature extractor parameters $\theta^*$, classifier parameters $W^*,b^*$
		\STATE Initialize the parameters $\Theta^{(0)}$ randomly and calculate $\mathscr{B}_y$ via Eq.\ref{Eq.biascompensate1}
		\FOR{$t = 0$ to $T_1$} 
    		\STATE Sample a mini-batch $\mathcal{B} = \{x_i,y_i\}_{i=1}^{\mathcal{N}} \leftarrow \mathcal{R}(\mathcal{D}_{train},\mathcal{N})$ 
    		\STATE Sample a mini-batch $\mathcal{B}^* = \{x_j^*,y_j^*\}_{j=1}^{\mathcal{N}} \leftarrow \mathcal{R^*}(\mathcal{D}_{train},\mathcal{N})$
    		\STATE Calculate UniMix factor \textbf{$\xi^*$} via Eq.\ref{Eq.UnimixFactor}
    		\STATE Construct VRM dataset $\mathcal{B}_{\nu}= \{\widetilde{x}_{k},\widetilde{y}_{k}\}_{k=1}^{\mathcal{N}}$ via Eq.\ref{Eq.UniMix}
    % 		\STATE Calculate $\mathcal{L}_{\mathcal{B_{\nu}}} = \frac{1}{\mathcal{N}}\sum_{t=1}^{\mathcal{N}} \xi_{ij} \mathcal{L}(\psi(\widetilde{x}_t;\theta,W,\mathscr{B}_{y}), y_i) + (1-\xi_{ij}) \mathcal{L}(\psi(\widetilde{x}_t;\theta,W,\mathscr{B}_{y}), y_j)$
            \STATE Calculate $\mathcal{L}_{\mathcal{B_{\nu}}} = \mathds{E} [\xi^*_{i,j} \mathcal{L}_{\mathscr{B}}(y_i,\psi(\widetilde{x};\Theta^{(t)})) + (1-\xi^*_{i,j}) \mathcal{L}_{\mathscr{B}}(y^*_j,\psi(\widetilde{x};\Theta^{(t)}))]$ via Eq.\ref{Eq.CalVRMLoss},\ref{Eq.b-ce}
    		\STATE Update $\Theta^{(t+1)} \leftarrow \Theta^{(t)} - \alpha \nabla_{\Theta^{(t)}} \mathcal{L}_{\mathcal{B_{\nu}}}$
		\ENDFOR
    	\FOR{$t=T_1$ to $T_2$} 
    	    \STATE Sample a mini-batch $\mathcal{B} = \{x_i,y_i\}_{i=1}^{\mathcal{N}} \leftarrow \mathcal{R}(\mathcal{D}_{train},\mathcal{N})$ 
    	    \STATE Calculate $\mathcal{L}_{\mathcal{B}} = \mathds{E}[\mathcal{L}_{\mathscr{B}}(y_i,\psi(x_i;\Theta))]$ via Eq.\ref{Eq.b-ce}
    		\STATE Update $\Theta^{(t+1)} \leftarrow \Theta^{(t)} - \alpha \nabla_{\Theta^{(t)}} \mathcal{L}_{\mathcal{B}}$
    % 		\STATE Update $W \leftarrow W - \alpha \nabla_{W} \mathcal{L}_{\mathcal{B_{\nu}}}$, $b \leftarrow b - \alpha \nabla_{b} \mathcal{L}_{\mathcal{B_{\nu}}}$
    	\ENDFOR
	\end{algorithmic} 
\end{algorithm}

It's intuitive to integrate feature improvement methods with loss modification ones for better performance. However, we find such combinations fail in most cases and are counterproductive with each other, i.e., the combined methods reach unsatisfactory performance gains. We suspect that these methods take contradictory trade-offs and thus result in overconfidence and bad \textit{calibration}. Fortunately, the proposed UniMix and Bayias are both proven to ensure \textit{calibration}. To achieve a better-calibrated model for superior performance gains, we introduce Alg.\ref{Alg:trainingmanner} to tackle the previous dilemma and integrate our two proposed approaches to deal with poor generalization of tail classes. Specially, inspired by previous work \cite{Corr/RethinkingCleanAugmented}, we overcome the coverage difficulty in \textit{mixup} \cite{Iclr/mixup} by removing UniMix in the last several epochs and thus maintain the same epoch as baselines. Note that Bayias-compensated CE is only adopted in the training process as discussed in Sec.\ref{sec:bayias}.

\section{Experiment}
\subsection{Results on synthetic dataset}
We make an ideal binary classification using Support Vector Machine (SVM) \cite{Ac/SVM} to show the distinguish effectiveness of UniMix. Suppose there are samples from two disjoint circles respectively:
\begin{equation}
\label{Eq.fakedataset}
\begin{aligned}
    z^+ =\{(x,y)|(x-x_0)^2+(y-y_0)^2 \leq r^2\} \\
    z^-=\{(x,y)|(x+x_0)^2+(y+y_0)^2 \leq r^2\}
\end{aligned}
\end{equation}
To this end, we randomly sample $m$ discrete point pairs from $z^+$ to compose positive samples $z^+_{p}=\{(x^+_1,y^+_1),\cdots,(x^+_m,y^+_m)\}$, and $m$ negative samples $z^-_{n}=\{(x^-_1,y^+_1),\cdots,(x^-_m,y^-_m)\}$ from $z^-$ correspondingly, thus to generate a balanced dataset $\mathcal{D}_{bal}=\{z^+_{p},z^-_{n}\}$ with $\mathds{P}\left ((x,y)\in z^+_{p} \right )=\mathds{P}\left ((x,y)\in z^-_{n} \right )=0.5$. For imbalance data, we sample $n (n \ll m)$ negative data from $z^-$ to generate $z^-_{n'}=\{(x'^-_1,y'^-_1),\cdots,(x'^-_n,y'^-_n)\}$, so as to compose the imbalance dataset $\mathcal{D}_{imbal}=\{z^+_{p},z^-_{n'}\}$, with $\mathds{P}\left ((x,y)\in z^+_{p} \right ) \gg \mathds{P}\left ((x,y)\in z^-_{n'} \right )$. We train the SVM model on the two synthetic datasets, and visualize the classification boundary of each dataset in Fig.\ref{exp:SVM}.
\begin{figure}[h!]
\vspace{-11pt}
	\centering
	\subfigure[imbalanced scenario]{
		\begin{minipage}[b]{0.23\textwidth}
			\includegraphics[width=1\textwidth]{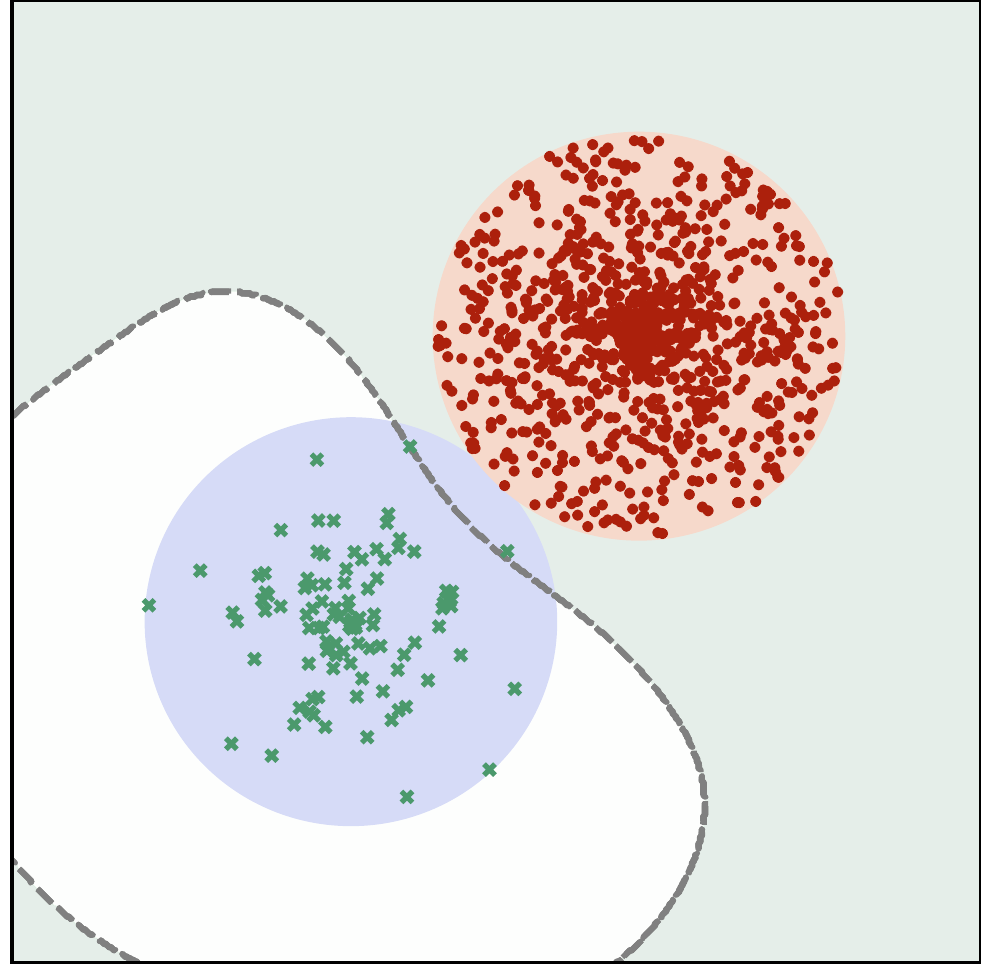} 
		\end{minipage}
		\label{exp:SVM(a)}
	}
    	\subfigure[with \textit{mixup}]{
    		\begin{minipage}[b]{0.23\textwidth}
  		 	\includegraphics[width=1\textwidth]{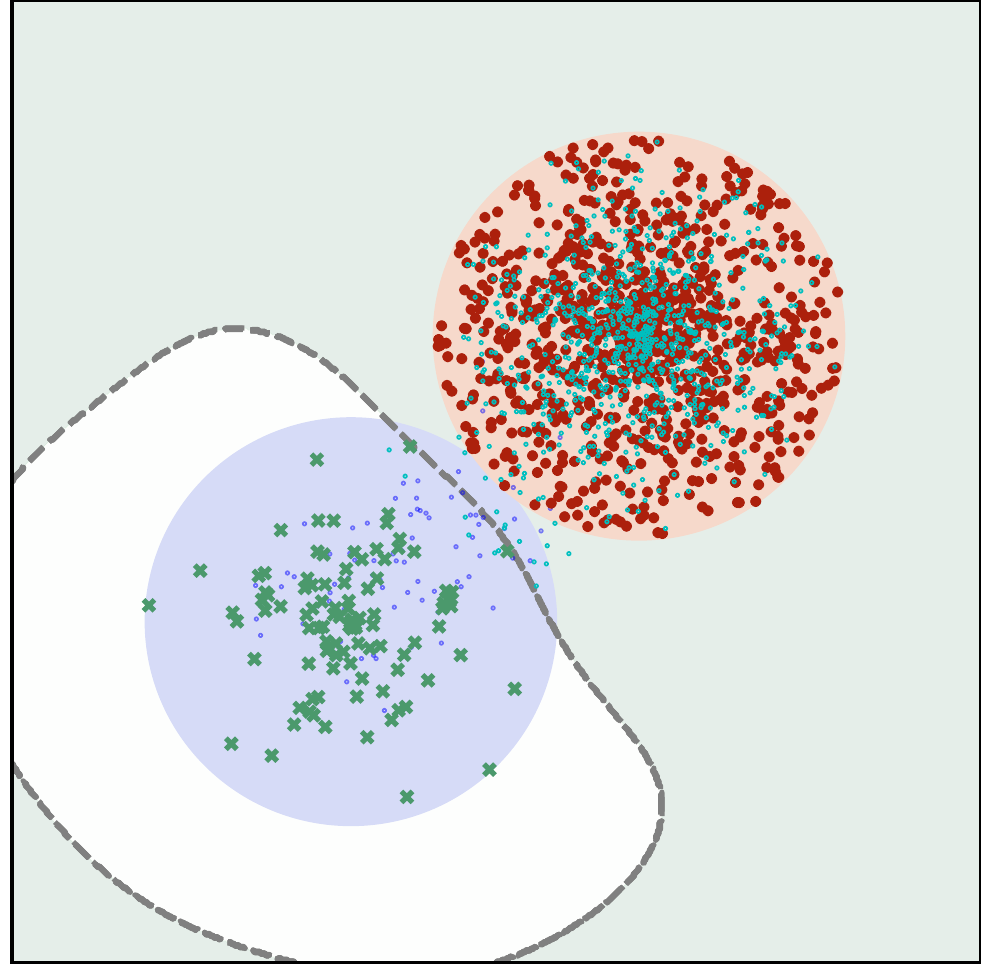}
    		\end{minipage}
		\label{exp:SVM(b)}
    	}
	\subfigure[with UniMix]{
		\begin{minipage}[b]{0.23\textwidth}
			\includegraphics[width=1\textwidth]{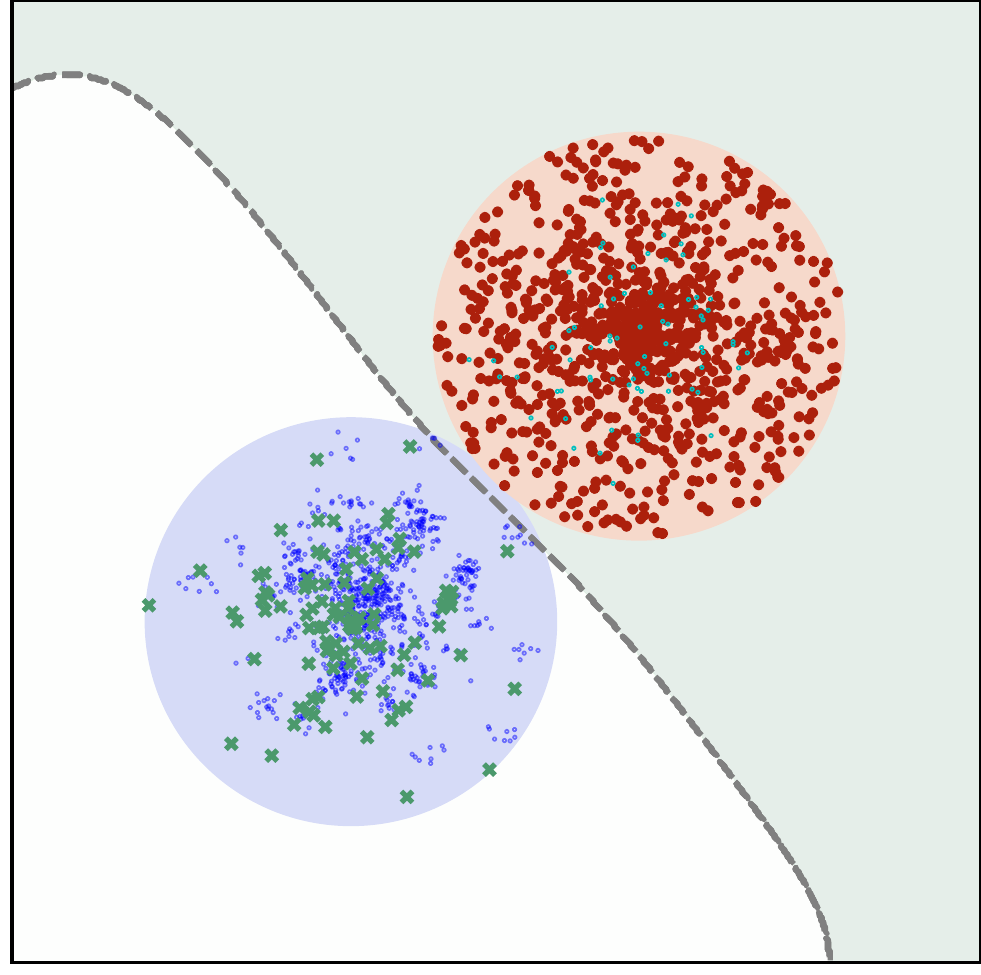} 
		\end{minipage}
		\label{exp:SVM(c)}
	}
    	\subfigure[balanced scenario]{
    		\begin{minipage}[b]{0.23\textwidth}
		 	\includegraphics[width=1\textwidth]{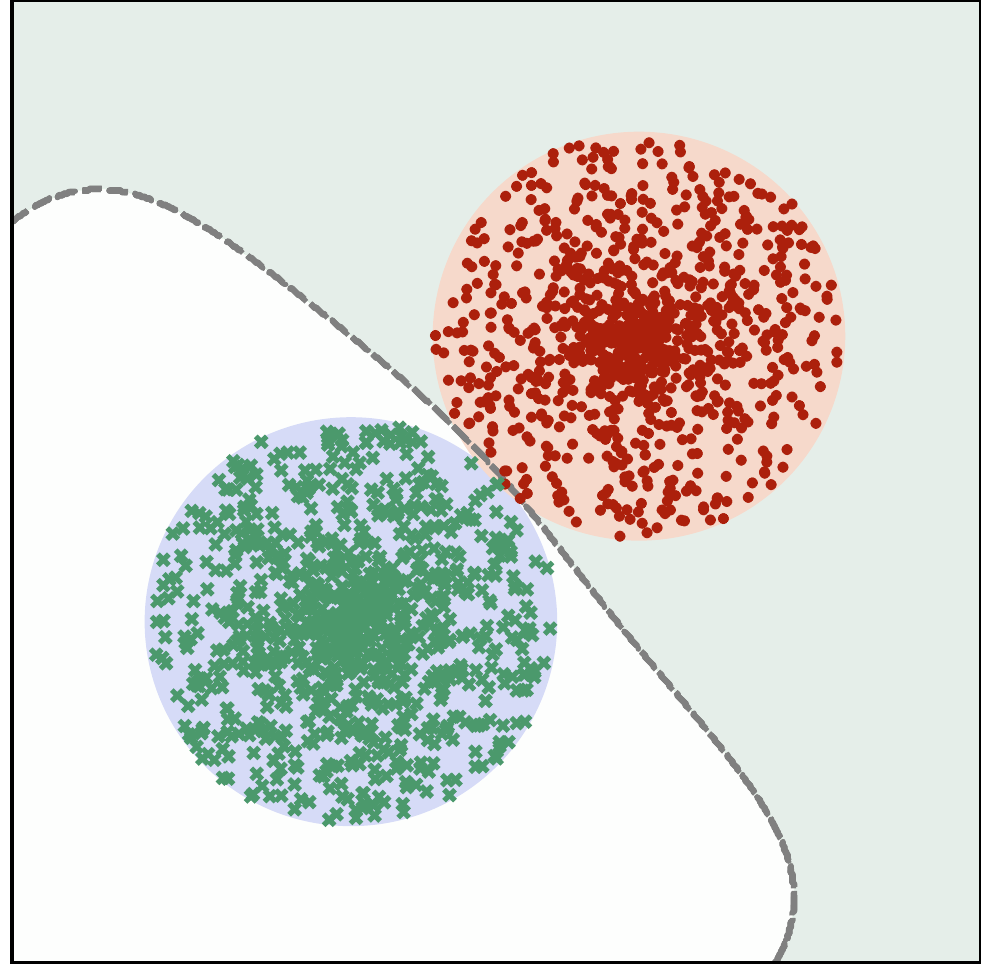}
    		\end{minipage}
		\label{exp:SVM(d)}
    	}
    % 	\vspace{-0.3cm}
	\topcaption{SVM decision boundary on the synthetic balanced dataset (Fig.\ref{exp:SVM(d)}) and imbalanced dataset (Fig.\ref{exp:SVM(a)},\ref{exp:SVM(b)},\ref{exp:SVM(c)}). The theoretical classification boundary of the synthetic dataset is $y$=$-x$. "$\circ$" represents generated pseudo data, where \textcolor{svmdotblue}{blue} and \textcolor{svmdotgreen}{green} represent belong to $z^-$ and $z^+$, respectively.}
	\label{exp:SVM}
	\vspace{-11pt}
\end{figure}

The SVM reaches an approximate ideal boundary on balanced datasets (Fig.\ref{exp:SVM(d)}) but severely deviates from the $y=-x$ in the imbalanced dataset (Fig.\ref{exp:SVM(a)}). As proven in Sec.\ref{limitmixup}, \textit{mixup} (Fig.\ref{exp:SVM(b)}) is incapable of shifting imbalance distribution, resulting in no better result than the original one (Fig.\ref{exp:SVM(a)}). After adopting the proposed UniMix, the classification boundary in Fig.\ref{exp:SVM(c)} shows much better results than the original imbalanced dataset, which gets closed to the ideal boundary.
\subsection{Results on CIFAR-LT}
The imbalanced datasets CIFAR-10-LT and CIFAR-100-LT are constructed via suitably discarding training samples following previous works \cite{Cvpr/BBN, Cvpr/CB, Aaai/logit_adjustment, Nips/LDAM}. The instance numbers exponentially decay per class in train dataset and keep balanced during inference. We extensively adopt $\rho \in \{10,50,100,200\}$ for comprehensive comparisons. See implementation details in Appendix \ref{Apdx:cifardetail}.

\textbf{Comparison methods.} We evaluate the proposed method against various representative and effective approaches extensively, summarized into the following groups: 
\textbf{a) Baseline.} We conduct plain training with CE loss called ERM as baseline. \textbf{b) Feature improvement methods} modify the input feature to cope with LT datasets. \textit{mixup} \cite{Iclr/mixup} convexly combines images and labels to build virtual data for VRM. Manifold mixup \cite{Icml/Manifold_Mixup} performs the linear combination in latent states. Remix \cite{Eccv/remix} conducts the same combination on images and adopts tail-favored rules on labels. M2m \cite{Cvpr/M2m} converts majority images to minority ones by adding noise perturbation, which need an additional pre-trained classifier. BBN \cite{Cvpr/BBN} utilizes features from two branches in a cumulative learning manner. \textbf{c) Loss modification methods} either adjust the logits \textit{weight} or \textit{margin} before the \textit{Softmax} operation. Specifically, focal loss \cite{Iccv/Focal}, CB \cite{Cvpr/CB} and CDT \cite{Corr/CDT} re-weight the logits with elaborate strategies, while LDAM \cite{Nips/LDAM} and Logit Adjustment \cite{Aaai/logit_adjustment} add the logits \textit{margin} to shift decision boundary away from tail classes. \textbf{d) Other methods.} We also compare the proposed method with other two-stage approaches (e.g. DRW \cite{Nips/LDAM}) for comprehensive comparisons.

\begin{table*}[h!]
% \vspace{-11pt}
\topcaption{Top-1 validation accuracy(\%) of ResNet-32 on CIFAR-10/100-LT. E2E: end to end training. Underscore: the best performance in each group. $\dagger$: our reproduced results. $\ddagger$: reported results in \cite{Cvpr/BBN}. $\star$: reported results in \cite{Corr/CDT}. Our \textit{calibration} ensured method achieves the best performance.}
\resizebox{1\textwidth}{!}{%
\centering
% \tiny
\begin{tabular}{l|c|cccc|cccc}
\toprule
\multicolumn{1}{c|}{Dataset} & E2E & \multicolumn{4}{c|}{CIFAR-10-LT}                                  & \multicolumn{4}{c}{CIFAR-100-LT}                                  \\ \midrule
\multicolumn{1}{c|}{$\rho$ (easy $\rightarrow$ hard)}        & -   & 10             & 50             & 100            & 200            & 10             & 50             & 100            & 200            \\ \midrule
ERM$^\dagger$                                             & \cmark   & 86.39          & 74.94          & 70.36          & 66.21          & 55.70       & 44.02          & 38.32          & 34.56          \\ \midrule
\textit{mixup}$^\ddagger$  \cite{Iclr/mixup}                & \cmark  & 87.10          & 77.82          & 73.06          & {\ul 67.73}    & 58.02          & 44.99          & 39.54          & 34.97          \\
Manifold Mixup$^\ddagger$  \cite{Icml/Manifold_Mixup}       & \cmark   & 87.03          & 77.95          & 72.96          & -              & 56.55          & 43.09          & 38.25          & -              \\
Remix          \cite{Eccv/remix}                & \cmark   & 88.15          & 79.20          & 75.36          & 67.08          & {\ul 59.36}    & 46.21          & 41.94          & {\ul 36.99}    \\
M2m            \cite{Cvpr/M2m}                  & \xmark   & 87.90          & -              & 78.30          & -              & 58.20          & -              & {\ul 42.90}    & -              \\
BBN$^\ddagger$            \cite{Cvpr/BBN}          & \xmark    & {\ul 88.32}    & {\ul 82.18}    & {\ul 79.82}    & -              & 59.12          & {\ul 47.02}    & 42.56          & -              \\ \midrule
Focal$^\star$        \cite{Iccv/Focal}           & \cmark   & 86.55          & 76.71$^\dagger$          & 70.43          & 65.85          & 55.78          & 44.32$^\dagger$          & 38.41          & 35.62          \\
Urtasun et al  \cite{Icml/L2RW}                 & \cmark   & 82.12          & 76.45          & 72.23          & 66.25          & 52.12          & 43.17          & 38.90          & 33.00          \\
CB-Focal       \cite{Cvpr/CB}                   & \cmark   & 87.10          & 79.22          & 74.57          & 68.15          & 57.99          & 45.21          & 39.60          & 36.23          \\
$\tau$-norm$^\star$   \cite{Iclr/Decouple}        & \cmark   & 87.80          & 82.78$^\dagger$          & 75.10          & 70.30          & 59.10           & 48.23$^\dagger$          & 43.60          & 39.30          \\
LDAM$^\dagger$          \cite{Nips/LDAM}          & \cmark   & 86.96          & 79.84          & 74.47          & 69.50          & 56.91          & 46.16          & 41.76          & 37.73          \\
LDAM+DRW$^\dagger$       \cite{Nips/LDAM}          & \xmark    & 88.16          & 81.27          & 77.03          & 74.74          & 58.71          & 47.97          & 42.04          & 38.45          \\
CDT$^\star$         \cite{Corr/CDT}             & \cmark   & {\ul 89.40}     & 81.97$^\dagger$          & 79.40          & 74.70          & 58.90          & 45.15$^\dagger$          & {\ul 44.30}    & 40.50          \\
Logit Adjustment   \cite{Aaai/logit_adjustment}     & \cmark   & 89.26$^\dagger$          & {\ul 83.38}$^\dagger$     & {\ul 79.91}   & {\ul 75.13}$^\dagger$    & {\ul 59.87}$^\dagger$    & {\ul 49.76}$^\dagger$     & 43.89         & {\ul 40.87}$^\dagger$    \\ \midrule
Ours                                            & \cmark   & \textbf{89.66} & \textbf{84.32} & \textbf{82.75} & \textbf{78.48} & \textbf{61.25} & \textbf{51.11} & \textbf{45.45} & \textbf{42.07} \\
\bottomrule
\end{tabular}}
\label{Tab.ResultCifar}
\vspace{-5pt}
\end{table*}

\textbf{Results.} We present results of CIFAR-10-LT and CIFAR-100-LT in Tab.\ref{Tab.ResultCifar}. Our proposed method achieves state-of-the-art results against others on each $\rho$, with performance gains improved as $\rho$ gets increased (See Appendix \ref{Apdx:visualcifar}). Specifically, our method overcomes the ignorance in tail classes effectively with better \textit{calibration}, which integrates advantages of two group approaches and thus surpass most two-stage methods (i.e., BBN, M2m, LDAM+DRW). However, not all combinations can get ideal performance gains as expected. More details will be discussed in Sec.\ref{sec:abl}.

\begin{table*}[h!]
\setlength{\tabcolsep}{10pt}
% \vspace{-11pt}
% Besides baseline (ERM) and our methods, we divide all approaches into feature-wise and loss-wise groups.
\topcaption{Quantitative \textit{calibration} metric of ResNet-32 on CIFAR-10/100-LT test set. Smaller ECE and MCE indicate better \textit{calibration} results. Either of the proposed methods achieves a well-calibrated model compared with others. The combination of UniMix and Bayias achieves the best performance.}
\resizebox{1\textwidth}{!}{%
\centering
% \tiny
\begin{tabular}{l|c|c|c|c|cccc}
\toprule
Dataset &
  \multicolumn{4}{c|}{CIFAR-10-LT} &
  \multicolumn{4}{c}{CIFAR-100-LT} \\ \midrule
\multicolumn{1}{c|}{$\rho$} &
  \multicolumn{2}{c|}{10} &
  \multicolumn{2}{c|}{100} &
  \multicolumn{2}{c|}{10} &
  \multicolumn{2}{c}{100} \\ \midrule
\textit{Calibration} Metric (\%)&
  ECE &
  MCE &
  ECE &
  MCE &
  \multicolumn{1}{c|}{ECE} &
  \multicolumn{1}{c|}{MCE} &
  \multicolumn{1}{c|}{ECE} &
  MCE\\ \midrule
ERM &
  6.60 &
  74.96 &
  20.53 &
  73.91 &
  \multicolumn{1}{c|}{22.85} &
  \multicolumn{1}{c|}{34.50} &
  \multicolumn{1}{c|}{38.23} &
  87.22 \\
\textit{mixup} \cite{Iclr/mixup} & 6.55 & 24.54 & 19.20 & 37.84 & \multicolumn{1}{c|}{19.69} & \multicolumn{1}{c|}{38.53} & \multicolumn{1}{c|}{32.72} & 50.46 \\
Remix \cite{Eccv/remix} &
  6.81 &
  22.44 &
  15.38 &
  27.99 &
  \multicolumn{1}{c|}{20.17} &
  \multicolumn{1}{c|}{32.99} &
  \multicolumn{1}{c|}{33.56} &
  50.96 \\
LDAM+DRW \cite{Nips/LDAM} &
  11.22 &
  45.92 &
  19.89 &
  49.07 &
  \multicolumn{1}{c|}{30.54} &
  \multicolumn{1}{c|}{55.57} &
  \multicolumn{1}{c|}{42.18} &
  64.78 \\ \midrule
UniMix (ours)&
  6.00 &
  25.99 &
  12.87 &
  28.30 &
  \multicolumn{1}{c|}{19.38} &
  \multicolumn{1}{c|}{33.40} &
  \multicolumn{1}{c|}{27.12} &
  41.46 \\
Bayias (ours)&
  5.52 &
  20.14 &
  11.05 &
  \textbf{23.72} &
  \multicolumn{1}{c|}{17.42} &
  \multicolumn{1}{c|}{28.26} &
  \multicolumn{1}{c|}{24.31} &
  39.66 \\
UniMix+Bayias (ours) &
  \textbf{4.74} &
  \textbf{13.67} &
  \textbf{10.19} &
  25.47 &
  \multicolumn{1}{c|}{\textbf{15.24}} &
  \multicolumn{1}{c|}{\textbf{23.67}} &
  \multicolumn{1}{c|}{\textbf{23.04}} &
  \textbf{37.36} \\ 
  \bottomrule
\end{tabular}%
}
\label{Tab.ECEResultCifar}
% \vspace{-5pt}
\end{table*}

To quantitatively describe the contribution to model \textit{calibration}, we make quantitative comparisons on CIFAR-10-LT and CIFAR-100-LT. According to the definition ECE and MCE (see Appendix Eq.\ref{Eq.defECE},\ref{Eq.defMCE}), a well-calibrated model should minimize the ECE and MCE for better generalization and robustness. In this experiment, we adopt the most representative $\rho \in \{10,100\}$ with previous mainstream state-of-the-art methods.

The results in Tab.\ref{Tab.ECEResultCifar} show that either of the proposed methods generally outperforms previous methods, and their combination enables better classification \textit{calibration} with smaller ECE and MCE. Specifically, \textit{mixup} and Remix have negligible contributions to model \textit{calibration}. As analyzed before, such methods tend to generate head-head pairs in favor of the feature learning of majority classes. However, more head-tail pairs are required for better feature representation of the tail classes. In contrast, both the proposed UniMix and Bayias pay more attention to the tail and achieve satisfactory results. It is worth mentioning that improving \textit{calibration} in post-hoc manners \cite{Icml/Calibration-NN, DBLP:conf/cvpr/MiSLAS} is also effective, and we will discuss it in Appendix \ref{Apdx.calibrationvis}. Note that LDAM is even worse calibrated compared with baseline. We suggest that LDAM adopts an additional margin only for the ground-truth label from the angular perspective, which shifts the decision boundary away from the tail class and makes the tail predicting score tend to be larger. Additionally, LDAM requires the normalization of input features and classifier weight matrix. Although a scale factor is proposed to enlarge the logits for better \textit{Softmax} operation \cite{Mm/norm_face}, it is still harmful to \textit{calibration}. It also accounts for its contradiction with other methods. Miscalibration methods combined will make models become even more overconfident and damage the generalization and robustness severely.

\subsection{Results on large-scale datasets}
We further verify the proposed method's effectiveness quantitatively on large-scale imbalanced datasets, i.e. ImageNet-LT and iNaturalist 2018. ImageNet-LT is the LT version of ImageNet \cite{Ijcv/ImageNet} by sampling a subset following \textit{Pareto} distribution, which contains about $115K$ images from $1,000$ classes. The number of images per class varies from $5$ to $1,280$ exponentially, i.e., $\rho=256$. In our experiment, we utilize the balanced validation set constructed by Cui \textit{et al.} \cite{Cvpr/CB} for fair comparisons. The iNaturalist species classification dataset \cite{Cvpr/INaturalist} is a large-scale real-world dataset which suffers from extremely label LT distribution and fine-grained problems \cite{Cvpr/INaturalist, Cvpr/BBN}. It is composed of $435,713$ images over $8,142$ classes with $\rho=500$. The official splits of train and validation images \cite{Nips/LDAM, Cvpr/BBN, Iclr/Decouple} are adopted for fair comparisons. See implementation details in Appendix \ref{Apdx:largedetail}.

\begin{table*}[h!]
\vspace{-5pt}
\setlength{\tabcolsep}{10pt}
\centering
\topcaption{Top-1 validation accuracy(\%) of ResNet-10/50 on ImageNet-LT and ResNet-50 on iNaturalist 2018. E2E: end to end training. $\dagger$: our reproduced results. $\ddagger$: results reported in origin paper.}
%  Our method outperforms others in challenging large-scale datasets. 
% \tiny
\resizebox{1\textwidth}{!}{%
\begin{tabular}{l|l|cccc|cc}
\toprule
Dataset & \multicolumn{5}{c|}{ImageNet-LT}                                                        & \multicolumn{2}{c}{iNaturalist 2018} \\ \midrule
Method                          & E2E                           & ResNet-10       & $\Delta$       & ResNet-50                &$\Delta$        & ResNet-50        & $\Delta$         \\ \midrule
CE$^\dagger$                    & \multicolumn{1}{c|}{\cmark}   & 35.88           & -              & 38.88                    & -              & 60.88            & -                \\
CB-CE$^\dagger$ \cite{Cvpr/CB}  & \multicolumn{1}{c|}{\cmark}   & 37.06           & +1.18          & 40.85                    & +1.97          & 63.50            & +2.62            \\
LDAM \cite{Nips/LDAM}           & \multicolumn{1}{c|}{\cmark}   & 36.05$^\dagger$ & +0.17          & 41.86$^\dagger$          & +2.98          & 64.58$^\ddagger$ & +3.70            \\ \midrule
OLTR$^\ddagger$ \cite{Cvpr/OLTR}& \multicolumn{1}{c|}{\xmark}   & 35.60           & -0.28          & 40.36                    & +1.48          & 63.90            & +3.02            \\
LDAM+DRW \cite{Nips/LDAM}       & \multicolumn{1}{c|}{\xmark}   & 38.22$^\dagger$ & +2.34          & 45.75$^\dagger$          & +6.87          & 68.00$^\ddagger$ & +7.12            \\
BBN$^\ddagger$ \cite{Cvpr/BBN}  & \multicolumn{1}{c|}{\xmark}   & \multicolumn{2}{c}{-}                         & \multicolumn{2}{c|}{-}       & 66.29            & +5.41            \\
c-RT \cite{Iclr/Decouple}       & \multicolumn{1}{c|}{\xmark}   & 41.80$^\ddagger$ & +5.92         & 47.54$^\dagger$          & +8.66          & 67.60$^\dagger$  & +6.72            \\ \midrule
Ours & \multicolumn{1}{c|}{\cmark}   & \textbf{42.90} & \textbf{+7.02} & \textbf{48.41} & \textbf{+9.53} & \textbf{69.15}   & \textbf{+8.27}   \\ 
\bottomrule
\end{tabular}%
}
\label{Tab.ResultLarge}
\vspace{-5pt}
\end{table*}

\textbf{Results.} Tab.\ref{Tab.ResultLarge} illustrates the results on large-scale datasets. Ours is consistently effective and outperforms existing mainstream methods, achieving distinguish improvement compared with previous SOTA c-RT \cite{Iclr/Decouple} in the compared backbones. Especially, our method outperforms the baseline on ImageNet-LT and iNaturalist 2018 by \textbf{9.53\%} and \textbf{8.27\%} with ResNet-50, respectively. As can be noticed in Tab.\ref{Tab.ResultLarge}, the proposed method also surpasses the well-known two-stage methods \cite{Iclr/Decouple, Nips/LDAM, Cvpr/BBN}, achieving superior accuracy with less computation load in a concise training manner.

\subsection{Further Analysis} \label{sec:abl}
\textbf{Effectiveness of UniMix and Bayias.} We conduct extensive ablation studies in Tab.\ref{Tab.Ablation} to demonstrate the effectiveness of the proposed UnixMix and Bayias, with detailed analysis in various combinations of feature-wise and loss-wise methods on CIFAR-10-LT and CIFAR-100-LT. Indeed, both UniMix and Bayias turn out to be effective in LT scenarios. Further observation shows that with \textit{calibration} ensured, the proposed method gets significant performance gains and achieve state-of-the-art results. Noteworthy, LDAM \cite{Nips/LDAM} makes classifiers miscalibrated, which leads to unsatisfactory improvement when combined with \textit{mixup} manners.
%  (see Appendix \ref{Apdx:cali_exp_add})

\begin{table*}[h!]
\vspace{-5pt}
\centering
\topcaption{Ablation study between feature-wise and loss-wise methods. LDAM is counterproductive to \textit{mixup} and its extensions. Bayias-compensated CE ensures \textit{calibration} and shows excellent performance gains especially combined with UniMix.} 
% \tiny
\resizebox{1\textwidth}{!}{%
\begin{tabular}{c|c|cc|cc|cccc}
\toprule
\multicolumn{2}{c|}{Dataset} & \multicolumn{4}{c|}{CIFAR-10-LT}                                                 & \multicolumn{4}{c}{CIFAR-100-LT}                                                                      \\ \midrule
\multicolumn{2}{c|}{$\rho$ (easy $\rightarrow$ hard)}  & \multicolumn{2}{c|}{100}    & \multicolumn{2}{c|}{200}    & \multicolumn{2}{c|}{100}                                    & \multicolumn{2}{c}{200}                \\ \midrule
Mix           & Loss         & Top1 Acc       & $\Delta$              & Top1 Acc       & $\Delta$              & Top1 Acc       & \multicolumn{1}{c|}{$\Delta$}              & Top1 Acc       & $\Delta$              \\ \midrule
None          & CE           & 70.36          & -                     & 66.21          & -                     & 38.32          & \multicolumn{1}{c|}{-}                     & 34.56          & -                     \\
\textit{mixup}        & CE           & 73.06          & +2.70                  & 67.73          & +1.52                  & 39.54          & \multicolumn{1}{c|}{+1.22}                  & 34.97          & +0.41                  \\
Remix         & CE           & 75.36          & +5.00                  & 67.08          & +0.87                  & 41.94          & \multicolumn{1}{c|}{+3.62}                  & 36.99          & +2.43                  \\
UniMix        & CE           & 76.47          & +6.11                  & 68.42          & +2.21                  & 41.46          & \multicolumn{1}{c|}{+3.14}                  & 37.63          & +3.07                  \\ \midrule
None          & LDAM         & 74.47          & -                     & 69.50          & -                     & 41.76          & \multicolumn{1}{c|}{-   }                   & 37.73          & -                     \\
\textit{mixup}         & LDAM         & 73.96          & -0.15                 & 67.89          & -1.61                 & 40.22          & \multicolumn{1}{c|}{-1.54}                  & 37.52          & -0.21                  \\
Remix         & LDAM         & 74.33          & -0.14                 & 69.66          & +0.16                 & 40.59          & \multicolumn{1}{c|}{-1.17}                  & 37.66          & -0.07                  \\
UniMix        & LDAM         & 75.35          & +0.88                 & 70.77          & +1.27                 & 41.67          & \multicolumn{1}{c|}{-0.09}                  & 37.83          & +0.01                  \\ \midrule
None          & Bayias       & 78.70          & -                     & 74.21          & -                     & 43.52          & \multicolumn{1}{c|}{-   }                  & 38.83          & -                     \\
\textit{mixup}         & Bayias       & 81.75          & +3.05                 & 76.69          & +2.48                 & 44.56          & \multicolumn{1}{c|}{+1.04}                  & 41.19          & +2.36                  \\
Remix         & Bayias       & 81.55          & +2.85                 & 75.81          & +1.60                  & 45.01          & \multicolumn{1}{c|}{+1.49}                  & 41.44          & +2.61                  \\
UniMix        & Bayias       & \textbf{82.75} & \textbf{+4.05}        & \textbf{78.48} & \textbf{+4.27}        & \textbf{45.45} & \multicolumn{1}{c|}{\textbf{+1.93}}         & \textbf{42.07} & \textbf{+3.24}         \\ 
\bottomrule
\end{tabular}}
\label{Tab.Ablation}
\vspace{-5pt}
\end{table*}

\textbf{Evaluating different UniMix Sampler.} Corollary \ref{Col.1},\ref{Col.2},\ref{Col.3} demonstrate distinguish influence of UniMix. However, the \textit{$\xi$-sample} can not be completely equivalent with orginal ones. Hence, an appropriate $\tau$ in Eq.\ref{Eq.reversesample} is also worth further searching. Fig.\ref{Fig.AblationTau} illustrates the accuracy with different $\tau$ on CIFAR-10-LT and CIFAR-100-LT setting $\rho=10$ and $100$. For CIFAR-10-LT (Fig.\ref{exp:tau(a)},\ref{exp:tau(c)}), $\tau=-1$ is possibly ideal, which forces more head-tail instead of head-head pairs get generated to compensate tail classes. In the more challenging CIFAR-100-LT, $\tau=0$ achieves the best result. We suspect that unlike simple datasets (e.g., CIFAR-10-LT), where overconfidence occurs in head classes, all classes need to get enhanced in complicated LT scenarios. Hence, the augmentation is effective and necessary both on head and tail. $\tau=0$ allows both head and tail get improved simultaneously.

\begin{figure*}[h!]
\vspace{-5pt}
	\centering
	\subfigure[CIFAR-10-LT-10.]{
		\begin{minipage}[b]{0.233\textwidth}
			\includegraphics[width=1\textwidth]{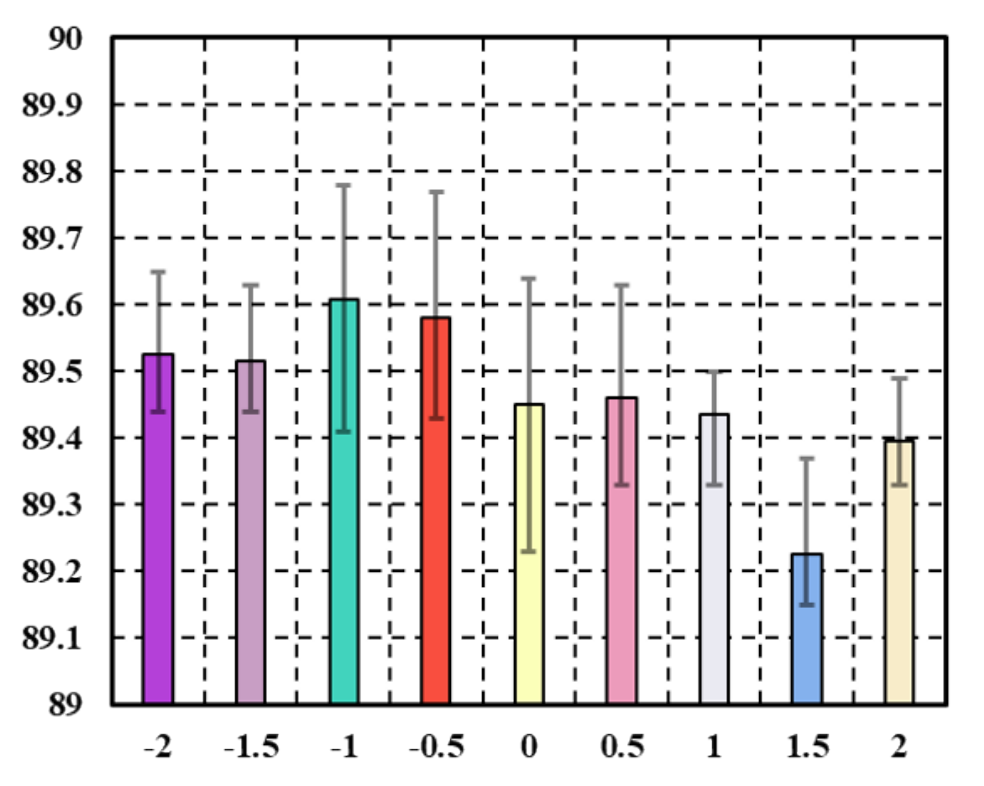} 
		\end{minipage}
		\label{exp:tau(a)}
	}
	\subfigure[CIFAR-10-LT-100.]{
		\begin{minipage}[b]{0.233\textwidth}
			\includegraphics[width=1\textwidth]{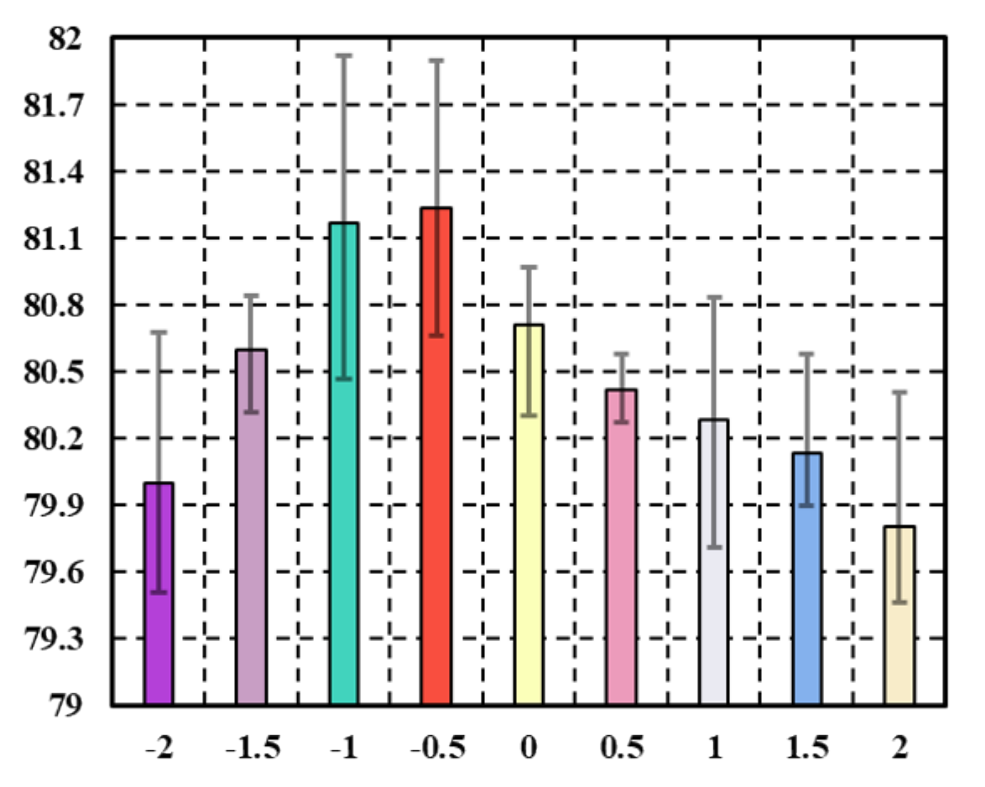} 
		\end{minipage}
		\label{exp:tau(c)}
	}
	    	\subfigure[CIFAR-100-LT-10.]{
    		\begin{minipage}[b]{0.233\textwidth}
  		 	\includegraphics[width=1\textwidth]{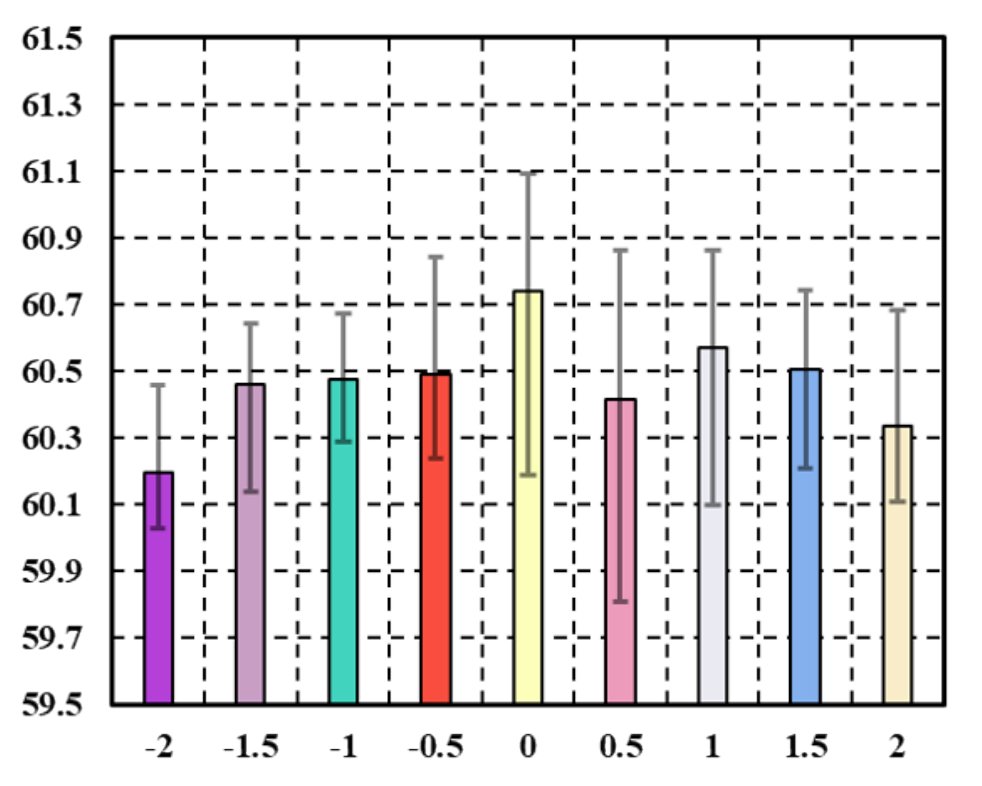}
    		\end{minipage}
		\label{exp:tau(b)}
    	}
    	\subfigure[CIFAR-100-LT-100.]{
    		\begin{minipage}[b]{0.233\textwidth}
		 	\includegraphics[width=1\textwidth]{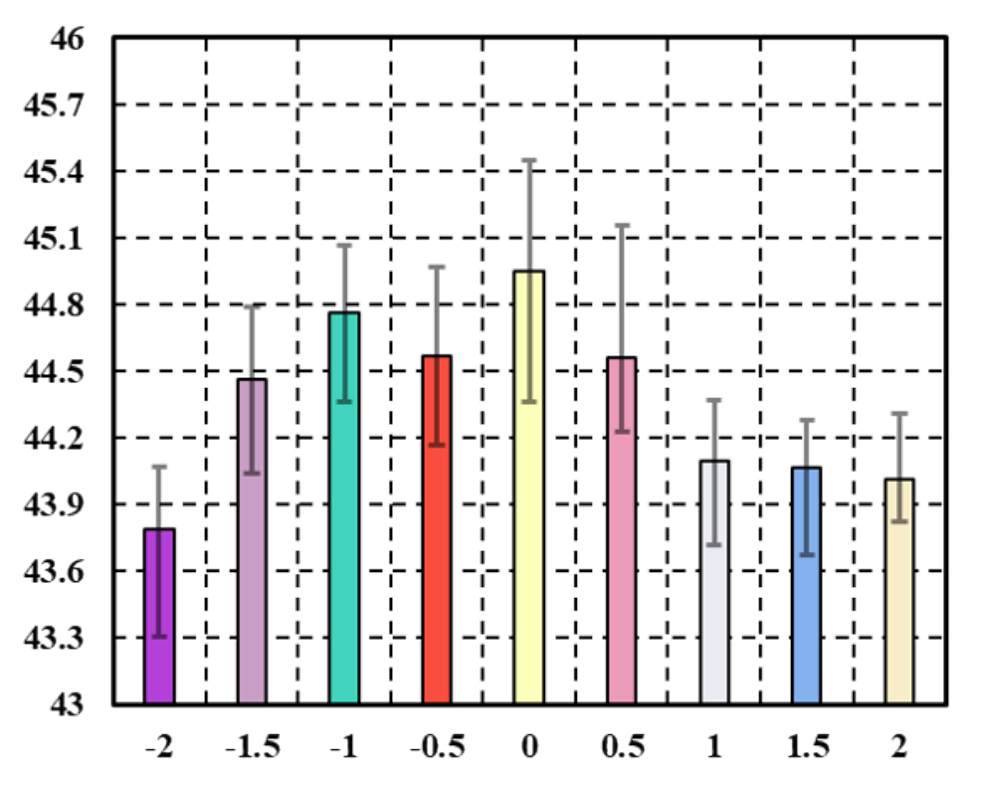}
    		\end{minipage}
		\label{exp:tau(d)}
    	}
	\topcaption{Comparison of top-1 validation accuracy(\%) of ResNet-32 on CIFAR-LT when varying $\tau$ in Eq.\ref{Eq.reversesample} for UniMix. The histogram indicates average results in repeated experiments.}
    \label{Fig.AblationTau}
    \vspace{-5pt}
\end{figure*}

\textbf{Do minorities really get improved?}
To observe the amelioration on tail classes, Fig.\ref{Fig.confusematrix} visualizes $\log$-confusion matrices on CIFAR-100-LT-100. In Fig.\ref{Fig.confusematrix(E)}, our method exhibits satisfactory generalization on the tail. Vanilla ERM model (Fig.\ref{Fig.confusematrix(A)}) is a trivial predictor which simplifies tail instances as head labels to minimize the error rate. Feature improvement \cite{Eccv/remix} and loss modification \cite{Nips/LDAM,Corr/CDT} methods do alleviate LT problem to some extent. The misclassification cases (i.e., non-diagonal elements) in Fig.\ref{Fig.confusematrix(B)},\ref{Fig.confusematrix(C)},\ref{Fig.confusematrix(D)} become smaller and more balanced distributed compared with ERM. However, the error cases are still mainly in the upper or lower triangular, indicating the existence of inherent bias between the head and tail. Our method (Fig.\ref{Fig.confusematrix(E)}) significantly alleviates such dilemma. The non-diagonal elements are more uniformly distributed throughout the matrix rather than in the corners, showing superiority to erase the bias in LT scenarios. Our method enables effective feature improvement for data-scarce classes and alleviates the \textit{prior} bias, suggesting our success in regularizing tail remarkably.

\begin{figure}[t!]
% \vspace{-5pt}
	\centering
	\subfigure[ERM]{
		\begin{minipage}[b]{0.182\textwidth}
			\includegraphics[width=1\textwidth]{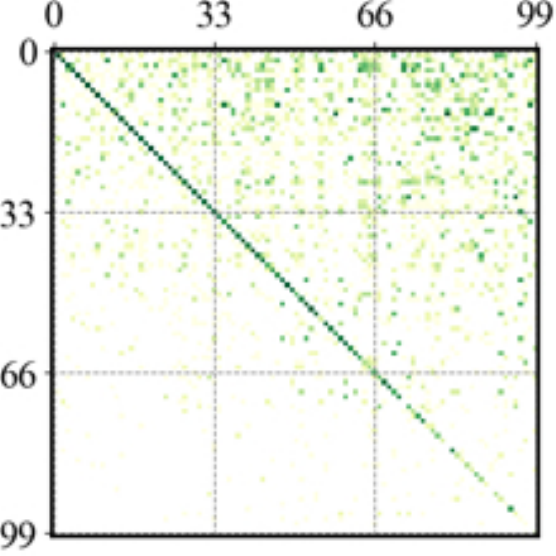}
		\end{minipage}
		\label{Fig.confusematrix(A)}
	}
		\subfigure[LDAM+DRW \cite{Nips/LDAM}]{
		\begin{minipage}[b]{0.182\textwidth}
			\includegraphics[width=1\textwidth]{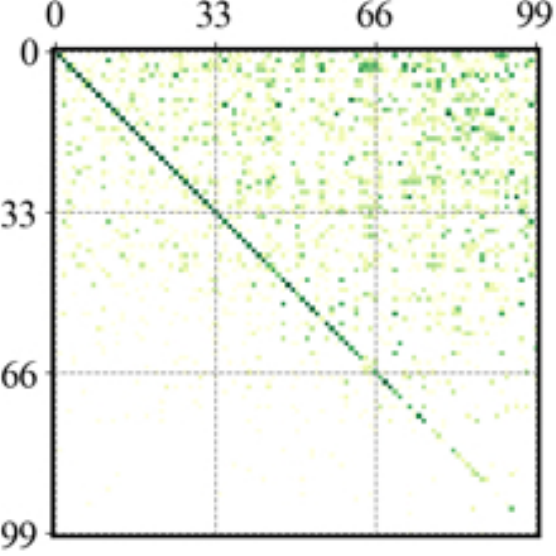} 
		\end{minipage}
		\label{Fig.confusematrix(B)}
	}
    	\subfigure[CDT \cite{Corr/CDT}]{
    		\begin{minipage}[b]{0.182\textwidth}
  		 	\includegraphics[width=1\textwidth]{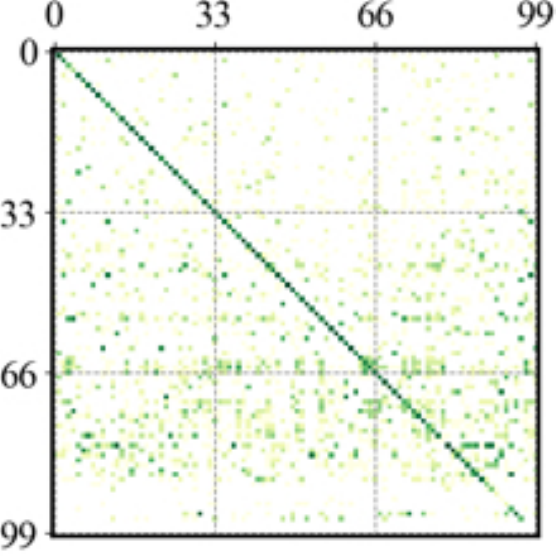}
    		\end{minipage}
		\label{Fig.confusematrix(C)}
    }
        	\subfigure[Remix \cite{Eccv/remix}]{
    		\begin{minipage}[b]{0.182\textwidth}
		 	\includegraphics[width=1\textwidth]{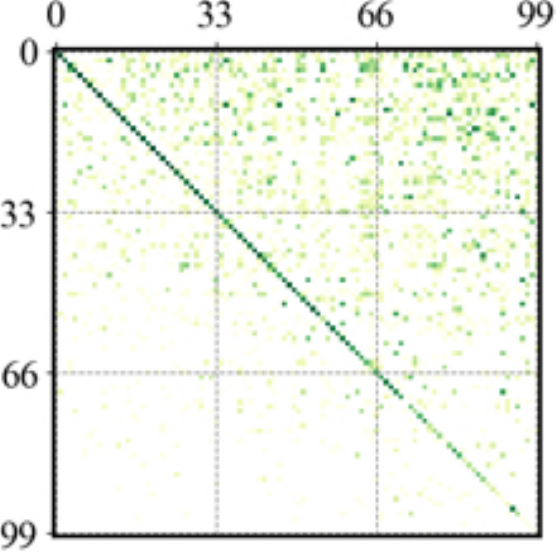}
    		\end{minipage}
		\label{Fig.confusematrix(D)}
    }
    	\subfigure[Ours]{
    		\begin{minipage}[b]{0.182\textwidth}
		 	\includegraphics[width=1\textwidth]{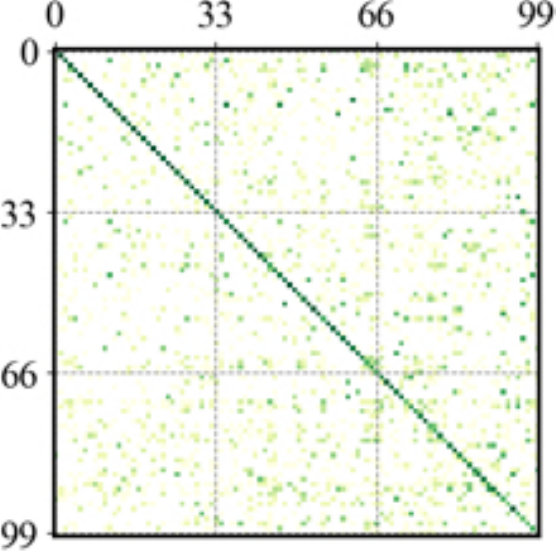}
    		\end{minipage}
		\label{Fig.confusematrix(E)}
% 		\vspace{-0.5cm}
    }
\topcaption{The $\log$-confusion matrix on CIFAR-100-LT-100 validation dataset. The $x$-axis and $y$-axis indicate the ground truth and predicted labels, respectively. Deeper color indicates larger values.}
\label{Fig.confusematrix}
\vspace{-5pt}
\end{figure}

\section{Related work and discussion}
\textbf{Why need \textit{calibration}?} To quantify the predictive uncertainty, \textit{calibration} \cite{Journal/Classification_Calibration} is put forward to describe the relevance between predictive score and actual correctness likelihood. A well-calibrated model is more reliable with better interpretability, which probabilities indicate optimal expected costs in Bayesian decision scenarios \cite{Ijon/Calibration_bayesian_networks}. Guo \textit{et al.} \cite{Icml/Calibration-NN} firstly provide metric to measure the \textit{calibration} of CNN and figure out well-performed models are always in lack of \textit{calibration}, indicating that CNN is sensitive to be overconfidence and lacks robustness. Thulasidasan \textit{et al.} \cite{Nips/On_Mixup_Training} point out that the effectiveness of \textit{mixup} in balanced datasets originates from superior \textit{calibration} modification. Menon \textit{et al.} \cite{Aaai/logit_adjustment} further show how to ensure optimal classification \textit{calibration} for a pair-wise loss. 

\textbf{Feature-wise methods.} Intuitively, under-sampling the head \cite{Icic/Board_SMOTE, Nn/OverOrUnderSample, Tkde/Under-sampling1} or over-sampling the tail \cite{Jair/SMOTE, Nn/OverOrUnderSample, Iccv/Gene-SMOTE, Eccv/Over-samlping1, Icml/Over-sampling2} can improve the inconsistent performance of imbalanced datasets but tend to either weaken the head or over-fitting the tail. Hence, many effective works generate additional samples \cite{Eccv/FSA, Aaai/Bag-tricks-LT, Cvpr/M2m} to compensate the tail classes. BBN \cite{Cvpr/BBN} uses two branches to extract features from head and tail simultaneously, while c-RT \cite{Iclr/Decouple} trains feature representation learning and classification stage separately. \textit{mixup} \cite{Iclr/mixup} and its variants \cite{Icml/Manifold_Mixup, Iccv/CutMix, Eccv/remix} are effective and easy-implement feature-wise methods that convexly combine input and label pairs to generate virtual samples. However, na\"ive \textit{mixup} manners are deficient in LT scenarios as we discussed in Sec.\ref{limitmixup}. In contrast, our UniMix tackles such a dilemma by constructing class balance-oriented virtual data as describe in Sec.\ref{Sec.UniMix} and shows satisfactory \textit{calibration} as Fig.\ref{Fig.confacc} exhibits.

\textbf{Loss modification.} Numerous experimental and theoretical studies \cite{Nips/BMSoftmax, Cvpr/CB, Nips/Causal-LT, Nips/Rethinking-labels-LT} have demonstrated the existence of inherent \textit{bias} between LT train set and balanced test set in supervised learning. Previous works \cite{Cvpr/reweight1, Pami/reweight2, Tnn/CSCE, Cvpr/CB, Iccv/Focal} make networks prefer learning tail samples by additional class-related weight on CE loss. Some works further correct CE according to the gradient generated by different samples \cite{Cvpr/EQL, Aaai/GHM} or from the perspective of Gaussian distribution and Bayesian estimation \cite{Iccv/Gaussian-Max-Margin-cls, Cvpr/Striking-Uncertainty}. Meta-learning approaches \cite{Icml/Meta1, Nips/Meta2, Nips/MWNet, Icml/L2RW, Cvpr/RethinkDA} optimize the weights of each class in CE as learnable parameters and achieve remarkable success. Cao \textit{et al.} \cite{Nips/LDAM} theoretically provides the ideal optimal \textit{margin} for CE from the perspective of \textit{VC} generalization bound. Compared with Logit Adjustment \cite{Aaai/logit_adjustment} motivated by balance error rate, our Bayias-compensated CE eliminates \textit{bias} incured by \textit{prior} and is consistent with balanced datasets, which ensures classification \textit{calibration} as well.

\section{Conclusion}\label{Sec.Conclusion}

We systematically analyze the limitations of mainstream feature improvement methods, i.e., \textit{mixup} and its extensions in the label-imbalanced situation, and propose the UniMix to construct a more class-balanced virtual dataset that significantly improves classification \textit{calibration}. We further pinpoint an inherent bias induced by the inconstancy of label distribution \textit{prior} between long-tailed train set and balanced test set. We prove that the standard cross-entropy loss with the proposed Bayias compensated can ensure classification \textit{calibration}. The combination of UniMix and Bayias achieves state-of-the-art performance and contributes to a better-calibrated model (Fig.\ref{Fig.confacc}). Further study in Tab.\ref{Tab.Ablation} shows that the bad \textit{calibration} methods are counterproductive with each other. However, more in-depth analysis and theoretical guarantees are still required, which we leave for our future work.
\begin{ack}
This work was supported by NSFC project Grant No. U1833101, SZSTI Grant No. JCYJ20190809 172201639 and WDZC20200820200655001, the Joint Research Center of Tencent and Tsinghua.
\end{ack}

% %%%%%%%%%%%%%%%%%%%%%%%
{\small
% \bibliographystyle{plain}
% \bibliography{egbib.bib}
\bibliography{all.bib}

\begin{thebibliography}{10}

\bibitem{Nips/Meta2}
Marcin Andrychowicz, Misha Denil, Sergio~Gomez Colmenarejo, Matthew~W. Hoffman,
  David Pfau, Tom Schaul, and Nando de~Freitas.
\newblock Learning to learn by gradient descent by gradient descent.
\newblock In {\em Advances in Neural Information Processing Systems 29: Annual
  Conference on Neural Information Processing Systems 2016}, pages 3981--3989,
  2016.

\bibitem{DBLP:conf/iclr/AshukhaLMV20}
Arsenii Ashukha, Alexander Lyzhov, Dmitry Molchanov, and Dmitry~P. Vetrov.
\newblock Pitfalls of in-domain uncertainty estimation and ensembling in deep
  learning.
\newblock In {\em 8th International Conference on Learning Representations,
  {ICLR} 2020}. OpenReview.net, 2020.

\bibitem{Journal/Classification_Calibration}
Peter~L Bartlett, Michael~I Jordan, and Jon~D McAuliffe.
\newblock Convexity, classification, and risk bounds.
\newblock {\em Journal of the American Statistical Association},
  101(473):138--156, 2006.

\bibitem{Corr/Auto_drive}
Mariusz Bojarski, Davide~Del Testa, Daniel Dworakowski, Bernhard Firner, Beat
  Flepp, Prasoon Goyal, Lawrence~D. Jackel, Mathew Monfort, Urs Muller, Jiakai
  Zhang, Xin Zhang, Jake Zhao, and Karol Zieba.
\newblock End to end learning for self-driving cars.
\newblock {\em CoRR}, abs/1604.07316, 2016.

\bibitem{Brier-Score}
Glenn~W Brier et~al.
\newblock Verification of forecasts expressed in terms of probability.
\newblock {\em Monthly weather review}, 78(1):1--3.

\bibitem{Nn/OverOrUnderSample}
Mateusz Buda, Atsuto Maki, and Maciej~A. Mazurowski.
\newblock A systematic study of the class imbalance problem in convolutional
  neural networks.
\newblock {\em Neural Networks}, 106:249--259, 2018.

\bibitem{Icml/Over-sampling2}
Jonathon Byrd and Zachary~Chase Lipton.
\newblock What is the effect of importance weighting in deep learning?
\newblock In {\em Proceedings of the 36th International Conference on Machine
  Learning, {ICML} 2019}, volume~97 of {\em Proceedings of Machine Learning
  Research}, pages 872--881. {PMLR}, 2019.

\bibitem{Nips/LDAM}
Kaidi Cao, Colin Wei, Adrien Gaidon, Nikos Ar{\'{e}}chiga, and Tengyu Ma.
\newblock Learning imbalanced datasets with label-distribution-aware margin
  loss.
\newblock In {\em Advances in Neural Information Processing Systems NeurIPS
  2019}, pages 1565--1576, 2019.

\bibitem{Kdd/Medical_diagnosis}
Rich Caruana, Yin Lou, Johannes Gehrke, Paul Koch, Marc Sturm, and Noemie
  Elhadad.
\newblock Intelligible models for healthcare: Predicting pneumonia risk and
  hospital 30-day readmission.
\newblock In {\em Proceedings of the 21th {ACM} {SIGKDD} International
  Conference on Knowledge Discovery and Data Mining}, pages 1721--1730. {ACM},
  2015.

\bibitem{Nips/VRM}
Olivier Chapelle, Jason Weston, L{\'{e}}on Bottou, and Vladimir Vapnik.
\newblock Vicinal risk minimization.
\newblock In {\em Advances in Neural Information Processing Systems 13, Papers
  from Neural Information Processing Systems {(NIPS)} 2000, Denver, CO, {USA}},
  pages 416--422. {MIT} Press, 2000.

\bibitem{Jair/SMOTE}
Nitesh~V. Chawla, Kevin~W. Bowyer, Lawrence~O. Hall, and W.~Philip Kegelmeyer.
\newblock {SMOTE:} synthetic minority over-sampling technique.
\newblock {\em J. Artif. Intell. Res.}, 16:321--357, 2002.

\bibitem{Eccv/remix}
Hsin{-}Ping Chou, Shih{-}Chieh Chang, Jia{-}Yu Pan, Wei Wei, and Da{-}Cheng
  Juan.
\newblock Remix: Rebalanced mixup.
\newblock In {\em Computer Vision - {ECCV} 2020 Workshops}, volume 12540, pages
  95--110. Springer, 2020.

\bibitem{Eccv/FSA}
Peng Chu, Xiao Bian, Shaopeng Liu, and Haibin Ling.
\newblock Feature space augmentation for long-tailed data.
\newblock In {\em Computer Vision - {ECCV} 2020 - 16th European Conference},
  volume 12374 of {\em Lecture Notes in Computer Science}, pages 694--710.
  Springer, 2020.

\bibitem{Cvpr/CB}
Yin Cui, Menglin Jia, Tsung{-}Yi Lin, Yang Song, and Serge~J. Belongie.
\newblock Class-balanced loss based on effective number of samples.
\newblock In {\em {IEEE} Conference on Computer Vision and Pattern Recognition,
  {CVPR} 2019}, pages 9268--9277. Computer Vision Foundation / {IEEE}, 2019.

\bibitem{Cvpr/ArcFace}
Jiankang Deng, Jia Guo, Niannan Xue, and Stefanos Zafeiriou.
\newblock Arcface: Additive angular margin loss for deep face recognition.
\newblock In {\em {IEEE} Conference on Computer Vision and Pattern Recognition,
  {CVPR} 2019}, pages 4690--4699. Computer Vision Foundation / {IEEE}, 2019.

\bibitem{Ac/SVM}
Theodoros Evgeniou and Massimiliano Pontil.
\newblock Support vector machines: Theory and applications.
\newblock In {\em Machine Learning and Its Applications, Advanced Lectures},
  volume 2049 of {\em Lecture Notes in Computer Science}, pages 249--257.
  Springer, 2001.

\bibitem{Icml/Meta1}
Chelsea Finn, Pieter Abbeel, and Sergey Levine.
\newblock Model-agnostic meta-learning for fast adaptation of deep networks.
\newblock In {\em Proceedings of the 34th International Conference on Machine
  Learning, {ICML} 2017}, volume~70 of {\em Proceedings of Machine Learning
  Research}, pages 1126--1135. {PMLR}, 2017.

\bibitem{Icml/Calibration-NN}
Chuan Guo, Geoff Pleiss, Yu~Sun, and Kilian~Q. Weinberger.
\newblock On calibration of modern neural networks.
\newblock In {\em Proceedings of the 34th International Conference on Machine
  Learning, {ICML} 2017}, volume~70 of {\em Proceedings of Machine Learning
  Research}, pages 1321--1330. {PMLR}, 2017.

\bibitem{Cvpr/LVIS}
Agrim Gupta, Piotr Doll{\'{a}}r, and Ross~B. Girshick.
\newblock {LVIS:} {A} dataset for large vocabulary instance segmentation.
\newblock In {\em {IEEE} Conference on Computer Vision and Pattern Recognition,
  {CVPR} 2019}, pages 5356--5364. Computer Vision Foundation / {IEEE}, 2019.

\bibitem{Icic/Board_SMOTE}
Hui Han, Wenyuan Wang, and Binghuan Mao.
\newblock Borderline-smote: {A} new over-sampling method in imbalanced data
  sets learning.
\newblock In {\em Advances in Intelligent Computing, International Conference
  on Intelligent Computing, {ICIC} 2005}, volume 3644 of {\em Lecture Notes in
  Computer Science}, pages 878--887. Springer, 2005.

\bibitem{Iccv/Gaussian-Max-Margin-cls}
Munawar Hayat, Salman~H. Khan, Syed~Waqas Zamir, Jianbing Shen, and Ling Shao.
\newblock Gaussian affinity for max-margin class imbalanced learning.
\newblock In {\em 2019 {IEEE/CVF} International Conference on Computer Vision,
  {ICCV} 2019}, pages 6468--6478. {IEEE}, 2019.

\bibitem{Tkde/Under-sampling1}
Haibo He and Edwardo~A. Garcia.
\newblock Learning from imbalanced data.
\newblock {\em {IEEE} Trans. Knowl. Data Eng.}, 21(9):1263--1284, 2009.

\bibitem{Iccv/Mask-RCNN}
Kaiming He, Georgia Gkioxari, Piotr Doll{\'{a}}r, and Ross~B. Girshick.
\newblock Mask {R-CNN}.
\newblock In {\em {IEEE} International Conference on Computer Vision, {ICCV}
  2017}, pages 2980--2988. {IEEE} Computer Society, 2017.

\bibitem{Cvpr/ResNet}
Kaiming He, Xiangyu Zhang, Shaoqing Ren, and Jian Sun.
\newblock Deep residual learning for image recognition.
\newblock In {\em 2016 {IEEE} Conference on Computer Vision and Pattern
  Recognition, {CVPR} 2016}, pages 770--778. {IEEE} Computer Society, 2016.

\bibitem{Cvpr/Bag-tricks-cls}
Tong He, Zhi Zhang, Hang Zhang, Zhongyue Zhang, Junyuan Xie, and Mu~Li.
\newblock Bag of tricks for image classification with convolutional neural
  networks.
\newblock In {\em {IEEE} Conference on Computer Vision and Pattern Recognition,
  {CVPR} 2019}, pages 558--567. Computer Vision Foundation / {IEEE}, 2019.

\bibitem{Corr/RethinkingCleanAugmented}
Zhuoxun He, Lingxi Xie, Xin Chen, Ya~Zhang, Yanfeng Wang, and Qi~Tian.
\newblock Data augmentation revisited: Rethinking the distribution gap between
  clean and augmented data.
\newblock {\em CoRR}, abs/1909.09148, 2019.

\bibitem{LADE}
Youngkyu Hong, Seungju Han, Kwanghee Choi, Seokjun Seo, Beomsu Kim, and Buru
  Chang.
\newblock Disentangling label distribution for long-tailed visual recognition.
\newblock In {\em Proceedings of the IEEE/CVF Conference on Computer Vision and
  Pattern Recognition}, pages 6626--6636.

\bibitem{Cvpr/INaturalist}
Grant~Van Horn, Oisin~Mac Aodha, Yang Song, Yin Cui, Chen Sun, Alexander
  Shepard, Hartwig Adam, Pietro Perona, and Serge~J. Belongie.
\newblock The inaturalist species classification and detection dataset.
\newblock In {\em 2018 {IEEE} Conference on Computer Vision and Pattern
  Recognition, {CVPR} 2018}, pages 8769--8778. {IEEE} Computer Society, 2018.

\bibitem{Corr/DevilinTail}
Grant~Van Horn and Pietro Perona.
\newblock The devil is in the tails: Fine-grained classification in the wild.
\newblock {\em CoRR}, abs/1709.01450, 2017.

\bibitem{Cvpr/reweight1}
Chen Huang, Yining Li, Chen~Change Loy, and Xiaoou Tang.
\newblock Learning deep representation for imbalanced classification.
\newblock In {\em 2016 {IEEE} Conference on Computer Vision and Pattern
  Recognition, {CVPR} 2016}, pages 5375--5384. {IEEE} Computer Society, 2016.

\bibitem{Pami/reweight2}
Chen Huang, Yining Li, Chen~Change Loy, and Xiaoou Tang.
\newblock Deep imbalanced learning for face recognition and attribute
  prediction.
\newblock {\em {IEEE} Trans. Pattern Anal. Mach. Intell.}, 42(11):2781--2794,
  2020.

\bibitem{Cvpr/RethinkDA}
Muhammad~Abdullah Jamal, Matthew Brown, Ming{-}Hsuan Yang, Liqiang Wang, and
  Boqing Gong.
\newblock Rethinking class-balanced methods for long-tailed visual recognition
  from a domain adaptation perspective.
\newblock In {\em 2020 {IEEE/CVF} Conference on Computer Vision and Pattern
  Recognition, {CVPR} 2020}, pages 7607--7616. {IEEE}, 2020.

\bibitem{Iclr/Decouple}
Bingyi Kang, Saining Xie, Marcus Rohrbach, Zhicheng Yan, Albert Gordo, Jiashi
  Feng, and Yannis Kalantidis.
\newblock Decoupling representation and classifier for long-tailed recognition.
\newblock In {\em 8th International Conference on Learning Representations,
  {ICLR} 2020}. OpenReview.net, 2020.

\bibitem{Tnn/CSCE}
Salman~H. Khan, Munawar Hayat, Mohammed Bennamoun, Ferdous~Ahmed Sohel, and
  Roberto Togneri.
\newblock Cost-sensitive learning of deep feature representations from
  imbalanced data.
\newblock {\em {IEEE} Trans. Neural Networks Learn. Syst.}, 29(8):3573--3587,
  2018.

\bibitem{Cvpr/Striking-Uncertainty}
Salman~H. Khan, Munawar Hayat, Syed~Waqas Zamir, Jianbing Shen, and Ling Shao.
\newblock Striking the right balance with uncertainty.
\newblock In {\em {IEEE} Conference on Computer Vision and Pattern Recognition,
  {CVPR} 2019}, pages 103--112. Computer Vision Foundation / {IEEE}, 2019.

\bibitem{Cvpr/M2m}
Jaehyung Kim, Jongheon Jeong, and Jinwoo Shin.
\newblock M2m: Imbalanced classification via major-to-minor translation.
\newblock In {\em 2020 {IEEE/CVF} Conference on Computer Vision and Pattern
  Recognition, {CVPR} 2020}, pages 13893--13902. {IEEE}, 2020.

\bibitem{Ijcv/Long-tailed-dataset}
Ranjay Krishna, Yuke Zhu, Oliver Groth, Justin Johnson, Kenji Hata, Joshua
  Kravitz, Stephanie Chen, Yannis Kalantidis, Li{-}Jia Li, David~A. Shamma,
  Michael~S. Bernstein, and Li~Fei{-}Fei.
\newblock Visual genome: Connecting language and vision using crowdsourced
  dense image annotations.
\newblock {\em Int. J. Comput. Vis.}, 123(1):32--73, 2017.

\bibitem{CIFAR}
Alex Krizhevsky, Geoffrey Hinton, et~al.
\newblock Learning multiple layers of features from tiny images.
\newblock 2009.

\bibitem{Aaai/GHM}
Buyu Li, Yu~Liu, and Xiaogang Wang.
\newblock Gradient harmonized single-stage detector.
\newblock In {\em The Thirty-Third {AAAI} Conference on Artificial
  Intelligence, {AAAI} 2019}, pages 8577--8584. {AAAI} Press, 2019.

\bibitem{Cvpr/GroupSoftmax}
Yu~Li, Tao Wang, Bingyi Kang, Sheng Tang, Chunfeng Wang, Jintao Li, and Jiashi
  Feng.
\newblock Overcoming classifier imbalance for long-tail object detection with
  balanced group softmax.
\newblock In {\em Proceedings of the IEEE/CVF Conference on Computer Vision and
  Pattern Recognition}, pages 10991--11000, 2020.

\bibitem{Iccv/Focal}
Tsung{-}Yi Lin, Priya Goyal, Ross~B. Girshick, Kaiming He, and Piotr
  Doll{\'{a}}r.
\newblock Focal loss for dense object detection.
\newblock In {\em {IEEE} International Conference on Computer Vision, {ICCV}
  2017}, pages 2999--3007. {IEEE} Computer Society, 2017.

\bibitem{Eccv/COCO}
Tsung{-}Yi Lin, Michael Maire, Serge~J. Belongie, James Hays, Pietro Perona,
  Deva Ramanan, Piotr Doll{\'{a}}r, and C.~Lawrence Zitnick.
\newblock Microsoft {COCO:} common objects in context.
\newblock In {\em Computer Vision - {ECCV} 2014 - 13th European Conference},
  volume 8693, pages 740--755. Springer, 2014.

\bibitem{Cvpr/AdaptiveFace}
Hao Liu, Xiangyu Zhu, Zhen Lei, and Stan~Z. Li.
\newblock Adaptiveface: Adaptive margin and sampling for face recognition.
\newblock In {\em {IEEE} Conference on Computer Vision and Pattern Recognition,
  {CVPR} 2019}, pages 11947--11956. Computer Vision Foundation / {IEEE}, 2019.

\bibitem{Cvpr/SphereFace}
Weiyang Liu, Yandong Wen, Zhiding Yu, Ming Li, Bhiksha Raj, and Le~Song.
\newblock Sphereface: Deep hypersphere embedding for face recognition.
\newblock In {\em 2017 {IEEE} Conference on Computer Vision and Pattern
  Recognition, {CVPR} 2017}, pages 6738--6746. {IEEE} Computer Society, 2017.

\bibitem{Cvpr/OLTR}
Ziwei Liu, Zhongqi Miao, Xiaohang Zhan, Jiayun Wang, Boqing Gong, and Stella~X.
  Yu.
\newblock Large-scale long-tailed recognition in an open world.
\newblock In {\em {IEEE} Conference on Computer Vision and Pattern Recognition,
  {CVPR} 2019, Long Beach, CA, USA, June 16-20, 2019}, pages 2537--2546.
  Computer Vision Foundation / {IEEE}, 2019.

\bibitem{Ijon/Calibration_bayesian_networks}
Juan Maro{\~{n}}as, Roberto Paredes, and Daniel Ramos.
\newblock Calibration of deep probabilistic models with decoupled bayesian
  neural networks.
\newblock {\em Neurocomputing}, 407:194--205, 2020.

\bibitem{Aaai/logit_adjustment}
Aditya~Krishna Menon, Sadeep Jayasumana, Ankit~Singh Rawat, Himanshu Jain,
  Andreas Veit, and Sanjiv Kumar.
\newblock Long-tail learning via logit adjustment.
\newblock In {\em International Conference on Learning Representations}, 2021.

\bibitem{Iccv/Gene-SMOTE}
Sankha~Subhra Mullick, Shounak Datta, and Swagatam Das.
\newblock Generative adversarial minority oversampling.
\newblock In {\em 2019 {IEEE/CVF} International Conference on Computer Vision,
  {ICCV} 2019}, pages 1695--1704. {IEEE}, 2019.

\bibitem{DBLP:conf/cvpr/NixonDZJT19}
Jeremy Nixon, Michael~W. Dusenberry, Linchuan Zhang, Ghassen Jerfel, and Dustin
  Tran.
\newblock Measuring calibration in deep learning.
\newblock In {\em {IEEE} Conference on Computer Vision and Pattern Recognition
  Workshops, {CVPR} Workshops 2019}, pages 38--41.

\bibitem{Nips/BMSoftmax}
Jiawei Ren, Cunjun Yu, Shunan Sheng, Xiao Ma, Haiyu Zhao, Shuai Yi, and
  Hongsheng Li.
\newblock Balanced meta-softmax for long-tailed visual recognition.
\newblock In {\em Advances in Neural Information Processing Systems 33: Annual
  Conference on Neural Information Processing Systems 2020, NeurIPS 2020},
  2020.

\bibitem{Icml/L2RW}
Mengye Ren, Wenyuan Zeng, Bin Yang, and Raquel Urtasun.
\newblock Learning to reweight examples for robust deep learning.
\newblock In {\em Proceedings of the 35th International Conference on Machine
  Learning, {ICML} 2018}, volume~80 of {\em Proceedings of Machine Learning
  Research}, pages 4331--4340. {PMLR}, 2018.

\bibitem{Nips/Faster-RCNN}
Shaoqing Ren, Kaiming He, Ross~B. Girshick, and Jian Sun.
\newblock Faster {R-CNN:} towards real-time object detection with region
  proposal networks.
\newblock In {\em Advances in Neural Information Processing Systems 28: Annual
  Conference on Neural Information Processing Systems 2015}, pages 91--99,
  2015.

\bibitem{Ijcv/ImageNet}
Olga Russakovsky, Jia Deng, Hao Su, Jonathan Krause, Sanjeev Satheesh, Sean Ma,
  Zhiheng Huang, Andrej Karpathy, Aditya Khosla, Michael~S. Bernstein,
  Alexander~C. Berg, and Fei{-}Fei Li.
\newblock Imagenet large scale visual recognition challenge.
\newblock {\em Int. J. Comput. Vis.}, 115(3):211--252, 2015.

\bibitem{Eccv/Over-samlping1}
Li~Shen, Zhouchen Lin, and Qingming Huang.
\newblock Relay backpropagation for effective learning of deep convolutional
  neural networks.
\newblock In {\em Computer Vision - {ECCV} 2016 - 14th European Conference},
  volume 9911, pages 467--482. Springer, 2016.

\bibitem{Nips/MWNet}
Jun Shu, Qi~Xie, Lixuan Yi, Qian Zhao, Sanping Zhou, Zongben Xu, and Deyu Meng.
\newblock Meta-weight-net: Learning an explicit mapping for sample weighting.
\newblock In {\em Advances in Neural Information Processing Systems 32: Annual
  Conference on Neural Information Processing Systems 2019, NeurIPS 2019},
  pages 1917--1928, 2019.

\bibitem{Cvpr/EQL}
Jingru Tan, Changbao Wang, Buyu Li, Quanquan Li, Wanli Ouyang, Changqing Yin,
  and Junjie Yan.
\newblock Equalization loss for long-tailed object recognition.
\newblock In {\em 2020 {IEEE/CVF} Conference on Computer Vision and Pattern
  Recognition, {CVPR} 2020}, pages 11659--11668. {IEEE}, 2020.

\bibitem{Nips/Causal-LT}
Kaihua Tang, Jianqiang Huang, and Hanwang Zhang.
\newblock Long-tailed classification by keeping the good and removing the bad
  momentum causal effect.
\newblock In {\em Advances in Neural Information Processing Systems 33: Annual
  Conference on Neural Information Processing Systems 2020, NeurIPS 2020},
  2020.

\bibitem{Nips/On_Mixup_Training}
Sunil Thulasidasan, Gopinath Chennupati, Jeff~A. Bilmes, Tanmoy Bhattacharya,
  and Sarah Michalak.
\newblock On mixup training: Improved calibration and predictive uncertainty
  for deep neural networks.
\newblock In {\em Advances in Neural Information Processing Systems 32: Annual
  Conference on Neural Information Processing Systems 2019, NeurIPS 2019},
  pages 13888--13899, 2019.

\bibitem{Icml/Manifold_Mixup}
Vikas Verma, Alex Lamb, Christopher Beckham, Amir Najafi, Ioannis Mitliagkas,
  David Lopez{-}Paz, and Yoshua Bengio.
\newblock Manifold mixup: Better representations by interpolating hidden
  states.
\newblock In {\em Proceedings of the 36th International Conference on Machine
  Learning, {ICML} 2019}, volume~97, pages 6438--6447. {PMLR}, 2019.

\bibitem{Mm/norm_face}
Feng Wang, Xiang Xiang, Jian Cheng, and Alan~Loddon Yuille.
\newblock Normface: L\({}_{\mbox{2}}\) hypersphere embedding for face
  verification.
\newblock In {\em Proceedings of the 2017 {ACM} on Multimedia Conference, {MM}
  2017}, pages 1041--1049. {ACM}, 2017.

\bibitem{Nips/Rethinking-labels-LT}
Yuzhe Yang and Zhi Xu.
\newblock Rethinking the value of labels for improving class-imbalanced
  learning.
\newblock In {\em Advances in Neural Information Processing Systems 33: Annual
  Conference on Neural Information Processing Systems 2020, NeurIPS 2020},
  2020.

\bibitem{Corr/CDT}
Han{-}Jia Ye, Hong{-}You Chen, De{-}Chuan Zhan, and Wei{-}Lun Chao.
\newblock Identifying and compensating for feature deviation in imbalanced deep
  learning.
\newblock {\em CoRR}, abs/2001.01385, 2020.

\bibitem{Iccv/CutMix}
Sangdoo Yun, Dongyoon Han, Sanghyuk Chun, Seong~Joon Oh, Youngjoon Yoo, and
  Junsuk Choe.
\newblock Cutmix: Regularization strategy to train strong classifiers with
  localizable features.
\newblock In {\em 2019 {IEEE/CVF} International Conference on Computer Vision,
  {ICCV} 2019}, pages 6022--6031. {IEEE}, 2019.

\bibitem{Iclr/mixup}
Hongyi Zhang, Moustapha Ciss{\'{e}}, Yann~N. Dauphin, and David Lopez{-}Paz.
\newblock mixup: Beyond empirical risk minimization.
\newblock In {\em 6th International Conference on Learning Representations,
  {ICLR} 2018}. OpenReview.net, 2018.

\bibitem{Corr/CalibrationForMixup}
Linjun Zhang, Zhun Deng, Kenji Kawaguchi, Amirata Ghorbani, and James Zou.
\newblock How does mixup help with robustness and generalization?
\newblock In {\em International Conference on Learning Representations}, 2021.

\bibitem{zhang2021test}
Yifan Zhang, Bryan Hooi, Lanqing Hong, and Jiashi Feng.
\newblock Test-agnostic long-tailed recognition by test-time aggregating
  diverse experts with self-supervision.
\newblock {\em arXiv preprint arXiv:2107.09249}, 2021.

\bibitem{Aaai/Bag-tricks-LT}
Yongshun Zhang, Xiu{-}Shen Wei, Boyan Zhou, and Jianxin Wu.
\newblock Bag of tricks for long-tailed visual recognition with deep
  convolutional neural networks.
\newblock In {\em AAAI}, 2021.

\bibitem{DBLP:conf/cvpr/MiSLAS}
Zhisheng Zhong, Jiequan Cui, Shu Liu, and Jiaya Jia.
\newblock Improving calibration for long-tailed recognition.
\newblock In {\em {IEEE} Conference on Computer Vision and Pattern Recognition,
  {CVPR} 2021}, pages 16489--16498. Computer Vision Foundation / {IEEE}, 2021.

\bibitem{Cvpr/BBN}
Boyan Zhou, Quan Cui, Xiu{-}Shen Wei, and Zhao{-}Min Chen.
\newblock {BBN:} bilateral-branch network with cumulative learning for
  long-tailed visual recognition.
\newblock In {\em 2020 {IEEE/CVF} Conference on Computer Vision and Pattern
  Recognition, {CVPR} 2020}, pages 9716--9725. {IEEE}, 2020.

\end{thebibliography}
}
\newpage

\appendix
\setcounter{equation}{0}
\setcounter{subsection}{0}
\setcounter{table}{0}   
\setcounter{figure}{0}
\setcounter{corollary}{0}

\renewcommand{\thetable}{A\arabic{table}}
\renewcommand{\thefigure}{A\arabic{figure}}
\renewcommand{\theequation}{A.\arabic{equation}}

\section{Missing proofs and derivations of UniMix} \label{Apdx.UniMix}

\subsection{Basic setting}

\begin{basicsetting}
\label{basicsetting}
Without loss of generality, we suppose that the long-tailed distribution satisfies some kind exponential distribution with parameter $\lambda$ \cite{Cvpr/CB}. The imbalance factor is defined as $\rho = n_{\max} / n_{\min}$, where $n_{\max}$ and $n_{\min}$ represent the number of the most and least samples in the dataset with $C$ classes, respectively. Then the probability of class $Y$ belongs to $y_i$ is:
\begin{equation} 
    \mathds{P}(Y=y_i) = \left \{
\begin{array}{lll}
\alpha e^{-\lambda y_i}    &      & {y_i \in [1,C]}\\
0                        &      & {others}
\end{array} \right. \quad \sim \quad long-tailed \ \ distribution
\end{equation}
According to the definition of $\rho$, we can directly deduce the relationship between $\rho$ and $\lambda$:
\begin{equation}
\begin{aligned}
    \rho &= \frac{n_{\max}}{ n_{\min}} = \frac{\mathds{P}(Y=y_i)_{\max}}{\mathds{P}(Y=y_i)_{\min}} \Rightarrow \frac{\mathds{P}(Y=y_1)}{\mathds{P}(Y=y_C)} =\frac{\alpha e^{-\lambda \cdot 1}}{\alpha e^{-\lambda \cdot C}} \quad \Rightarrow \quad \lambda = \frac{\ln \rho}{C-1} 
\end{aligned}
\end{equation}
Considering the normalization of probability density, we have:
\begin{equation}
    \begin{aligned}
    \int_{y_i \in \mathcal{Y}} {\mathds{P}(Y = y_i)dy_i} &= \int_1^C {\alpha {e^{ - \lambda {y_i}}}d{y_i}} \equiv 1\\
    &\left. \Rightarrow -\frac{\alpha }{\lambda }e^{-\lambda y_i}\right |_1^C \equiv 1 \Rightarrow \frac{\alpha }{\lambda }\left( {{e^{ - \lambda }} - {e^{ - \lambda C}}} \right) \equiv 1 \\
    &\Rightarrow \alpha  = \frac{\lambda }{{{e^{ - \lambda }} - {e^{ - \lambda C}}}}
\end{aligned}
\end{equation}
Hence, we can express the distribution of LT dataset $\mathcal{D}_{train}$, i.e., $\mathds{P}(Y=y_i)$ represents the probability of class $y_i$ in $\mathcal{D}_{train}$, which can be calculated as follows:
\begin{equation}
\label{Eq.LTdistribution}
    \begin{aligned}
\mathds{P}(Y=y_i) = \frac{{\iint_{x_i \in \mathcal{X},y_j \in \mathcal{Y}} {\mathds{1}(X = x_i,Y =y_i) dx_i dy_j}}}{{\iint_{{x_i} \in \mathcal{X},{y_j} \in \mathcal{Y}} {\mathds{1}(X = {x_i},Y={y_j})dx_i dy_j}}}= \frac{\lambda}{e^{-\lambda}-e^{-\lambda C}} {e ^{-\lambda y_i}},y_i \in [1,C]
\end{aligned}
\end{equation}
Notice that Eq.\ref{Eq.LTdistribution} is determined by total class number $C$ and the imbalance factor $\rho$.
\end{basicsetting}

% =============================================
\subsection{Proof of Corollary \ref{Col.1}} \label{Apdx:Coro1}
\begin{corollary}
\label{Col.1}
    When $\bm{\xi \sim Beta(\alpha,\alpha)},\alpha\in [0,1]$, the newly mixed dataset $\mathcal{D}_\nu$ composed of $\xi$-Aug samples $(\widetilde{x}_{i,j}, \widetilde{y}_{i,j})$ follows the same long-tailed distribution as the origin dataset $\mathcal{D}_{train}$, where $(x_i,y_i)$ and $(x_j,y_j)$ are \textbf{randomly} sampled from $\mathcal{D}_{train}$.
\begin{equation}
% \small
\label{Eq.ProbMixup}
\begin{aligned}
    \mathds{P}_{\textit{mixup}}(Y^* = {y_i}) &= \mathds{P}^2(Y = {y_i}) + \mathds{P}(Y = {y_i})\iint_{y_i \ne y_j} {Beta(\alpha ,\alpha )}\mathds{P}(Y = {y_j})d\xi dy_j \\
    &= \frac{\lambda}{e^{-\lambda}-e^{-\lambda C}} {e ^{-\lambda y_i}} , y_i \in [1,C]
\end{aligned}
\end{equation}
\end{corollary}
\begin{proof}
Follow the Basic Setting \ref{basicsetting}, when mixing factor $\xi\sim Beta(\alpha,\alpha)$, consider a \textit{$\xi$-Aug} sample generated by $\widetilde{x}_{i,j}=\xi \cdot x_i + (1-\xi) \cdot x_j$ and $\widetilde{y}_{i,j}=\xi \cdot y_i + (1-\xi) \cdot y_j$. The $\xi$-Aug sample $\widetilde{x}_{i,j}$ contributes to class $y_i$ when both pair samples $(x_i,y_i)$ and $(x_j,y_j)$ are in class $y_i$ or one of them are in other labels while the mixing factor $\xi$ is in favor of class $y_i$.
\begin{equation}
    \begin{aligned}
    \mathds{P}({\tilde y_{i,j}} &= {y_i}) =\mathds{P}({\tilde y_{i,j}} = {y_i}) \cdot \mathds{P}({\tilde y_{i,j}} = {y_i}){\mathds{P}_\xi }(0\leq \xi \leq 1) \\
    &\quad +\mathds{P}({\tilde y_{i,j}} = {y_i}) \cdot \mathds{P}({\tilde y_{i,j}} \neq {y_i}){\mathds{P}_\xi }(\xi  \geq 0.5) \\
    &\quad +\mathds{P}({\tilde y_{i,j}} \neq {y_i}) \cdot \mathds{P}({\tilde y_{i,j}} = {y_i}){\mathds{P}_\xi }(\xi  < 0.5) \\
    \end{aligned}
\end{equation}
Hence, the distribution of new dataset $D_\nu$ is:
\begin{equation}
\label{Eq.apdxmixupformu}
    \begin{aligned}
    \mathds{P}_{\textit{mixup}}(Y = {y_i}) &= \mathds{P}(Y = {y_i}) \cdot \mathds{P}(Y = {y_i}) \cdot \int_{0\le \xi \le 1}{Beta(\alpha,\alpha) d\xi} \\
    &\quad + \mathds{P}(Y = {y_i}) \cdot \int_{{y_j} \ne {y_i}} \int_{\xi \ge 0.5} {Beta(\alpha ,\alpha) \cdot \mathds{P}(Y = {y_j})d\xi d{y_j}} \\
    &\quad + \int_{{y_j} \ne {y_i}} \int_{\xi < 0.5} {Beta(\alpha ,\alpha) \cdot \mathds{P}(Y = {y_j})d\xi d{y_j}} \cdot \mathds{P}(Y = {y_i}) \\
    &= \mathds{P}^2(Y = {y_i}) + 0.5\cdot \mathds{P}(Y = {y_i}) \cdot (1-\mathds{P}(Y = {y_i})) \\
    &\quad + 0.5\cdot (1-\mathds{P}(Y = {y_i})) \cdot \mathds{P}(Y = {y_i})\\
    &=\mathds{P}^2(Y = {y_i})+ 0.5\cdot \mathds{P}(Y=y_i) -0.5\cdot \mathds{P}^2(Y=y_i) \\
    &\quad + 0.5\cdot \mathds{P}(Y=y_i) -0.5\cdot \mathds{P}^2(Y=y_i)\\
    &=\mathds{P}(Y = {y_i}) = \frac{\lambda}{e^{-\lambda}-e^{-\lambda C}} {e ^{-\lambda y_i}} 
    \end{aligned}
\end{equation}
% 因此，由mixup所生成的样本与原先的长尾数据具有相同的分布，得证。
According to Eq.\ref{Eq.LTdistribution} and Eq.\ref{Eq.apdxmixupformu}, the $\xi$-Aug samples in mixed dataset $\mathcal{D}_\nu$ generated by \textit{mixup} follow the same distribution of the original long-tailed one. Therefore, the head gets more regulation than the tail. One the one hand, the classification performance will be promoted. On the other hand, however, the performance gap between the head and tail still exists.

\end{proof}

% =============================================
\subsection{Proof of Corollary \ref{Col.2}} \label{Apdx:Coro2}
\begin{corollary}
\label{Col.2}
When $\bm{\xi_{i,j}^* \sim\mathscr{U}(\pi_{y_i},\pi_{y_j},\alpha,\alpha)},\alpha\in [0,1]$, the newly mixed dataset $\mathcal{D}_\nu$ composed of $\xi$-Aug samples $(\widetilde{x}_{i,j}, \widetilde{y}_{i,j})$ follows a middle-majority distribution, where $(x _i,y_i)$ and $(x_j,y_j)$ are both \textbf{randomly} sampled from $\mathcal{D}_{train}$.
\begin{equation}
% \small
\label{Eq.ProbIM}
\begin{aligned}
    \mathds{P}_{\textit{mixup}}^*(Y^* = {y_i}) &= \mathds{P}(Y = {y_i})\int_{y_j < {y_i}} {\mathds{1}\left(\int {\xi^*_{i,j} \mathscr{U}(\pi_i,\pi_j,\alpha ,\alpha )d\xi^*_{i,j} \geq 0.5 }\right) \mathds{P}(Y = {y_j})dy_j} \\
    & = \frac{\lambda }{{{{\left( {{e^{ - \lambda }} - {e^{ - C\lambda }}} \right)}^2}}}\left( {{e^{ - \lambda \left( {{y_i} + 1} \right)}} - {e^{ - 2\lambda {y_i}}}} \right) , y_i \in [1,C]
\end{aligned}
\end{equation}
\end{corollary}
\begin{proof}
Follow the settings of Proof \ref{Apdx:Coro1} and Eq.\ref{Eq.UnimixFactor}, we can get the following relationship of the label $y_i$ and $y_j$ with UniMix Factor $\xi_{i,j}^*$. It easy to know $\pi_{y_i}\leq \pi_{y_j}$ if the index $i\geq j$ for the reason that class $y_{j}$ occupies more instances than class $y_{i}$ in long-tailed distribution described as Eq.\ref{Eq.LTdistribution}. Under this circumstances, if consider UniMix Factor $\xi_{i,j}^*$ as a constant intuitively, we can deduce that $\xi_{i,j}=\pi_{y_j}/(\pi_{y_i}+\pi_{y_j}) \geq 0.5$, i.e.,
\begin{equation}
\label{Eq.constantunimixfactor}
    i \geq j \quad \Rightarrow \quad \pi_{y_i} \leq \pi_{y_j} \quad \Rightarrow \quad \xi_{i,j} = \frac{\pi_{y_i}}{\pi_{y_i}+\pi_{y_j}} \geq 0.5
\end{equation}

To improve the robustness and generalization, we consider $\xi_{i,j}^* \sim \mathscr{U}(\pi_{y_i},\pi_{y_j},\alpha,\alpha)$, which extends $\xi_{i,j}^*$ to $\pi_{y_j}/(\pi_{y_i}+\pi_{y_j})$ and its vicinity. Hence we can generalize Eq.\ref{Eq.constantunimixfactor} to: 
\begin{equation}
    y_i \geq y_j \quad \Rightarrow \quad \pi_{y_i} \leq \pi_{y_j} \quad \Rightarrow \quad \xi_{i,j} \geq 0.5 \Rightarrow \int {\xi^*_{i,j} \mathscr{U}(\pi_i,\pi_j,\alpha ,\alpha )d\xi^*_{i,j} \geq 0.5}
\end{equation}
Given any $i,j \in [1,C]$, the $\widetilde{x}_{i,j}$ will tend to be a $\xi$-Aug sample of the class $y_i$ if $i\geq j$ for the tail-favored mixing factor. In other words, $\widetilde{x}_{i,j}$ will be a \textit{$\xi$-Aug} sample for class $y_i$ with the mean of $\xi_{i,j}^*$ to be $\pi_{y_j}/(\pi_{y_i}+\pi_{y_j})$. Hence, the probability that {$\widetilde{x}_{i,j}$} belongs to class $y_i$ is:
\begin{equation}
\label{Eq.unimixfactordistribution}
    \begin{aligned}
    \mathds{P}^*_{\textit{mixup}}(Y = y_i) &= \mathds{P}(Y = y_i) \cdot \mathds{1}\left( {\int {\xi_{i,j}^*\mathscr{U}({\pi _i},{\pi _j},\alpha ,\alpha )d\xi_{i,j}^* \geq 0.5} } \right) \cdot \mathds{P}(Y < {y_i}) \\
    &=\mathds{P}(Y = {y_i})\int_{y_j < {y_i}} {\mathds{1}\left(\int {\xi^*_{i,j} \mathscr{U}(\pi_i,\pi_j,\alpha ,\alpha )d\xi^*_{i,j} \geq 0.5 }\right) \mathds{P}(Y = {y_j})dy_j} \\
    &=\mathds{P}(Y = {y_i}) \cdot \int_1^{y_i} {\mathds{P}(Y = {y_j})d{y_j}} \\
    &=\frac{\lambda }{{{e^{ - \lambda }} - {e^{ - \lambda C}}}}{e^{-\lambda {y_i}}} \cdot \int_1^{y_i} {\frac{\lambda }{{{e^{ - \lambda }} - {e^{ - \lambda C}}}}{e^{-\lambda {y_j}}}d{y_j}}  \\
    &= \frac{\lambda }{{{e^{ - \lambda }} - {e^{ - \lambda C}}}}{e^{-\lambda {y_i}}} \cdot \left. { - \frac{{\frac{\lambda }{{{e^{ - \lambda }} - {e^{ - \lambda C}}}}}}{\lambda }{e^{ - \lambda {y_j}}}} \right|_1^{y_i} \\
    &=\frac{\lambda }{{{{\left( {{e^{ - \lambda }} - {e^{ - C\lambda }}} \right)}^2}}}\left( {{e^{ - \lambda ({y_i} + 1)}} - {e^{ - 2\lambda {y_i}}}} \right)
    \end{aligned}
\end{equation}
According to Eq.\ref{Eq.unimixfactordistribution}, it's easy to find a derivative zero point in range [1,C]. Hence, the newly distributed dataset $\mathcal{D}_\nu$ generated by \textit{mixup} with UniMix Factor $\xi^*_{i,j}$ follows a middle-majority distribution that most data concentrates on middle classes.
\end{proof}

% =============================================
\subsection{Proof of Corollary \ref{Col.3}} \label{Apdx:Coro3}
\begin{corollary}
\label{Col.3}
When $\bm{\xi_{i,j}^* \sim \mathscr{U}(\pi_{y_i},\pi_{y_j},\alpha,\alpha)},\alpha\in [0,1]$, the newly mixed dataset $\mathcal{D}_\nu$ composed of $\xi$-Aug samples $(\widetilde{x}_{i,j}, \widetilde{y}_{i,j})$ follows a tail-majority distribution, where $(x _i,y_i)$ is \textbf{randomly} and $(x_j,y_j)$ is \textbf{inversely} sampled from $\mathcal{D}_{train}$, respectively.
\begin{equation}
% \small
    \label{Eq.Pours}
    \begin{aligned}
        {\mathds{P}_{UniMix}}({Y^*} = {y_i}) &= \mathds{P}(Y = {y_i})\int_{y_j < {y_i}} { \mathds{1}\left(\int {\xi^*_{i,j} \mathscr{U}(\pi_i,\pi_j,\alpha ,\alpha )d\xi^*_{i,j} \geq 0.5 }\right) \mathds{P}_{inv}(Y=y_j)dy_j}  \\
        &= \frac{\lambda }{{\left( {{e^{ - \lambda }} - {e^{ - C\lambda }}} \right)\left( {{e^{ - C\tau \lambda }} - {e^{ - \tau \lambda }}} \right)}}\left( {{e^{ - \lambda {y_i}\left( {\tau  + 1} \right)}} - {e^{ - \lambda \left( {\tau  + {y_i}} \right)}}} \right), y_i \in [1,C]
    \end{aligned}
\end{equation}
\end{corollary}
\begin{proof}
Follow the Basic Setting of Proof.\ref{basicsetting}, suppose $(x_i,y_i)$ is randomly sampled from $\mathcal{D}_{train}$, while $(x_j,y_j)$ is inversely sampled from $\mathcal{D}_{train}$ related to $\tau$. i.e., the probability of $\mathds{P}(Y=y_i)$ is:
\begin{equation}
    \begin{aligned}
    &\mathds{P}(Y=y_i)=\frac{\lambda }{{{e^{ - \lambda }} - {e^{ - \lambda C}}}}{e^{-\lambda {y_i}}} \\
    \end{aligned}
\end{equation}
The probability of UniMix Sampler that the sampled class $Y$ belongs to $y_i$ is $\mathds{P}_{inv}(Y=y_j)$:
\begin{equation}
    \begin{aligned}
    {\mathds{P}_{inv}}(Y = {y_j}) &= \frac{{{\mathds{P}^\tau }(Y = {y_j})}}{{\int_{{y_k} \in \mathcal{Y}} {{\mathds{P}^\tau }(Y = {y_k})d{y_k}} }} \\
    &= \frac{{{{\left( {\frac{\lambda }{{{e^{ - \lambda }} - {e^{ - \lambda C}}}}{e^{ - \lambda {y_j}}}} \right)}^\tau }}}{{\int_{{y_k} \in \mathcal{Y}} {{{\left( {\frac{\lambda }{{{e^{ - \lambda }} - {e^{ - \lambda C}}}}{e^{ - \lambda {y_k}}}} \right)}^\tau }d{y_k}} }} \\
    &=\frac{{{e^{ - \lambda \tau {y_j}}}}}{{\int_1^C {{e^{ - \lambda \tau {y_k}}}d{y_k}} }} = \frac{{\lambda \tau {e^{ - \lambda \tau {y_j}}}}}{{{e^{ - \lambda \tau }} - {e^{ - \lambda \tau C}}}}
    \end{aligned}
\end{equation}
Follow the same proof as Proof.\ref{Apdx:Coro2}, the probability that {$\widetilde{x}_{i,j}$} belongs to class $y_i$ is:
\begin{equation}
\label{Eq.apdxunimixdistribution}
    \begin{aligned}
    \mathds{P}_{\textit{UniMix}}(Y = y_i) &= \mathds{P}(Y = y_i) \cdot \mathds{1}\left( {\int {\xi_{i,j}^*\mathscr{U}({\pi _i},{\pi _j},\alpha ,\alpha )d\xi_{i,j}^* \geq 0.5} } \right) \cdot \mathds{P}_{inv}(Y < {y_i}) \\
    &=\mathds{P}(Y = {y_i})\int_{y_j < {y_i}} {\mathds{1}\left(\int {\xi^*_{i,j} \mathscr{U}(\pi_i,\pi_j,\alpha ,\alpha )d\xi^*_{i,j} \geq 0.5 }\right) \mathds{P}_{inv}(Y = {y_j})dy_j} \\
    &=\mathds{P}(Y = {y_i}) \cdot \int_1^{y_i} {\mathds{P}_{inv}(Y = {y_j})d{y_j}} \\
    &=\frac{\lambda }{{{e^{ - \lambda }} - {e^{ - \lambda C}}}}{e^{-\lambda {y_i}}} \cdot \int_1^{y_i} { \frac{{\lambda \tau {e^{ - \lambda \tau {y_j}}}}}{{{e^{ - \lambda \tau }} - {e^{ - \lambda \tau C}}}}d{y_j}}  \\
    &=\frac{\lambda }{{{e^{ - \lambda }} - {e^{ - \lambda C}}}}{e^{ - \lambda {y_i}}} \cdot \left. {\frac{{ - {e^{ - \lambda \tau {y_j}}}}}{{{e^{ - \lambda \tau }} - {e^{ - \lambda \tau C}}}}} \right|_1^{y_i} \\
    &= \frac{\lambda }{{\left( {{e^{ - \lambda }} - {e^{ - C\lambda }}} \right)\left( {{e^{ - C\tau \lambda }} - {e^{ - \tau \lambda }}} \right)}}\left( {{e^{ - \lambda {y_i}\left( {\tau  + 1} \right)}} - {e^{ - \lambda \left( {\tau  + {y_i}} \right)}}} \right)
    \end{aligned}
\end{equation}
According to Eq.\ref{Eq.apdxunimixdistribution}, $\mathds{P}_{\textit{UniMix}}(Y = y_i)$ is a gently increasing function in range [1,C]. Hence, the newly mixed dataset $\mathcal{D}_\nu$ generated by UniMix follows a tail-majority distribution that adequate data concentrates on tail classes.

\end{proof}

\subsection{More visualization of Corollary \ref{Col.1},\ref{Col.2},\ref{Col.3}} \label{Apdx:addcurve}
\begin{figure}[h!]
    \centering
    \includegraphics[width=0.8\textwidth]{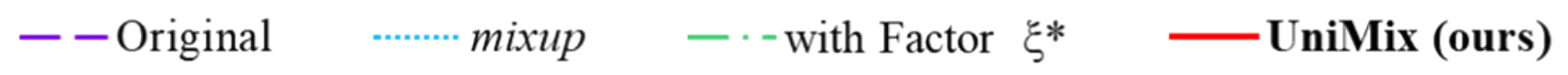}
    \subfigure[$C=10,\rho=10$]{
    \begin{minipage}[b]{0.315\textwidth}
        \centering
        \includegraphics[width=1\textwidth]{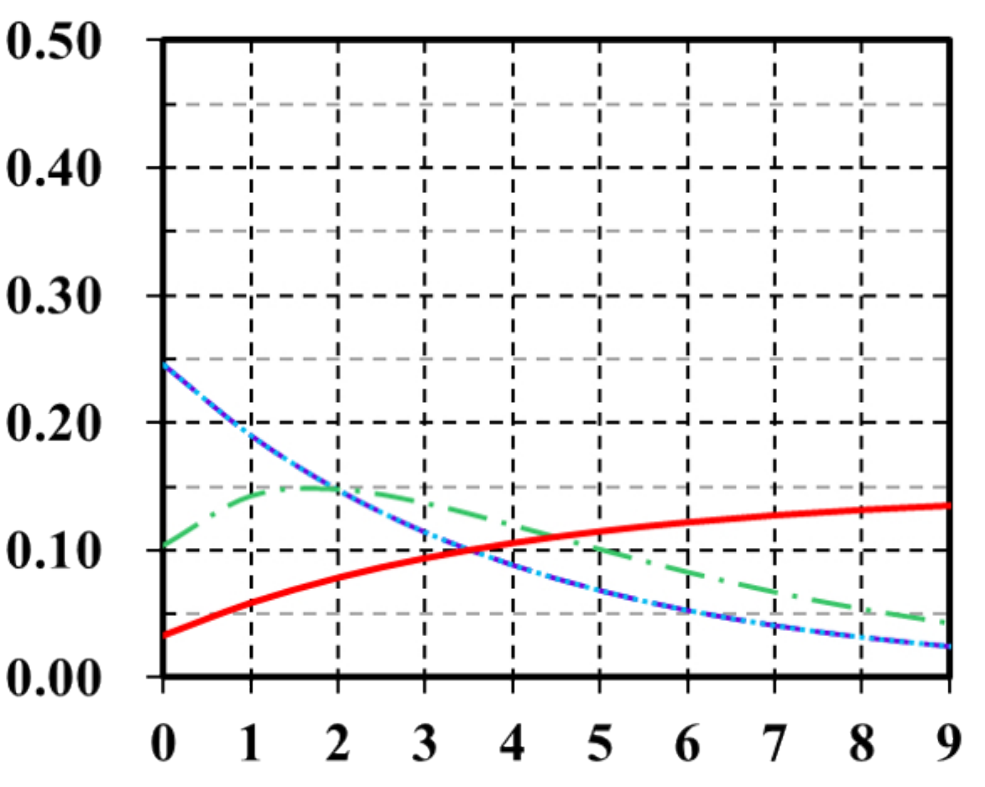}
        %\caption{fig1}
    \end{minipage}%
    }
    \subfigure[$C=10,\rho=100$]{
    \begin{minipage}[b]{0.315\textwidth}
        \centering
        \includegraphics[width=1\textwidth]{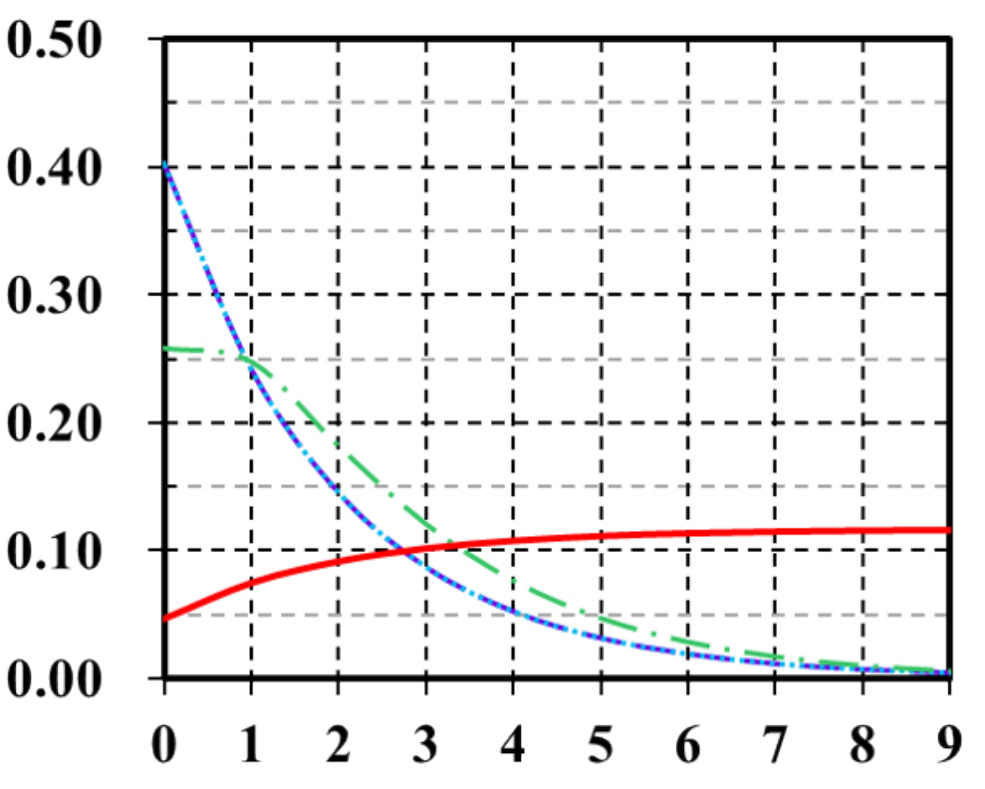}
        %\caption{fig2}
    \end{minipage}%
    }
    \subfigure[$C=10,\rho=200$]{
    \begin{minipage}[b]{0.315\textwidth}
        \centering
        \includegraphics[width=1\textwidth]{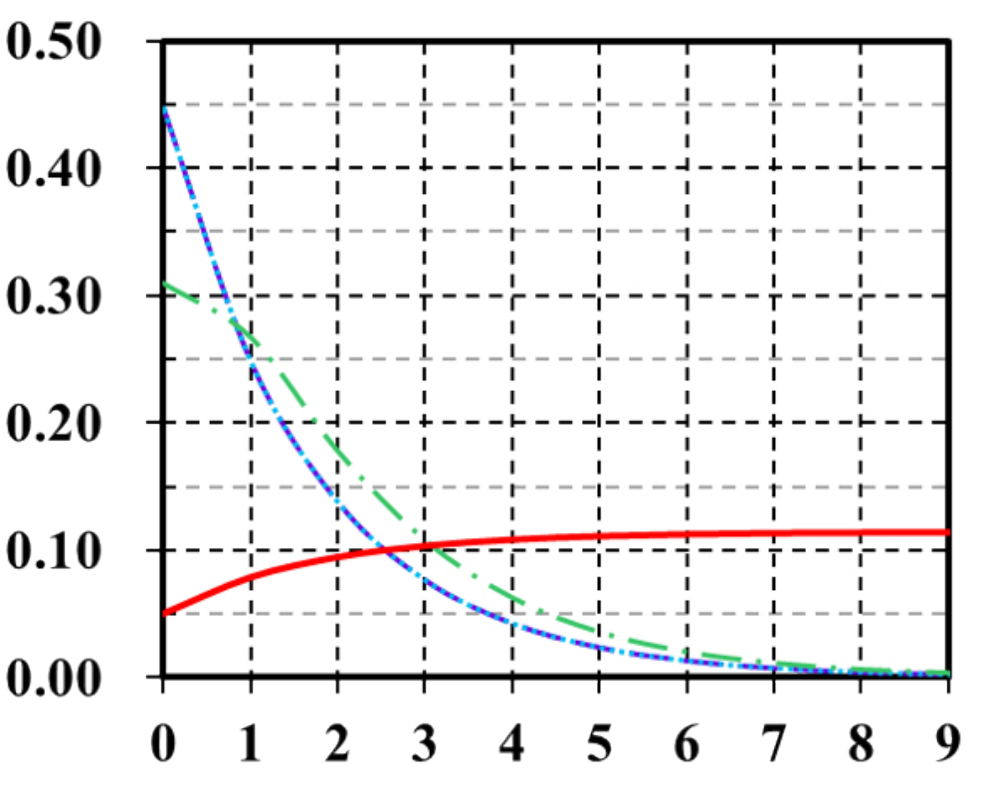}
        %\caption{fig2}
    \end{minipage}
    }%
    
    \subfigure[$C=100,\rho=10$]{
    \begin{minipage}[b]{0.315\textwidth}
        \centering
        \includegraphics[width=1\textwidth]{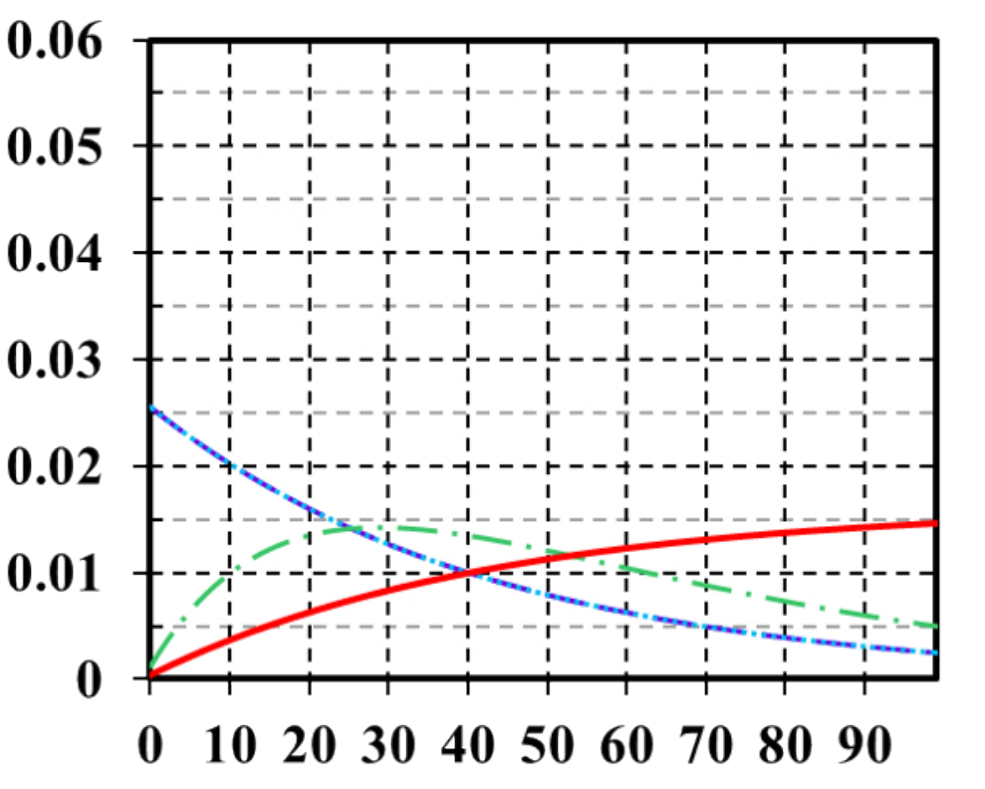}
    \end{minipage}
    }
    \subfigure[$C=100,\rho=100$]{
    \begin{minipage}[b]{0.315\textwidth}
        \centering
        \includegraphics[width=1\textwidth]{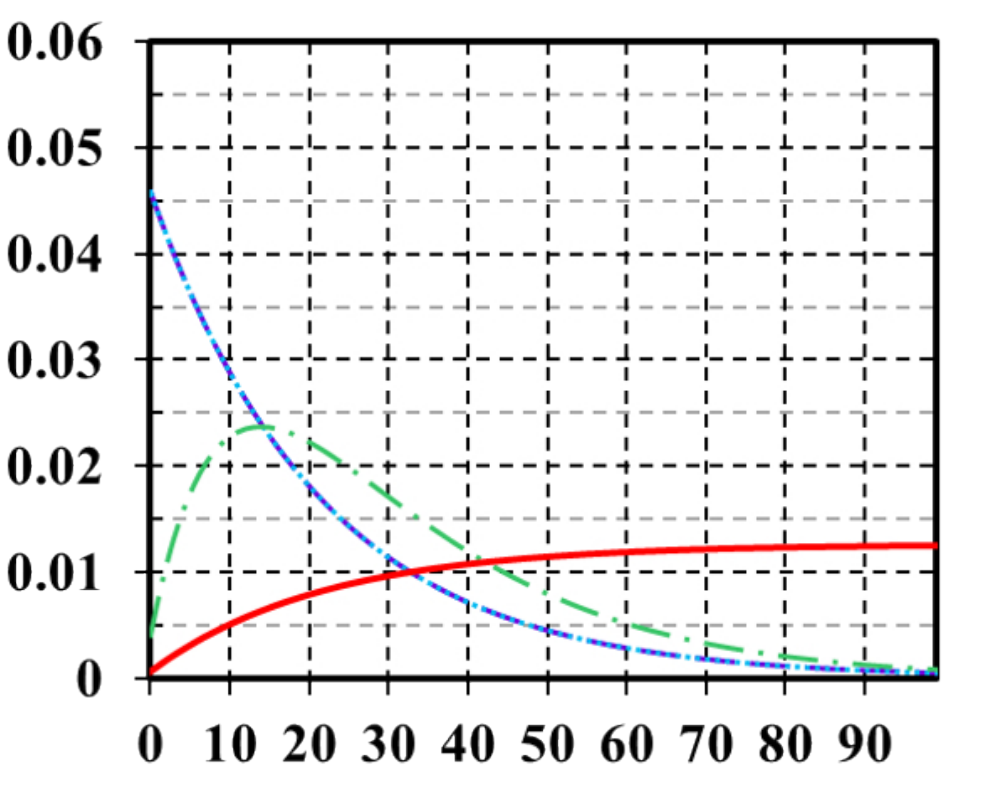}
    \end{minipage}
    }
    \subfigure[$C=100,\rho=200$]{
    \begin{minipage}[b]{0.315\textwidth}
        \centering
        \includegraphics[width=1\textwidth]{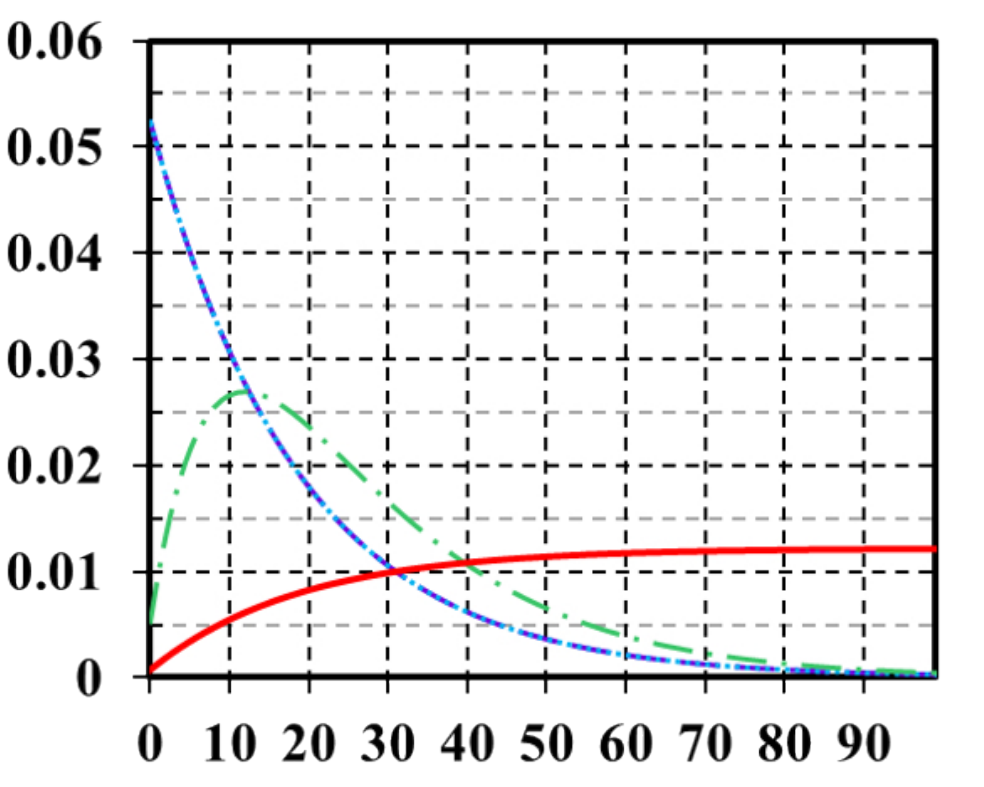}
    \end{minipage}
    }%
    \centering
    \caption{Additional visualized comparisons of \textit{$\xi$-Aug} samples distribution in Corollary \ref{Col.1},\ref{Col.2},\ref{Col.3}. $x$-axis: class indices. $y$-axis: probability of each class. \textit{mixup} (\textcolor{curvemixup}{blue}) exhibits the same LT distribution as origin (\textcolor{curveori}{purple}). UniMix Factor (\textcolor{curvefactor}{green}) alleviates such situation and the full pipeline ($\tau$=$-1$) constructs a more uniform distribution of \textit{$\xi$-Aug} (\textcolor{curveunimix}{red}), which contributes to a well-calibrated model. }
    \label{Fig.addcurve}
    % \vspace{-5pt}
\end{figure}

We present additional visualized distribution of \textit{$\xi$-Aug} samples with different sample strategies for comprehensive comparisons, including $\rho \in \{10,100,200\}$ and $C \in \{10,100\}$. As illustrated in Fig.\ref{Fig.addcurve}, when the class number $C$ and imbalance factor $\rho$ get larger, the limitations of \textit{mixup} in LT scenarios gradually appear. The VRM dataset $\mathcal{D}_\nu$ generated by \textit{mixup} with na\"ive mixing factor and random sampler will make $\mathcal{D}_\nu$ following the same LT distribution. The newly mixed pseudo data by \textit{mixup} follows the head-majority distribution. It has limited contribution for the tail class' feature learning and regulation, which is the reason for its poor \textit{calibration}.

In contrast, the proposed UniMix Factor $\xi^*_{i,j}$ significantly promotes such situation and improves tail class's feature learning by the preference on the relatively fewer classes of the two samples in a pair. Because of the random sampler, the samples are still mainly from the head, and the newly mixed dataset will follows a middle-majority distribution. Most data concentrates on the middle of classes as green line shows. As $C$ and $\rho$ get larger, the middle distribution will get close to the head. Thanks to the UniMix Sampler that inversely draws data from $\mathcal{D}_{train}$, the $D_\nu$ is mainly composed of the head-tail 
pairs that contribute to the feature learning for the tail when integrated with UniMix Factor. As a result, $D_\nu$ follows the tail-majority distribution that improves the generalization on the tail. Specifically, when $C$ and $\rho$ get larger, the distribution of $D_\nu$ generated by UniMix still maintains satisfactory tail-majority distribution.

\setcounter{equation}{0}
\setcounter{subsection}{0}
\setcounter{table}{0}  
\setcounter{figure}{0}
\setcounter{corollary}{0}

\renewcommand{\thetable}{B\arabic{table}}
\renewcommand{\thefigure}{B\arabic{figure}}
\renewcommand{\theequation}{B.\arabic{equation}}

\section{Missing Proofs and derivations of Bayias} \label{Apdx:baybiasprove}

\subsection{Proof of Bayias} \label{Apdx:baybiasexist}
\begin{thm}
    For classification, let $\psi(x;\theta,W,b)$ be a hypothesis class of neural networks of input $X=x$, the classification with \textit{Softmax} should contain the influence of prior, i.e., the predicted label during training should be:
    \begin{equation}
    % \small
\label{Eq.apdxthmbias}
    \begin{aligned}
    \hat{y}=\mathop{\arg\max}_{y_i \in \mathcal{Y}}\frac{{{e^{\psi {{(x;\theta ,W,b)}_{{y_i}}}+\bm{\log(\pi_{y_i})+\log (C)}}}}}{{\sum\nolimits_{{y_j} \in \mathcal{Y}} {{e^{\psi {{(x;\theta ,W,b)}_{{y_j}}}+\bm{\log(\pi_{y_j})+\log (C)}}}}}}
    \end{aligned}
\end{equation}
\end{thm}

\begin{proof}
Generally, a classifier can be modeled as:
\begin{equation}
% \small
\label{Eq.apdxOrigalOptim}
    \begin{aligned}
    \hat{y}=\mathop{\arg\max}_{y_i \in \mathcal{Y}}\frac{e^{\sum\nolimits_{{d_i} \in D} [{{{({W^T})_{{y_i}}^{({d_i})}\mathcal{F}{{(x;\theta )}^{(d_i)}}] + {b_{{y_i}}}}}} }}{{\sum\nolimits_{{y_j} \in \mathcal{Y}} e^{\sum\nolimits_{{d_i} \in D} [{{{({W^T})_{{y_j}}^{({d_i})}\mathcal{F}{{(x;\theta )}^{(d_i)}}] + {b_{{y_j}}}}}} } }} \triangleq \mathop{\arg\max}_{y_i \in \mathcal{Y}}\frac{{{e^{\psi {{(x;\theta ,W,b)}_{{y_i}}}}}}}{{\sum\nolimits_{{y_j} \in \mathcal{Y}} {{e^{\psi {{(x;\theta ,W,b)}_{{y_j}}}}}} }}
    \end{aligned}
\end{equation}
where $\hat{y}$ indicates the predicted label and $\mathcal{F}(x;\theta)\in \mathds{R}^{D\times1}$ is the $D$-dimension feature extracted by the backbone with parameter $\theta$. $W \in \mathds{R}^{D\times C}$ represents the parameter matrix of the classifier. The model attempts to get the maximum $\hat{y}$ given $x$, i.e., to maximize the \textit{posterior}, satisfying the following relationship between the \textit{prior} and \textit{likelihood} according to the Bayesian theorem:
\begin{equation}
\begin{aligned}
        % \mathds{P}_{train} (y|x) &= \frac{\mathds{P}_{train}(x|y)\cdot \mathds{P}_{train}(y)}{\mathds{P}_{train}(x)} \\
        \hat{y}&=\mathop{\arg\max}_{y_i \in \mathcal{Y}}{\mathds{P}}(Y = {y_i}|X = {x})\\
        &= \mathop{\arg\max}_{y_i \in \mathcal{Y}}\frac{{{\mathds{P}}(X = {x}|Y = {y_i}) \cdot {\mathds{P}}(Y = {y_i})}}{{\sum\nolimits_k {\mathds{P}(Y = {y_k})\prod\nolimits_j {\mathds{P}({X^{(j)}} = {x^{(j)}}|Y = {y_k})} } }} \\
        & \propto \mathop{\arg\max}_{y_i \in \mathcal{Y}}\mathds{P}(X=x|Y=y_i)\cdot \mathds{P}(Y=y_i)
\end{aligned}
\end{equation}
where $\sum\nolimits_k {\mathds{P}(Y = {y_k})\prod\nolimits_j {\mathds{P}({X^{(j)}} = {x^{(j)}}|Y = {y_k})} }$ is the \textit{normalized evidence factor}. $\mathds{P}(Y=y_i)$ is the \textit{prior} probability estimated by the instance proportion of each category in the dataset. To maximize \textit{posterior}, we need to get the $\mathds{P}(X|Y)$ in an ERM supervised training manner. However, in LT scenarios, the \textit{likelihood} is consistent in the train and test set, but the \textit{prior} is different. Hence, we derive the \textit{posterior} on train and test set separately:
\begin{equation}
        \left\{
\begin{aligned}
&\hat{y}\propto \mathop{\arg\max}_{y_i \in \mathcal{Y}}\mathds{P}(X=x|Y=y_i)\cdot \mathds{P}_{train}(Y=y_i)\\
&\hat{y}'\propto \mathop{\arg\max}_{y_i \in \mathcal{Y}}\mathds{P}(X=x'|Y'=y_i)\cdot \mathds{P}'_{test}(Y'=y_i)
\end{aligned}
\right.
\end{equation}
where $\hat{y},\hat{y}'$ represent the prediction results for the train and test set, respectively. The model $\psi(x;\theta,W,b)$ is just the \textit{likelihood} estimation $\mathds{P}(X=x|Y=y_i)$ of the train set. To obtain the posterior probability $y'=\mathop{\arg\max}_{y_i \in \mathcal{Y}}{\mathds{P}'(Y=y_i|X=x')}$ for inference, one should consider unifying the optimization direction on the train set and test set:
% where $\hat{y},\hat{y}'$ represent the prediction results and $\mathds{P}(Y=y_i),\mathds{P}( Y'=y_i)$  \textit{prior} of the train and test set, respectively. $\mathds{P}(X=x|Y=y_i)$ represents the output result of the model. The supervised model $\psi(x;\theta,W,b)$ is just the \textit{likelihood} estimation $\mathds{P}(X=x|Y=y_i)$ of the train set. To obtain the posterior probability $y'=\mathop{\arg\max}_{y_i \in \mathcal{Y}}{\mathds{P}'(Y=y_i|X=x')}$ for inference, one shell consider unifying the model optimization direction on the train set and test set:
\begin{equation}
\begin{aligned}
    \mathds{P}_{test}(Y=y_i|X=x')&\propto \frac{\mathds{P}_{train}(Y=y_i|X=x) }{\mathds{P}_{train}(Y=y_i)} \cdot \mathds{P}_{test}(Y'=y_i) =\frac{\mathds{P}_{train}(Y=y_i|X=x)}{C\cdot \pi_{y_i}} \\
\end{aligned}
\end{equation}
For the difference of label \textit{prior}, the learned parameters of the model will also yield class-level bias. Hence, the actual optimization direction is not described as Eq.\ref{Eq.apdxOrigalOptim} because the bias incurred by \textit{prior} should be compensated at first.
\begin{equation}   
\begin{aligned}
    \hat{\theta},\hat{W},\hat{b} \triangleq \Theta &= \mathop{\arg \min}_{\Theta}\sum\nolimits_{\substack{{x_i\in \mathcal{X}'\land y_i \in \mathcal{Y}}}} \mathds{1}\left({y_i} \neq \mathop{\arg \max}_{y_j \in \mathcal{Y}}({\mathds{P}_{test}(Y = {y_j}|X = {x_i})}\right) \\
    &\Leftrightarrow \mathop{\arg \min}_{\Theta}\sum\nolimits_{\substack{{x_i\in \mathcal{X} \land y_i \in \mathcal{Y}}}} \mathds{1}\left ({y_i} \neq \mathop{\arg \max}_{y_j \in \mathcal{Y}}\frac{\mathds{P}_{train}(Y=y_j|X=x_i) }{C\cdot \pi_{y_j}}\right ) \\
    &= \mathop{\arg \min}_{\Theta}\sum\nolimits_{\substack{{x_i\in \mathcal{X} \land y_i \in \mathcal{Y}}}} \mathds{1}\left ({y_i} \neq \mathop{\arg \max}_{y_j \in \mathcal{Y}}\frac{e^{\psi(x_i;\theta,W,b)_{y_j}-\bm{\log C-\log (\pi_{y_j})}}}{{\sum\nolimits_{{y_k} \in \mathcal{Y}} {e^{\psi {{(x_i;\theta ,W,b)}_{{y_k}}}}} }}\right ) \\
    &\Leftrightarrow \mathop{\arg \min}_{\Theta}\sum\nolimits_{\substack{{x_i\in \mathcal{X} \land y_i \in \mathcal{Y}}}} \mathds{1}\left ({y_i} \neq \mathop{\arg \max}_{y_j \in \mathcal{Y}}\frac{e^{\psi(x_i;\theta,W,b)_{y_j}-\bm{\log (\pi_{y_j})-\log C}}}{{\sum\nolimits_{{y_k} \in \mathcal{Y}} {e^{\psi {{(x_i;\theta ,W,b)}_{{y_k}}}-\bm{\log (\pi_{y_k})-\log C}}} }}\right )
\end{aligned} 
\end{equation}
To correct the bias for inferring, the offset term that the model in LT datasets needs to compensate is:
\begin{equation}
\label{Eq.apdxbiascompensate2}
    \mathscr{B}_y = log(\pi_y) + log(C)
\end{equation}
\end{proof}
\subsection{Proof of classification \textit{calibration}} \label{apdx:baybiascaliprove}
\begin{thm}
\label{Thm.apdxCalibration1}
    $\mathscr{B}_y$-compensated cross-entropy loss in Eq.\ref{Eq.b-ce2} ensures classification \textit{calibration}. 
\begin{equation}
\label{Eq.b-ce2}
\begin{aligned}
\mathcal{L}_\mathscr{B}\left( {y_i,\psi (x;\Theta)} \right) =  - \log \frac{{{e^{\psi {{(x;\Theta)}_{{y_i}}} + \log \left( {{\pi _{{y_i}}}} \right) + \log \left( C \right)}}}}{{\sum\nolimits_{{y_k} \in \mathcal{Y}} {{e^{\psi {{(x;\Theta)}_{y_j}} + \log \left( {{\pi _{{y_k}}}} \right) + \log \left( C \right)}}} }}
\end{aligned}
\end{equation}
\end{thm}
Classification \textit{calibration} \cite{Icml/Calibration-NN, Nips/On_Mixup_Training} represents the predicted winning \textit{Softmax} scores indicate the actual likelihood of a correct prediction. The miscalibrated models tend to be overconfident which results in that the minimiser of the expected loss (equally, the empirical risk in the infinite sample limit) can not lead to a minimal classification error. Previous work \cite{Aaai/logit_adjustment} has introduced a theory to measure the \textit{calibration} of a pair-wise loss:
\begin{lemma}
\label{Lem.Cali}
Pairwise loss $\mathcal{L}(y_i,\psi(x; \Theta)) = \alpha_{y_i}\cdot \log \left[1+\sum_{y_k \neq y_i} e^{\Delta_{y_i y_k}}\cdot e^{(\psi(x;\Theta)_{y_i} - \psi(x,\Theta)_{y_k})}\right]$ ensures classification \textit{calibration} if for any $\delta \in \mathds{R}_+^C$:
\begin{equation}
    \alpha_{y_i} = \delta_{y_i}/\pi_{y_i} \ \ \ \  \Delta{y_i y_k} = \log(\delta_{y_k}/\delta_{y_i})
\end{equation}
\end{lemma}
\begin{proof}
To begin, we rewritten the Bayias-compensated cross-entropy loss into the following equation:
\begin{equation}
    \begin{aligned}
    {\mathcal{L}_\mathscr{B}}\left( {y_i,\psi (x;\Theta)} \right) &=  - \log \frac{{{e^{\psi {{(x;\Theta)}_{{y_i}}} + \log \left( {{\pi _{{y_i}}}} \right) + \log \left( C \right)}}}}{{\sum\nolimits_{{y_k} \in \mathcal{Y}} {{e^{\psi {{(x;\Theta)}_{y_j}} + \log \left( {{\pi _{{y_k}}}} \right) + \log \left( C \right)}}} }} \\
    & \Leftrightarrow \log \frac{{\sum\nolimits_{{y_k} \in \mathcal{Y}} {{e^{\psi {{(x;\Theta)}_{y_k}} + \log \left( {{\pi _{{y_k}}}} \right) + \log \left( C \right)}}} }} {{{e^{\psi {{(x;\Theta)}_{{y_i}}} + \log \left( {{\pi _{{y_i}}}} \right) + \log \left( C \right)}}}} \\
    &=\log \left[ {1 + \frac{{\sum\nolimits_{{y_k} \ne {y_i}} {{e^{\psi {{(x;\Theta)}_{{y_k}}} + \log \left( {{\pi _{{y_k}}}} \right) + \log \left( C \right)}}} }}{{{e^{\psi {{({x};\Theta)}_{{y_i}}} + \log \left( {{\pi _{{y_i}}}} \right) + \log \left( C \right)}}}}} \right] \\
    &=\log \left[ {1 + \sum\nolimits_{{y_k} \ne {y_i}} {\frac{{{e^{\psi {{(x;\Theta)}_{{y_k}}} + \log \left( {{\pi _{{y_k}}}} \right) + \log \left( C \right)}}}}{{{e^{\psi {{({x};\Theta)}_{{y_i}}} + \log \left( {{\pi _{{y_i}}}} \right) + \log \left( C \right)}}}}} } \right] \\
    &=\log \left[ {1 + \sum\nolimits_{{y_k} \ne {y_i}} {\frac{{{e^{\log \left( {{\pi _{{y_k}}}} \right) + \log \left( C \right)}} \cdot {e^{\psi {{(x;\Theta)}_{{y_k}}}}}}}{{{e^{\log \left( {{\pi _{{y_i}}}} \right) + \log \left( C \right)}} \cdot {e^{\psi {{({x};\Theta)}_{{y_i}}}}}}}} } \right] \\
    &=\log \left[ {1 + \sum\nolimits_{{y_k} \ne {y_i}} {{e^{\mathscr{B}_{y_k} - \mathscr{B}_{y_i}}} \cdot {e^{\psi {{(x;\Theta)}_{{y_k}}} - \psi {{({x};\Theta)}_{{y_i}}}}}} } \right]
    \end{aligned}
    \label{Eq.apdxchangeformatloss}
\end{equation}
Compare Eq.\ref{Eq.apdxchangeformatloss} with Lemma\ref{Lem.Cali}, observed that when $\delta_y=\pi_y$.
\begin{equation}
\left\{
             \begin{array}{lr}
             \alpha_{y_i} =\delta_{y_i}/\pi_{y_i}= \pi_{y_i} / \pi_{y_i} = 1 \\
             \begin{aligned}
                \Delta_{y_i y_k} &= \mathscr{B}_{y_k}-\mathscr{B}_{y_i} \\
                &= [log(\pi_{y_k}) + log(C)] - [log(\pi_{y_i}) + log(C)] \\
                &= log(\pi_{y_k})-log(\pi_{y_i}) = log(\pi_{y_k}/\pi_{y_i})
                \end{aligned}
             \end{array}
\right.
\end{equation}
According to Lemma\ref{Lem.Cali}, we immediately deduce that Bayias-compensated cross-entropy loss ensures classification \textit{calibration}.
\end{proof}

\subsection{Comparisons with other losses}
Previous work \cite{Tnn/CSCE, Cvpr/CB, Cvpr/EQL, Corr/CDT, Nips/LDAM, Aaai/logit_adjustment} adjusts the logits \textit{weight} or \textit{margin} on standard \textit{Softmax} cross-entropy (CE) loss to tackle the long-tailed datasets. We summarize the loss modification methods reported in Tab.\ref{Tab.ResultCifar} and discuss the difference between theirs and ours here.
\paragraph{Weight-wise losses.} 
Focal loss \cite{Iccv/Focal} is proposed to balance the positive/negative samples during object detection and extends for classification by assigning low weight loss for easy samples. It re-wights the loss with a factor $(1-p_{y_i})^{\gamma}$ on standard cross-entropy loss, where $p_{y_i}$ is the prediction probability of class $y_i$:
\begin{equation}
\mathcal{L}_{Focal} = - (1-p_{y_i})^{\gamma} \log \frac{e^{\psi(x;\Theta)_{y_i}}} {\sum\nolimits_{{y_k} \in \mathcal{Y}} {{e^{\psi {(x;\Theta)}_{y_k}}}}}
\end{equation}
CB loss is proposed by Cui \textit{et al.} \cite{Cvpr/CB} with the \textit{effective number}. It adopts a coefficient $(1-\beta)/(1-\beta^{n_{y_i}})$ for standard cross-entropy loss:
\begin{equation}
\mathcal{L}_{CB} = - \frac{1-\beta}{1-\beta^{n_{y_i}}} \log \frac{e^{\psi(x;\Theta)_{y_i}}} {\sum\nolimits_{{y_k} \in \mathcal{Y}} {{e^{\psi {(x;\Theta)}_{y_k}}}}}
\end{equation}
CDT loss \cite{Corr/CDT} re-weights each \textit{Softmax} logit with a temperature scale factor $a_{y_i} = (\frac{n_{max}}{n_{y_i}})^\gamma$. It artificially reduces the decision values for head classes:
\begin{equation}
\mathcal{L}_{CDT} = -  \log \frac{ e^{\psi(x;\Theta)_{y_i}/ (\frac{n_{max}}{n_{y_i}})^\gamma } } {\sum\nolimits_{{y_k} \in \mathcal{Y}} {{e^{\psi {(x;\Theta)}_{y_k}/(\frac{n_{max}}{n_{y_k}})^\gamma }}}}
\end{equation}
All above re-weight methods are proven effective empirically, more or less. However, such approaches will confront the coverage dilemma when the train data gets highly imbalanced. The weights related to instances number may be large in this situation and result in unstable gradients that deteriorate head classes' performance gain. They are also sensitive to hyper-parameters, which makes it hard to adopt in varied datasets.
\paragraph{Margin-wise losses.} 
Previous work in Deep Metrics Learning \cite{Cvpr/SphereFace, Cvpr/ArcFace, Mm/norm_face, Cvpr/AdaptiveFace} attempts to obtain better inter-class and intra-class distance from the margin perspective. Cao \textit{et al.} \cite{Nips/LDAM} analyze the optimal margin between two different classes via generalization error bounds in the long tail visual recognition. They propose the LDAM loss to encourage the tail classes to enjoy larger margins. In detail, a label-aware margin $C/n_{y_i}^{1/4}$ is added to the groud truth logit where $C$ is an independent constant:
\begin{equation}
\mathcal{L}_{LDAM} = - \log \frac{e^{\psi(x;\Theta)_{y_i} - C/n_{y_i}^{1/4}}} {e^{\psi(x;\Theta)_{y_i} - C/n_{y_i}^{1/4}} + \sum_{{y_k} \neq {y_i}} {{e^{\psi {(x;\Theta)}_{y_k}}}}}
\end{equation}
Logit Adjustment \cite{Aaai/logit_adjustment} loss is proposed to overcome the long-tailed dataset motivated by the balanced error rate. They suppose that the model shows similar performance on each class in the validation dataset. The authors propose a margin $\tau \log(\pi_{y_i})$ on standard \textit{Softmax} cross-entropy loss. Different from the LDAM loss, such margin is added for all logits with label priors $\pi_{y}$:
\begin{equation}
\mathcal{L}_{LA} = - \log \frac{e^{\psi(x;\Theta)_{y_i} + \tau\log(\pi_{y_i}) }} {\sum\nolimits_{{y_k} \in \mathcal{Y}} {{e^{\psi {(x;\Theta)}_{y_k} + \tau\log(\pi_{y_k}) }}}}
\end{equation}
Our Bayias-compensated CE loss is motivated by erasing the difference of label \textit{prior} in the train set and test set based on the Bayesian theory. The network parameters are only the \textit{likelihood} estimation, while we need a reliable \textit{posterior} for unbiased inference. Hence, we compensate the bias incurred by various label \textit{prior} via adding a margin related to the train set label \textit{prior} and minus a margin related to the test set label \textit{prior}:
\begin{equation}
\mathcal{L}_{ours} = - \log \frac{e^{\psi(x;\Theta)_{y_i} + \log(\pi_{y_i}) - \log(1/C)}} {\sum\nolimits_{{y_k} \in \mathcal{Y}} {{e^{\psi {(x;\Theta)}_{y_k} + \log(\pi_{y_k}) - \log(1/C) }}}}
\end{equation}
% 1）都加一个整数，在指数函数上，让最大的那个值更容易突出；2）我们的形式在训练集均衡的时候收敛到margin=0
One may notice that our loss is similar to LA loss with an extra constant margin $-\log(1/C)$ when the $\tau=1$ in LA loss. However, the difference lies in two aspects: 1) the margin $-\log(1/C)$ is a positive value and makes all logits get larger, which makes \textit{Softmax} operation more distinguishable. 2) Notice our two margins are related to the label \textit{prior} of the train set and test set. Hence, when the train set is balanced, i.e., $\pi_{y_i}=1/C$, the margin will be $\log(1/C) - \log(1/C) \equiv 0$, and our loss coverage to the standard cross-entropy loss. Furthermore, our loss can handle the situation that both train set and test set are imbalanced distribution via directly setting the margin as $\log(\pi_{y_i}) - \log(\pi'_{y_i})$, where $\pi_{y_i}$ is the label \textit{prior} in the train set and $\pi'_{y_i}$ is the label \textit{prior} in test set correspondingly.

\setcounter{equation}{0}
\setcounter{subsection}{0}
\setcounter{table}{0}   
\setcounter{figure}{0}
\setcounter{corollary}{0}
\renewcommand{\thetable}{C\arabic{table}}
\renewcommand{\thefigure}{C\arabic{figure}}
\renewcommand{\theequation}{C.\arabic{equation}}

\section{Implement detail} \label{Apdx:impdetail}
We conduct experiments on CIFAR-10-LT \cite{CIFAR}, CIFAR-100-LT \cite{CIFAR}, ImageNet-LT \cite{Ijcv/ImageNet}, and iNaturalist 2018 \cite{Cvpr/INaturalist}. We adopt long-tailed version CIFAR datasets which are build by suitably discarding training instances following the Exp files given in \cite{Cvpr/CB, Nips/LDAM}. The instance number of each class exponentially decays in train dataset and keeps balanced during validation process. Here, an imbalance factor $\rho$ is define as $n_{max}/n_{min}$ to measure how imbalanced the dataset is. ImageNet-LT is the LT version of ImageNet. It samples instances for each class following \textit{Pareto} distribution which is also long-tailed. iNaturalist 2018 is a real-world dataset with $8,142$ classes and $\rho=500$, which suffers extremely class-imbalanced and fine-grained problems.

\subsection{Implement details on CIFAR-LT} \label{Apdx:cifardetail}
ResNet-32 is the backbone on CIFAR-LT. For all CIFAR-10-LT and CIFAR-100-LT, the universal data augmentation strategies \cite{Cvpr/ResNet} are used for training. Specifically, we pad $4$ pixels on each side and crop a $32 \times 32$ region. The cropped regions are flipped horizontally with $0.5$ probability and normalized by the mean and standard deviation for each color channel. We train the model via stochastic gradient decent (SGD) with momentum $0.9$ and weight decay of $2 \times 10^{-4}$ for all experiments. All models are trained for $200$ epochs setting mini-batch as $128$ with same learning rate \cite{Nips/LDAM} for fair comparisons. The learning rate is set as \cite{Nips/LDAM}: the initial value is $0.1$ and linear warm up at first $5$ epochs. The learning rate is decayed at the $160^{th}$ and $180^{th}$ epoch by $0.01$. For all \textit{mixup} \cite{Iclr/mixup} and its extension methods \cite{Icml/Manifold_Mixup, Eccv/remix}, we fine-tune the model after $120^{th}$ epoch. The $\alpha$ of Beta distribution is $0.5$ for UniMix and $1.0$ for others.

\subsection{Implement details on large-scale datasets} \label{Apdx:largedetail}
We adopt ResNet-10 \& ResNet-50 for ImageNet-LT and ResNet-50 for iNaturalist 2018. For ImageNet-LT, ResNet-10 and ResNet-50 are trained from scratch for $90$ epochs, setting mini-batch as $512$ and $64$, respectively. For iNaturalist 2018, we adopt the vanilla ResNet-50 with mini-batch as $512$ and train for $90$ epochs. With proper data augment strategy \cite{Cvpr/BBN}, the SGD optimizer with momentum $0.9$ and weight decay $1 \times 10^{-4}$ optimizes the model with same learning rate for fair comparisons \cite{Nips/LDAM, Cvpr/BBN, Iclr/Decouple}. We set base learning rate as $0.2$ and decay it at the $30^{th}$ and $60^{th}$ epoch by $0.1$, respectively.

\setcounter{equation}{0}
\setcounter{subsection}{0}
\setcounter{table}{0}   
\setcounter{figure}{0}
\setcounter{corollary}{0}
\renewcommand{\thetable}{D\arabic{table}}
\renewcommand{\thefigure}{D\arabic{figure}}
\renewcommand{\theequation}{D.\arabic{equation}}

\section{Additional Experiment Results} \label{Apdx:addexp}
We present additional experiments for comprehensive comparisons with previous methods and make a detailed analysis, which can be summarized as follows:

\begin{itemize}[leftmargin=*]
\item Visualized top-1 validation error (\%) comparisons of previous methods.
\item Additional quantity and visualized comparisons of classification \textit{calibration}.
\item Additional visualized comparisons of confusion matrix and $\log$-confusion matrix.
\item Additional comparisons on test imbalance scenarios.
\item Additional comparisons on state-of-the-art two-stage method.
\end{itemize}

\subsection{Visualized comparisons on CIFAR-10-LT and CIFAR-100-LT} \label{Apdx:visualcifar}
Previous sections have shown the remarkable performance of the proposed UniMix and Bayias. Fig.\ref{Fig.apdxcifarcmp} shows the visualized top-1 validation error rate (\%) comparisons on CIFAR-10-LT and CIFAR-100-LT with $\rho \in \{10,50,100,200\}$ for clear and comprehensive comparisons. The histogram indicates the value of each method. The positive error term represents its distance towards the best method, while the negative term indicates the advance towards the worst one.

\begin{figure}[h!]
  \centering
  \includegraphics[width=1\linewidth]{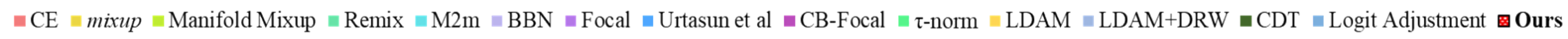}
  \subfigure[CIFAR-10-LT]{
  \begin{minipage}[t]{1\linewidth}
  \centering
  \includegraphics[width=1\linewidth]{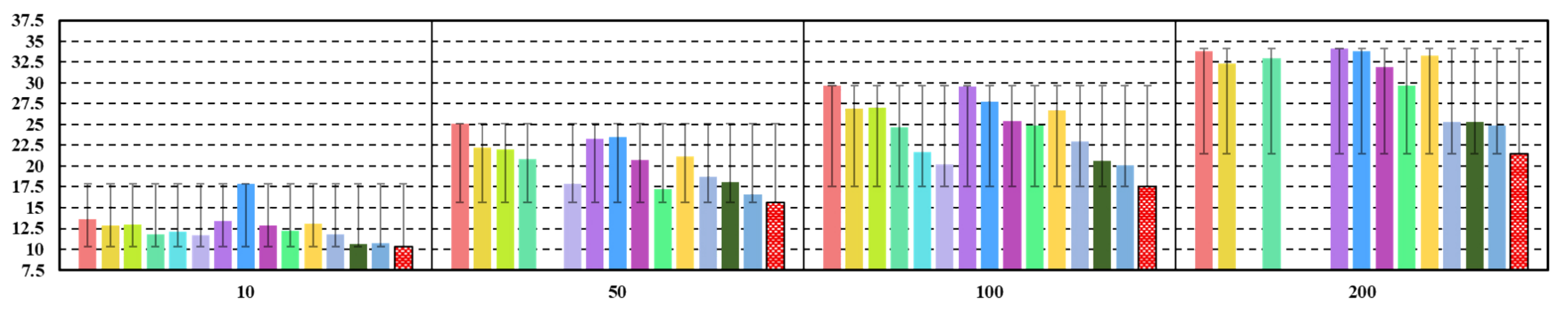}
  %\caption{fig1}
  \end{minipage}%
  }%
  \vspace{-11pt}
  \subfigure[CIFAR-100-LT]{
  \begin{minipage}[t]{1\linewidth}
  \centering
  \includegraphics[width=1\linewidth]{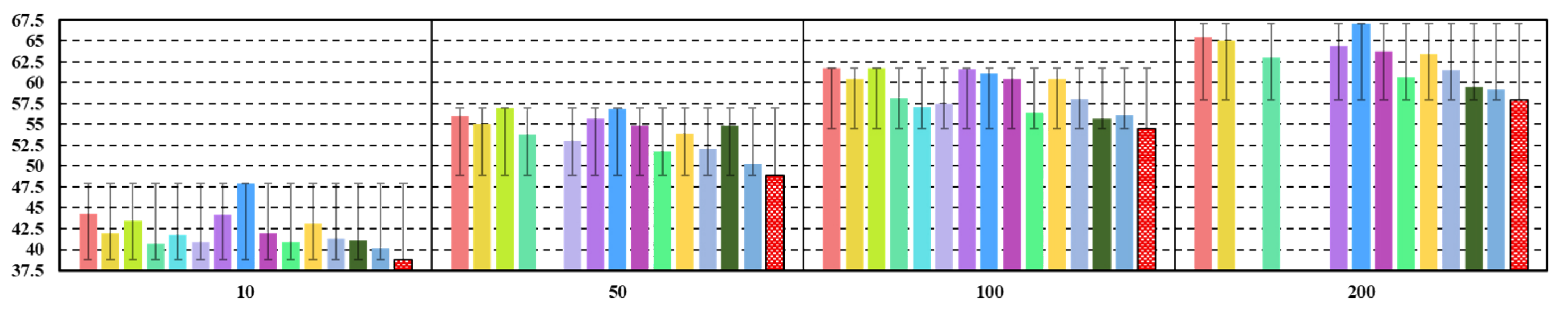}
  %\caption{fig2}
  \end{minipage}%
  }%
  \centering
%   \vspace{-0.5cm}
  \topcaption{Visualized top1 error (\%) ($y$-axis) comparisons of CIFAR-10-LT and CIFAR-100-LT on different $\rho$ ($x$-axis) in ResNet-32. The proposed method (\textcolor{red}{red}) achieves the lowest error rate compared with other methods. The leading advantage of our method increases consistently as imbalance factor gets larger.}
  % \caption{Frame-wise comparisons of the next 10 generated Moving MNIST frames at $80,000$ iterations. The slower trend indicates better performance. The proposed CMS-LSTM is the most high-performing method overall timestamps.}
  \label{Fig.apdxcifarcmp}
  \vspace{-11pt}
\end{figure}

Results in Fig.\ref{Fig.apdxcifarcmp} show that the proposed method outperforms others with lower error rate over all imbalance factors settings. As the dataset gets more skewed and imbalanced, the advantage of our method gradually emerges. On the one hand, the proposed UniMix generates a tail-majority pseudo dataset favoring the tail feature learning, which makes the model achieve better \textit{calibration}. It is practical to improve the generalization of all classes and avoid potential over-fitting and under-fitting risks. On the other hand, the proposed Bayias overcomes the bias caused by existing \textit{prior} differences, which improves the model's performance on the balanced validation dataset. Even in extremely imbalanced scenarios (e.g., CIFAR-100-LT-200), the proposed method still achieves satisfactory performance.

\subsection{Additional quantitative and qualitative \textit{calibration} comparisons} \label{Apdx:caliexpadd}

\subsubsection{Definition of \textit{calibration}} \label{Apdx:calidef}

\textit{Calibration} means a model's predicting probability can estimate the representative of the actual correctness likelihood, which is vital in many real-world decision-making applications \cite{Nips/Faster-RCNN, Kdd/Medical_diagnosis, Corr/Auto_drive}. Suppose a dataset contains $N$ samples $\mathcal{D}:=\{(x_k, y_k)\}_{k=1}^N$, where $x_i$ and $y_i$ represent the $i^{th}$ sample and its corresponding label, respectively. Let $\hat{p}_{y_i}=\mathds{P}(Y=y_i|x=x_i)$ be the confidence of the predicted label, and divide dataset $\mathcal{D}$ into a mini-batch size $m$ with $M=N/m$ in total. The accuracy (acc) and the confidence (cfd) in a mini-batch $\mathcal{B}_m$ are:
\begin{equation}
    \begin{aligned}
    &acc(\mathcal{B}_m)=\frac{1}{m}\sum\nolimits_{i \in {\mathcal{B}_m}} {\mathds{1}({\hat{y}_i} = {y_i})} \\
    &cfd(\mathcal{B}_m)=\frac{1}{m}\sum\nolimits_{i \in {\mathcal{B}_m}} {\hat{p}_{y_i}}
    \end{aligned}
    \label{Eq.defaccconf}
\end{equation}
According to Eq.\ref{Eq.defaccconf}, a perfectly calibrated model will strictly have $acc(\mathcal{B}_m)\equiv cfd(\mathcal{B}_m)$ for all $m\in \{1,\cdots,M\}$. Hence, the Expected Calibration Error (ECE) is proposed as a scalar statistic of \textit{calibration} to quantitatively measure classifiers' mean distance to the ideal $acc(\mathcal{B}_m)\equiv cfd(\mathcal{B}_m)$, which is defined as:
\begin{equation}
    \begin{aligned}
    ECE = \sum\nolimits_{m = 1}^M {\frac{{\left| {{\mathcal{B}_m}} \right|}}{n}\left| {acc({\mathcal{B}_m}) - cfd({\mathcal{B}_m})} \right|}
    \end{aligned}
    \label{Eq.defECE}
\end{equation}
In addition, reliable confidence measures are essential in high-risk applications. The Maximum Calibration Error (MCE) describes the worst-case deviation between confidence and accuracy, which can be defined as:
\begin{equation}
    \begin{aligned}
    MCE = \mathop {\max }\limits_{m \in \{ 1, \cdots ,M\} } \left| {acc({\mathcal{B}_m}) - cfd({\mathcal{B}_m})} \right|
    \end{aligned}
    \label{Eq.defMCE}
\end{equation}
According to Eq.\ref{Eq.defECE},\ref{Eq.defMCE}, a well-calibrated model should coverage ECE and MCE to $0$, indicating that the prediction score reflects the actual accuracy \textit{likelihood}. Under this circumstances, the calibrated model will show excellent robustness and generalization, especially in LT scenarios. Although the authors in \cite{Icml/Calibration-NN} propose the temperature scaling as a post-hoc method to adjust the classifier, our motivation is to train a calibrated model end to end without loss of accuracy.

Previous work \cite{DBLP:conf/iclr/AshukhaLMV20} points out that ECE scores suffer from several shortcomings. Hence, additional metrics \cite{DBLP:conf/cvpr/NixonDZJT19,DBLP:conf/iclr/AshukhaLMV20} are proposed to show more robustness, e.g., Adaptivity \& Adaptive Calibration Error (ACE), Thresholding \& Thresholded Adaptive Calibration Error (TACE), Static Calibration Error (SCE), and Brier Score (BS). The definition of the above metrics are as follows:

TACE disregards all predicted probabilities that are less than a certain threshold $\epsilon$. It adaptively chooses the bin locations to ensure each bin has the same instance numbers and estimates the miscalibration of probabilities across all classes in the prediction while the ECE only chooses the top-1 predicted class. TACE is defined as Eq.\ref{Eq.defTACE}:
\begin{equation}
    \begin{aligned}
    TACE = \frac{1}{CR} \sum\nolimits_{c=1}^C \sum\nolimits_{r=1}^R|acc(\mathcal{B}_r,c)-cfd(\mathcal{B}_r,c)|
    \end{aligned}
    \label{Eq.defTACE}
\end{equation}
where $acc(\mathcal{B}_r, c)$ and $cfd(\mathcal{B}_r, c)$ are the accuracy and confidence of adaptive calibration range $r$ for class label $c$, respectively. Calibration range $r$ defined by the $\lfloor N/R \rfloor_{th}$ index of the sorted and thresholded predictions, exclude the prediction less than $\epsilon$. If set $\epsilon = 0$, TACE converts to ACE.

SCE is a simple extension of ECE to every probability in the multi-class setting. SCE bins predictions separately for each class probability, computes the calibration error within the bin, and averages across bins. SCE is defined as Eq.\ref{Eq.defSCE}:
\begin{equation}
    \begin{aligned}
    SCE = \frac{1}{C} \sum\nolimits_{c=1}^C \sum\nolimits_{b=1}^B \frac{n_{bc}}{N}|acc(\mathcal{B}_b,c)-cfd(\mathcal{B}_b,c)|
    \end{aligned}
    \label{Eq.defSCE}
\end{equation}
where $acc(\mathcal{B}_b, c)$ and $cfd(\mathcal{B}_b, c)$ are the accuracy and confidence of bin $\mathcal{B}_b$ for class label $c$, respectively. $n_{bc}$ is the  prediction number in bin $\mathcal{B}_b$ for class $c$. $N$ is the total number of test set.

BS \cite{Brier-Score} has also been known as a metric for the verification of predicted probabilities. Similarly to the log-likelihood, BS penalizes low probabilities assigned to correct predictions and high probabilities assigned to wrong ones, which is defined as Eq.\ref{Eq.defBS}:
\begin{equation}
    \begin{aligned}
    BS = \frac{1}{NC} \sum\nolimits_{i=1}^N \sum\nolimits_{c=1}^C \left(\mathds{1}(y_i^*=c)-\mathds{P}(Y=y_c|X=x_i)\right)^2
    \end{aligned}
    \label{Eq.defBS}
\end{equation}
where $N$ represents the total number of test set. $y_i^*$, $\mathds{P}(Y=y_c|X=x_i)$ represent the ground truth and predicted label, respectively.

\subsubsection{Quantitative results of \textit{calibration} on CIFAR-LT}

Besides ECE and MCE results, we additionally provide quantitative metrics results, i.e., ACE, TACE setting threshold $\epsilon = 1e-3$, SCE, and BS. The comparisons are illustrated in Tab.\ref{Tab::ACE}.

\begin{table}[h!]
\setlength{\tabcolsep}{10pt}
\centering
\topcaption{Quantitative calibration metric of ResNet-32 on CIFAR-10/100-LT-100 test set. Smaller ACE, TACE, SCE, and BS indicate better calibration results. Either of the proposed methods achieves a well-calibrated model compared with others. The combination of UniMix and Bayias still achieves the best performance.}
\resizebox{1\textwidth}{!}{%
\begin{tabular}{l|cccc|cccc}
\toprule
Dataset                                                      & \multicolumn{4}{c|}{CIFAR-10-LT-100} & \multicolumn{4}{c}{CIFAR-100-LT-100} \\ \midrule
Calibration Metric ($\times$ 100) & ACE      & TACE    & SCE     & BS   & ACE      & TACE     & SCE     & BS   \\ \midrule
ERM                                                                   & 5.42     & 4.66    & 5.46    & 53.61          & 0.831    & 0.709    & 0.952   & 97.05         \\
\textit{mixup} \cite{Iclr/mixup}                                                                & 4.87     & 4.20    & 4.92    & 49.09          & 0.751    & 0.683    & 0.849   & 89.26         \\
Remix \cite{Eccv/remix}                                                               & 3.49     & 3.32    & 3.77    & 44.84          & 0.740    & 0.675    & 0.846   & 89.33         \\
LDAM+DRW \cite{Nips/LDAM}                                                             & 4.14     & 2.84    & 4.43    & 44.76          & 0.801    & 0.671    & 1.093   & 107.4         \\ \midrule
UniMix (ours)                                                         & 3.48     & 3.27    & 3.57    & 38.98          & 0.505    & 0.555    & 0.687   & 80.70         \\
Bayias (ours)                                                         & 2.45     & 2.40    & 2.69    & 34.30          & 0.460    & 0.518    & 0.631   & 78.78         \\
UniMix+Bayias (ours)                                                  & \textbf{2.31}     & \textbf{2.17}    & \textbf{2.43}    & \textbf{32.84}          & \textbf{0.450}    & \textbf{0.517}    & \textbf{0.623}   & \textbf{78.24}         \\ \bottomrule
\end{tabular}}
\label{Tab::ACE}
\end{table}

\subsubsection{Additional visualized \textit{calibration} comparison on CIFAR-LT} \label{Apdx.calibrationvis}

To make intuitive comparisons, we visualize additional confidence-accuracy joint density plots and reliability diagrams in CIFAR-LT test set setting $\rho=100$. The confidence data is obtained by the average \textit{Softmax} winning score in a test mini-batch \cite{Nips/On_Mixup_Training}. The reliability diagrams data is obtained follow \cite{Icml/Calibration-NN}, which groups all prediction score into $15$ interval bins and calculate the accuracy of each bin. The results are available in Fig.\ref{Fig.addcifarcali}.

\begin{figure}[t!]
% \vspace{-5pt}
	\centering
% 	\subfigure[CIFAR-10-LT-10]{
% 		\begin{minipage}[b]{0.4\textwidth}
% 			\includegraphics[width=1\textwidth]{Figure/cali/Cali_CIFAR10-10.pdf}
% 		\end{minipage}
% 		\label{Fig.conf-acc(A)}
% 	}
		\subfigure[CIFAR-10-LT-100]{
		\begin{minipage}[b]{0.48\textwidth}
			\includegraphics[width=1\textwidth]{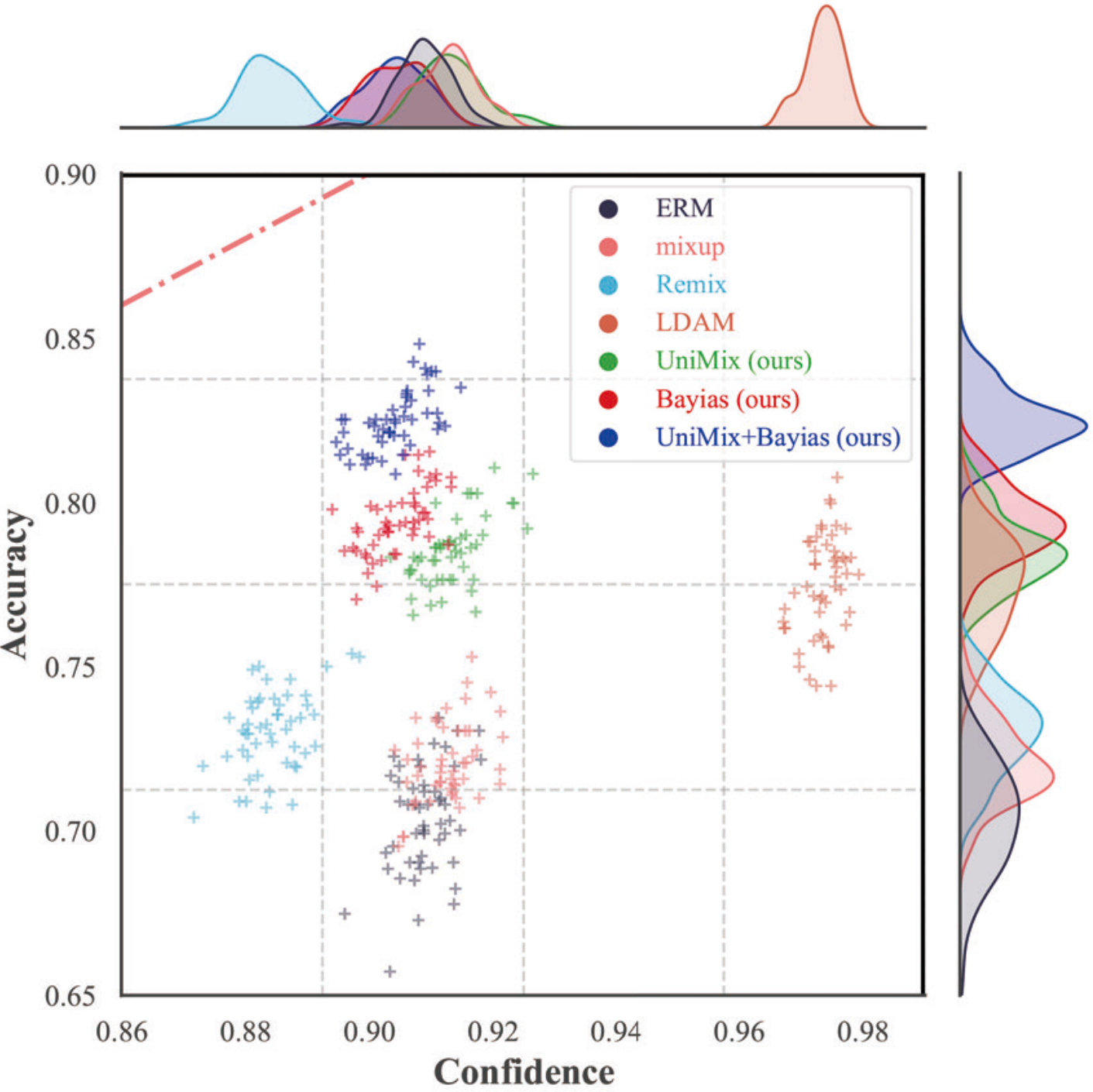} 
		\end{minipage}
		\label{Fig.conf-acc(B)}
	}
% 	\vspace{-0.5cm}
%     	\subfigure[CIFAR-100-LT-10]{
%     		\begin{minipage}[b]{0.4\textwidth}
%   		 	\includegraphics[width=1\textwidth]{Figure/cali/Cali_CIFAR100-10.pdf}
%     		\end{minipage}
% 		\label{Fig.conf-acc(C)}
%     }
        	\subfigure[CIFAR-100-LT-100]{
    		\begin{minipage}[b]{0.48\textwidth}
		 	\includegraphics[width=1\textwidth]{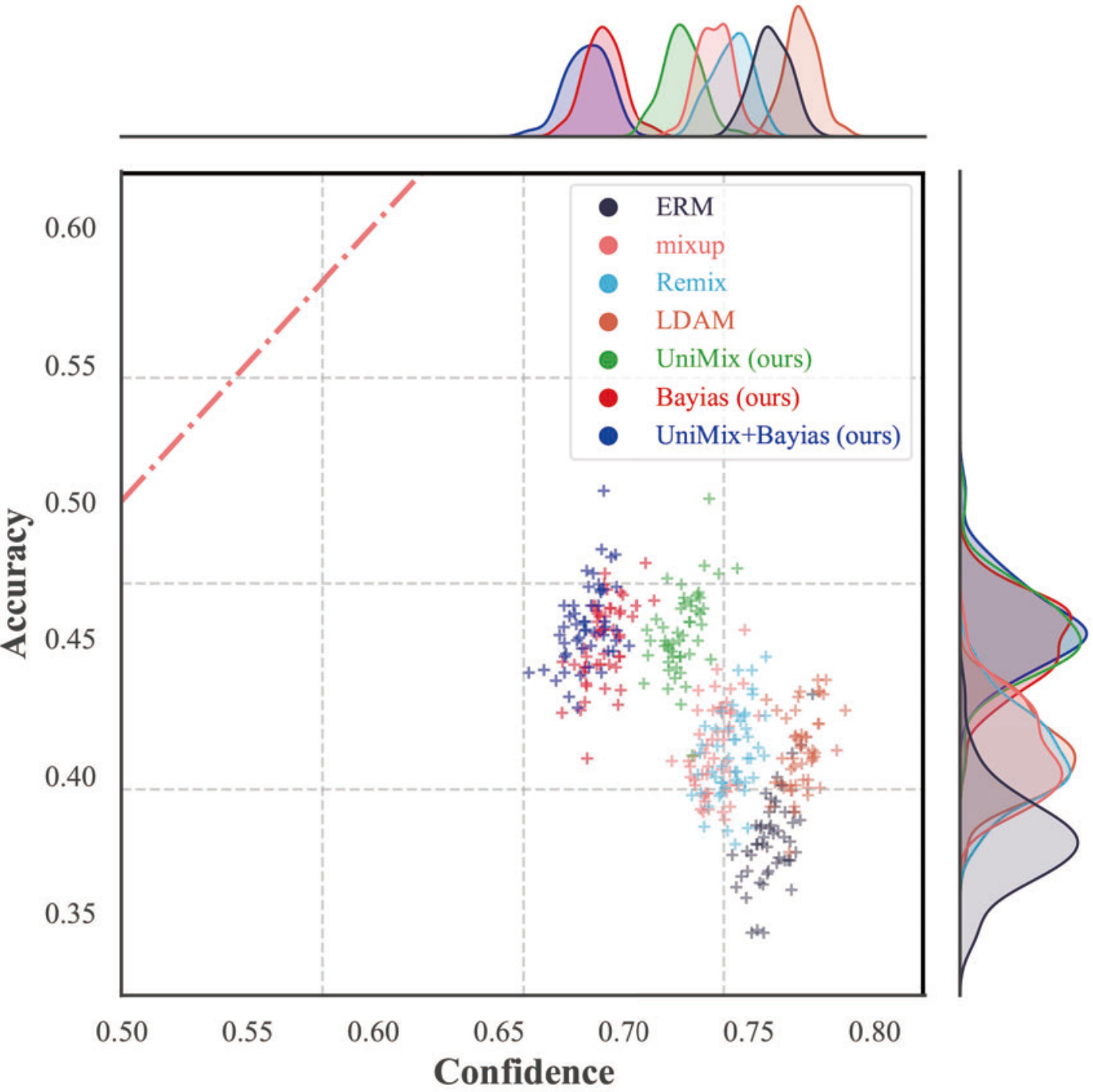}
    		\end{minipage}
		\label{Fig.conf-acc(D)}
    }
\topcaption{Additional comparisons of joint density plots of accuracy vs. confidence on CIFAR-10-LT and CIFAR-100-LT. A well-calibrated classifier’s density will lay around the $y = x$ (\textcolor{red}{red} dot line). The combination of UniMix and Bayias achieves remarkable results especially in the severely imbalance scenarios (i.e., $\rho=100$).}
\label{Fig.addcifarcali}
% \vspace{-0.5cm}
\end{figure}

\begin{figure}[h!]
% \vspace{-5pt}
	\flushleft
	\subfigure[ERM]{
		\begin{minipage}[b]{0.23\textwidth}
		\flushleft
			\includegraphics[width=1.1\textwidth]{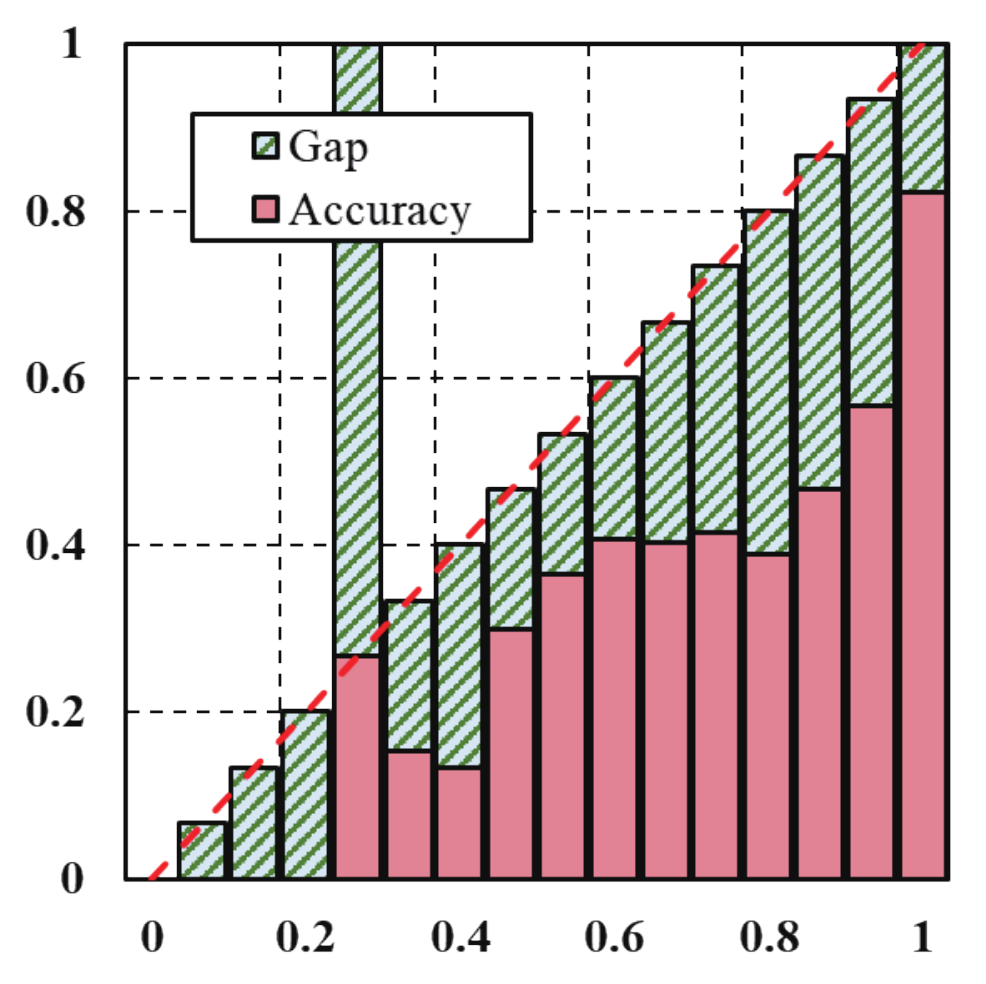}
		\end{minipage}
		\label{Fig.cifar-10-100-bar(A)}
	}
	 \subfigure[\textit{mixup} \cite{Iclr/mixup}]{
	 \centering
     \begin{minipage}[b]{0.23\textwidth}
	    \includegraphics[width=1.1\textwidth]{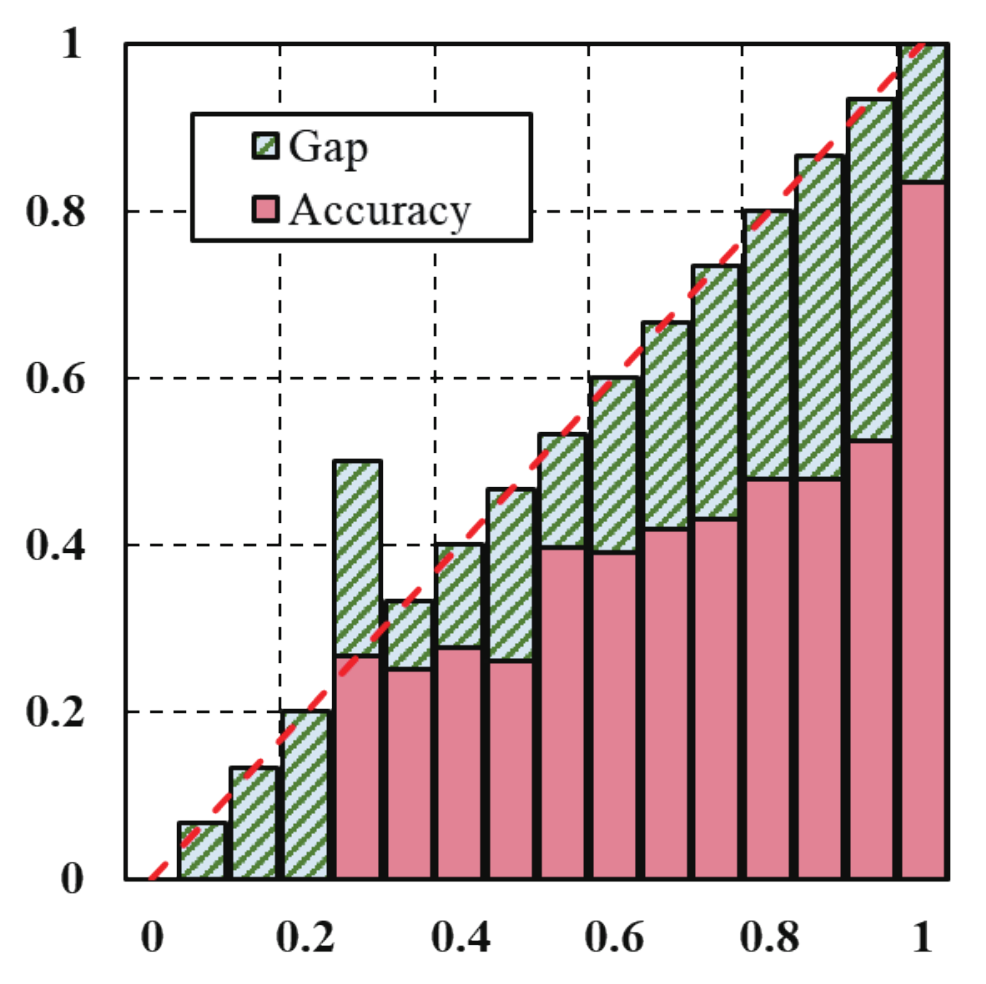}
     \end{minipage}
     \label{Fig.cifar-10-100-bar(B)}
    }
		\subfigure[LDAM+DRW \cite{Nips/LDAM}]{
		\centering
		\begin{minipage}[b]{0.23\textwidth}
			\includegraphics[width=1.1\textwidth]{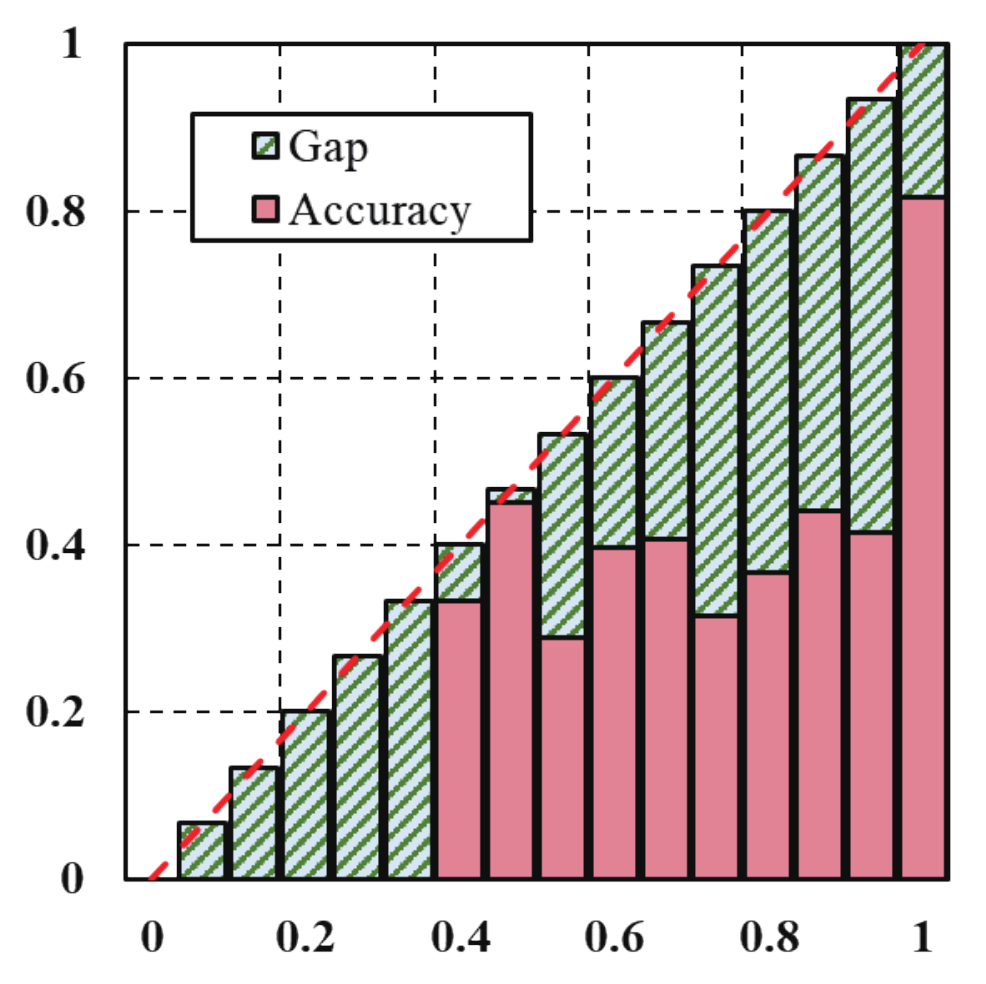} 
		\end{minipage}
		\label{Fig.cifar-10-100-bar(C)}
	}
    	\subfigure[Ours]{
    	
		\begin{minipage}[b]{0.23\textwidth}
		\flushright
  	 	    \includegraphics[width=1.1\textwidth]{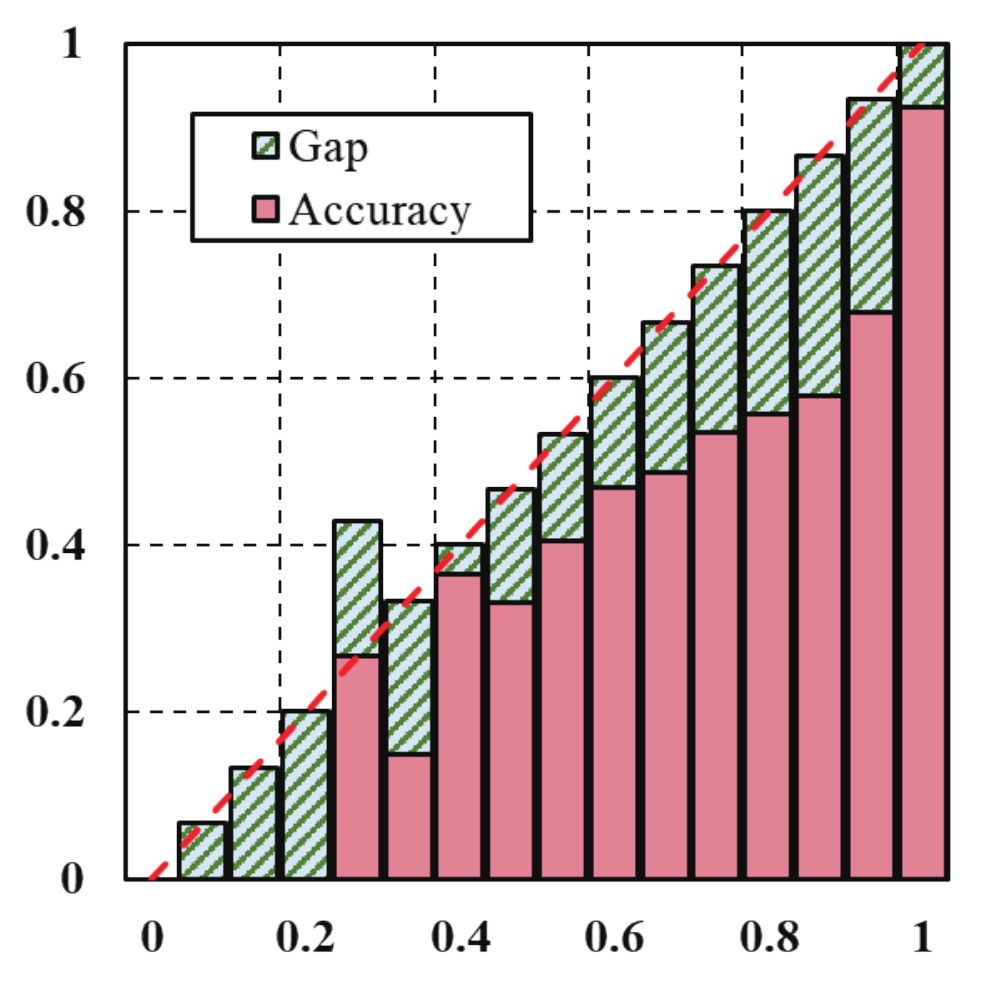}
		\end{minipage}
		\label{Fig.cifar-10-100-bar(D)}
    }
    
    % 	\flushleft
    	\subfigure[ERM]{
		\begin{minipage}[b]{0.23\textwidth}
		\flushleft 
			\includegraphics[width=1.1\textwidth]{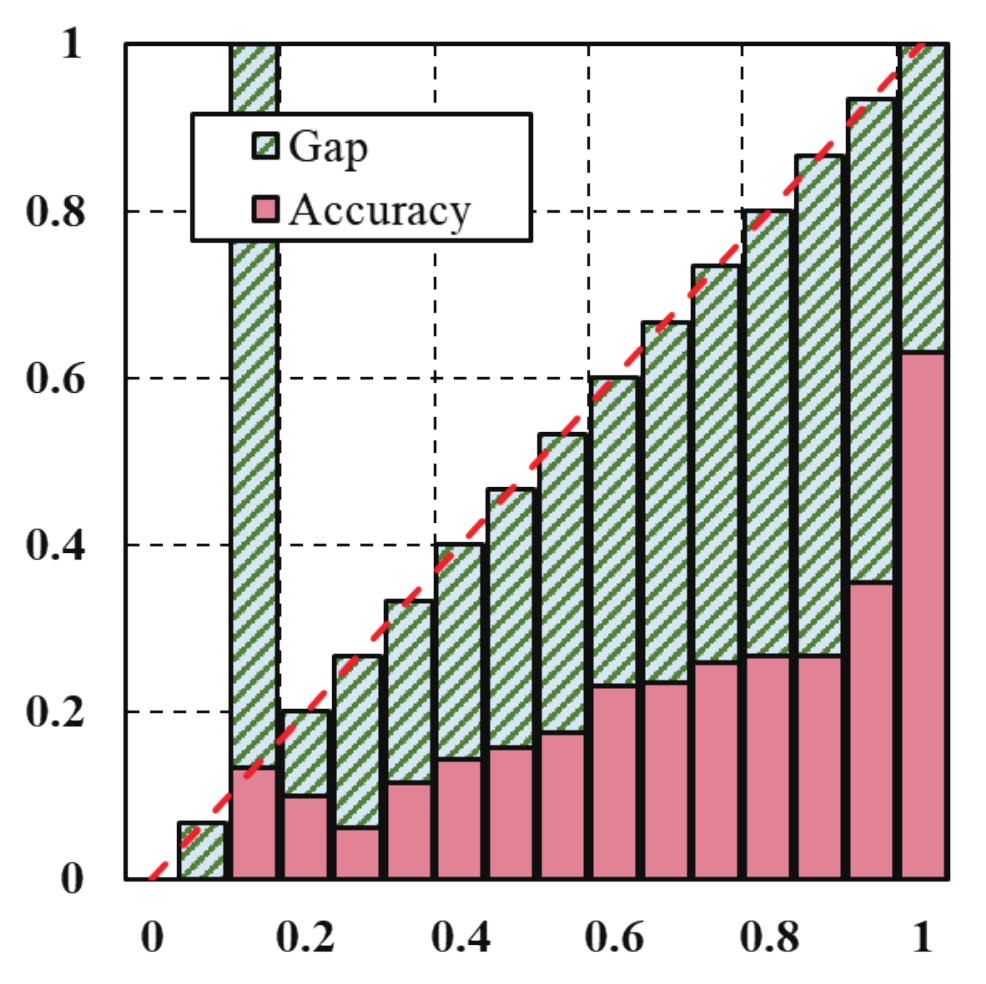}
		\end{minipage}
		\label{Fig.cifar-100-100-bar(A)}
	}
	        	\subfigure[\textit{mixup} \cite{Iclr/mixup}]{
	        	\centering
    		\begin{minipage}[b]{0.23\textwidth}
		 	\includegraphics[width=1.1\textwidth]{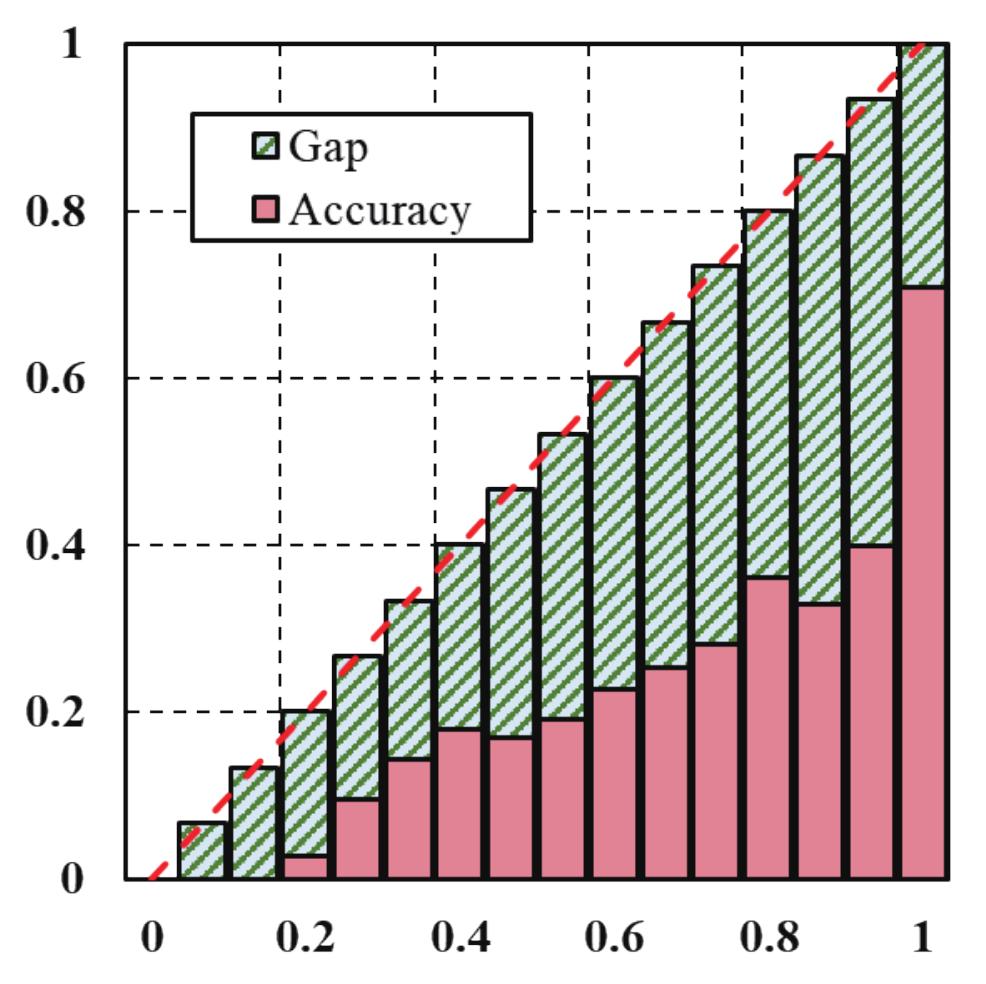}
    		\end{minipage}
		\label{Fig.cifar-100-100-bar(B)}
    }
		\subfigure[LDAM+DRW \cite{Nips/LDAM}]{
		\centering
		\begin{minipage}[b]{0.23\textwidth}
			\includegraphics[width=1.1\textwidth]{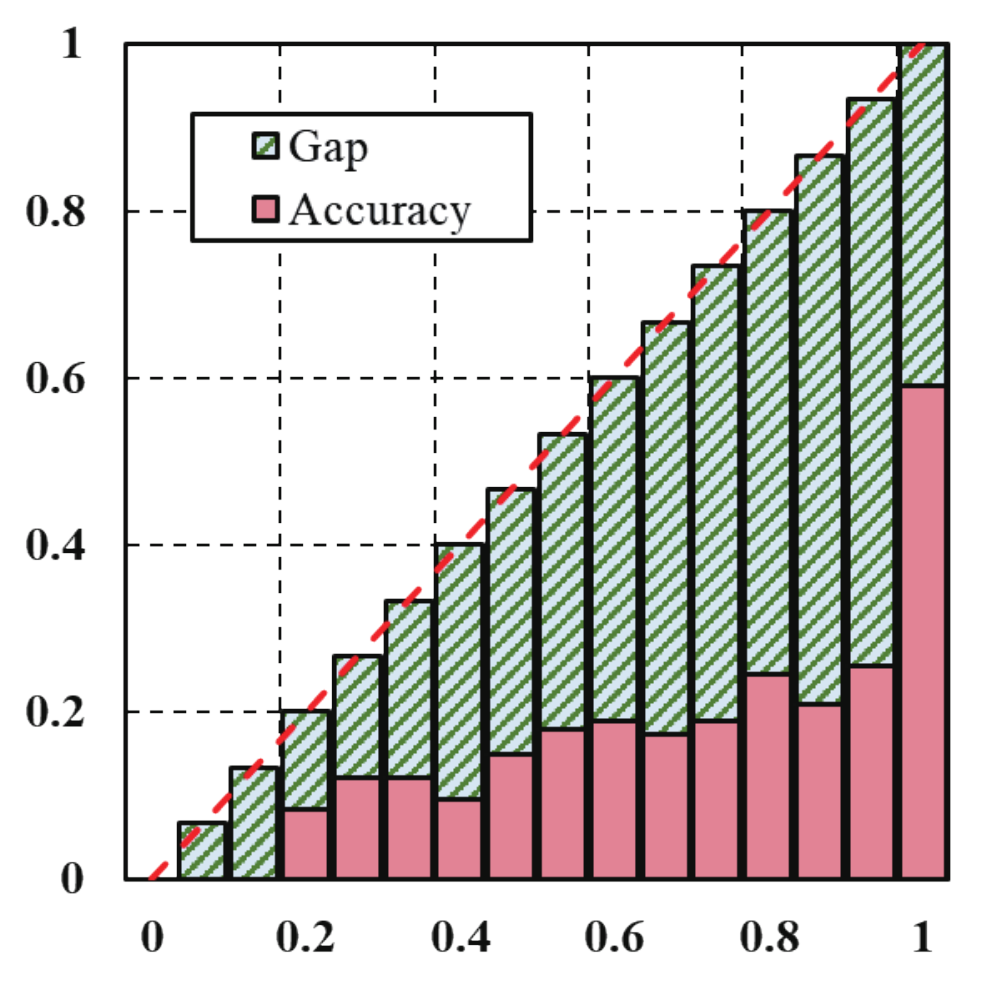} 
		\end{minipage}
		\label{Fig.cifar-100-100-bar(C)}
	}
    	\subfigure[Ours]{
    	
    		\begin{minipage}[b]{0.23\textwidth}
    		\flushright
  		 	\includegraphics[width=1.1\textwidth]{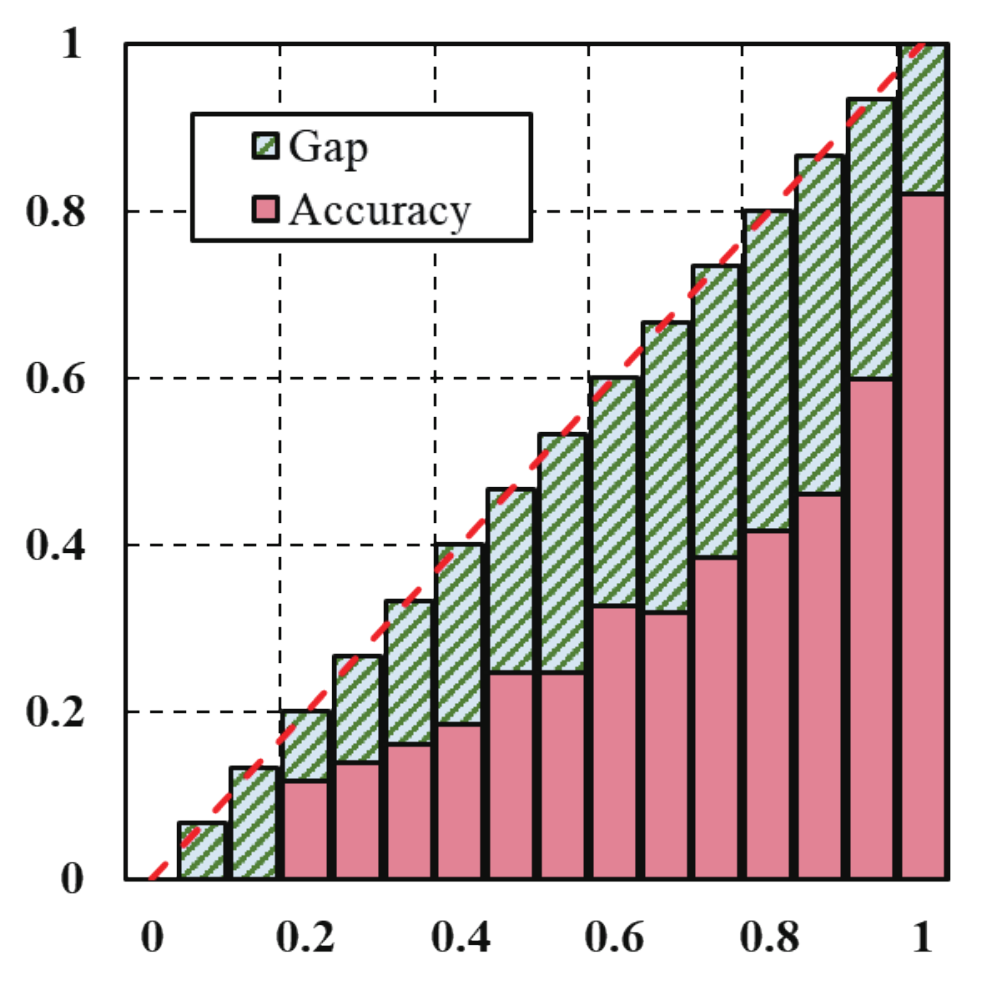}
    		\end{minipage}
		\label{Fig.cifar-100-100-bar(D)}
    }
\topcaption{Reliability diagrams of ResNet-32 on CIFAR-10-LT-100 (top) and CIFAR-100-LT-100 (bottom). $x$-axis and $y$-axis represent the confidence and accuracy, respectively. A well-calibrated model shows smaller gap in each bin, i.e., the accuracy bar is closed to $y=x$. Our proposed method shows best results compared with others without any further fine-tune.}
\label{Fig.addcifar-10-100-bar}
% \vspace{-0.5cm}
\end{figure}

Fig.\ref{Fig.addcifarcali} clearly shows that the proposed method achieves better \textit{calibration} in the CIFAR-LT test set. Each scatters data is the corresponding result of confidence and accuracy in a mini-batch $m$. In LT scenarios, the validation accuracy is usually lower than train accuracy, especially for the tail classes. Such overconfidence and miscalibration obstruct the network from having better generalization performance. In Fig.\ref{Fig.addcifarcali}, the distribution of scattered points reflects a model’s generalization and \textit{calibration}. Our proposed method is the closest towards ideal $y=x$. It means each class gets enough regulation and hence contributes to a well-calibrated model. In contrast, \textit{mixup} and its extensions are similar to ERM, which show limited improvement on classification \textit{calibration} in LT scenarios. LDAM ameliorates the LT situation to some extent but results in more severe overconfidence cases, which may explain why it is counterproductive to other methods.

The reliability diagram is another visual representation of model \textit{calibration}. As Fig.\ref{Fig.addcifar-10-100-bar} shows, the baseline (ERM) is overconfident in its predictions and the accuracy is generally below ideal $y=x$ for each bin. Our method produces much better confidence estimates than others, which means our success in regulating all classes. Although some post-hoc adjustment methods \cite{Icml/Calibration-NN, DBLP:conf/cvpr/MiSLAS} can achieve better classification \textit{calibration}, our approach trains a calibrated model end-to-end, which avoids the potential adverse effects on the original task. Considering the similarity of purposes with MiSLAS \cite{DBLP:conf/cvpr/MiSLAS}, we make detailed discussion in Appendix \ref{Apdx.cmpMiSLAS}.

\subsection{Additional visualized comparisons of confusion matrix on CIFAR-LT}
\subsubsection{Visualized confusion matrix on CIFAR-10-LT}
We further give additional visualized results on CIFAR-10-LT setting $\rho=10$ and $\rho=100$, which are available in Fig.\ref{Fig.addcifar-10-10-confusematrix},\ref{Fig.addcifar-10-100-confusematrix}. The results show that all methods perform well compared with ERM in the simple dataset (e.g., CIFAR-10-LT-10, CIFAR-10-LT-100). In general, the proposed method significantly improves the accuracy of the tail for the larger diagonal values. The improvement on the tail is superior than other methods, which leads to state-of-the-art performance.

\begin{figure}[h!]
% \vspace{-5pt}
	\centering
	\subfigure[ERM]{
		\begin{minipage}[b]{0.23\textwidth}
			\includegraphics[width=1\textwidth]{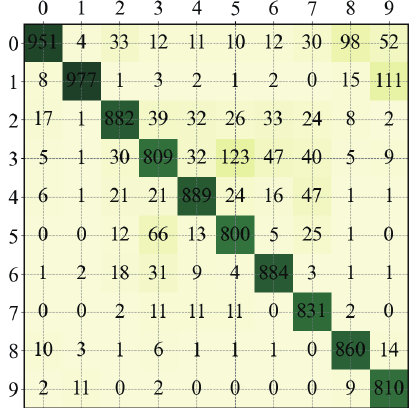}
		\end{minipage}
		\label{Fig.cifar-10-10-confusematrix(A)}
	}
	        	\subfigure[\textit{mixup} \cite{Iclr/mixup}]{
    		\begin{minipage}[b]{0.23\textwidth}
		 	\includegraphics[width=1\textwidth]{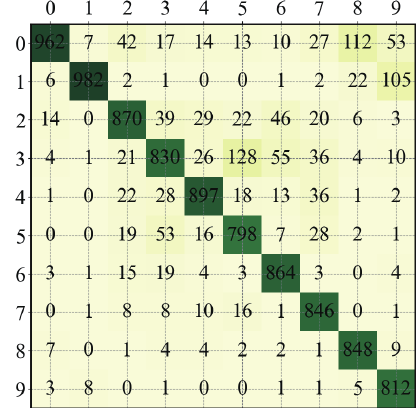}
    		\end{minipage}
		\label{Fig.cifar-10-10-confusematrix(B)}
    }
		\subfigure[LDAM+DRW \cite{Nips/LDAM}]{
		\begin{minipage}[b]{0.23\textwidth}
			\includegraphics[width=1\textwidth]{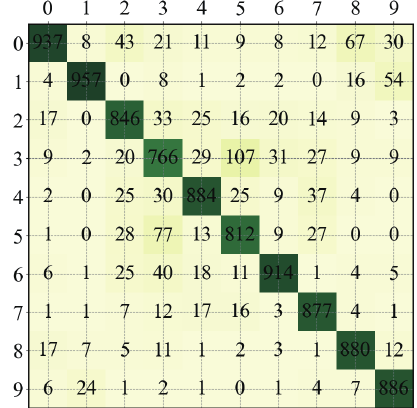} 
		\end{minipage}
		\label{Fig.cifar-10-10-confusematrix(C)}
	}
    	\subfigure[CDT \cite{Corr/CDT}]{
    		\begin{minipage}[b]{0.23\textwidth}
  		 	\includegraphics[width=1\textwidth]{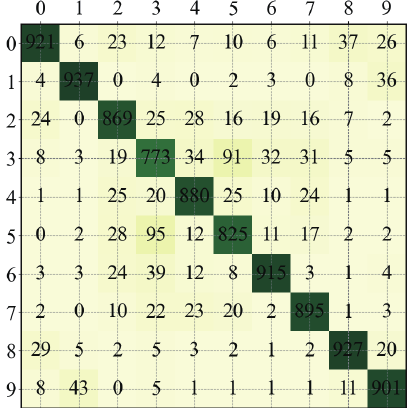}
    		\end{minipage}
		\label{Fig.cifar-10-10-confusematrix(D)}
    }
    
        	\subfigure[Remix \cite{Eccv/remix}]{
    		\begin{minipage}[b]{0.23\textwidth}
		 	\includegraphics[width=1\textwidth]{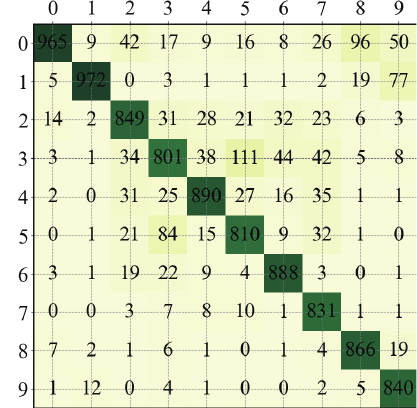}
    		\end{minipage}
		\label{Fig.cifar-10-10-confusematrix(E)}
    }
            	\subfigure[CB+DRW \cite{Cvpr/CB}]{
    		\begin{minipage}[b]{0.23\textwidth}
		 	\includegraphics[width=1\textwidth]{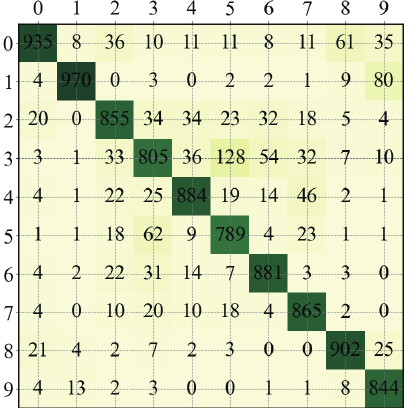}
    		\end{minipage}
		\label{Fig.cifar-10-10-confusematrix(F)}
    }
            	\subfigure[Logit Adjustment \cite{Aaai/logit_adjustment}]{
    		\begin{minipage}[b]{0.23\textwidth}
		 	\includegraphics[width=1\textwidth]{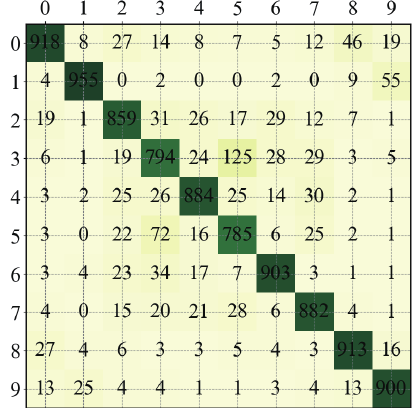}
    		\end{minipage}
		\label{Fig.cifar-10-10-confusematrix(G)}
    }
            	\subfigure[Ours]{
    		\begin{minipage}[b]{0.23\textwidth}
		 	\includegraphics[width=1\textwidth]{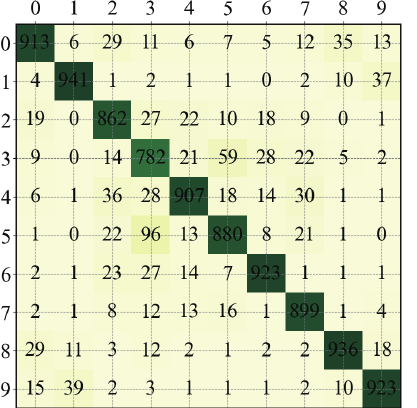}
    		\end{minipage}
		\label{Fig.cifar-10-10-confusematrix(H)}
    }
\topcaption{Additional confusion matrix comparisons on CIFAR-10-LT-10. The $x$-axis and $y$-axis indicate the ground truth and predicted labels, respectively. Deeper color indicates larger values. All methods achieve satisfactory results while ours further improves the tail feature learning and make misclassification cases more balanced distributed.}
\label{Fig.addcifar-10-10-confusematrix}
% \vspace{-0.5cm}
\end{figure}

As for misclassification results, the non-diagonal elements concentrate on the top-right triangle. Hence, most of previous methods have limited improvement on minority classes, which tend to simply predict the instances in tail as majority ones. In contrast, the proposed method improves the performance for the tail and makes the misclassification case more balance distributed.

\begin{figure}[h!]
% \vspace{-5pt}
	\centering
	\subfigure[ERM]{
		\begin{minipage}[b]{0.23\textwidth}
			\includegraphics[width=1\textwidth]{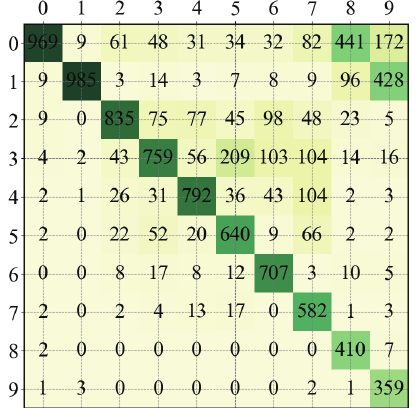}
		\end{minipage}
		\label{Fig.cifar-10-100-confusematrix(A)}
	}
	        	\subfigure[\textit{mixup} \cite{Iclr/mixup}]{
    		\begin{minipage}[b]{0.23\textwidth}
		 	\includegraphics[width=1\textwidth]{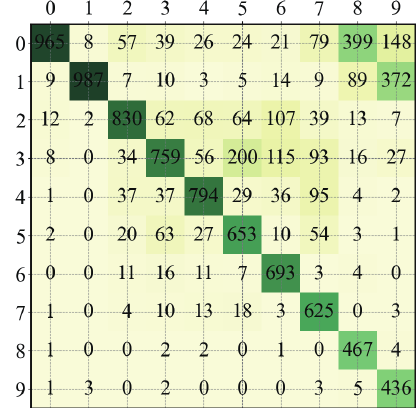}
    		\end{minipage}
		\label{Fig.cifar-10-100-confusematrix(B)}
    }
		\subfigure[LDAM+DRW \cite{Nips/LDAM}]{
		\begin{minipage}[b]{0.23\textwidth}
			\includegraphics[width=1\textwidth]{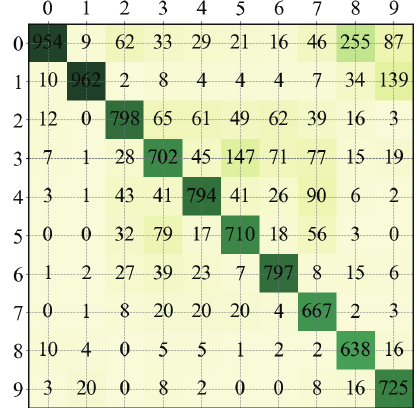} 
		\end{minipage}
		\label{Fig.cifar-10-100-confusematrix(C)}
	}
    	\subfigure[CDT \cite{Corr/CDT}]{
    		\begin{minipage}[b]{0.23\textwidth}
  		 	\includegraphics[width=1\textwidth]{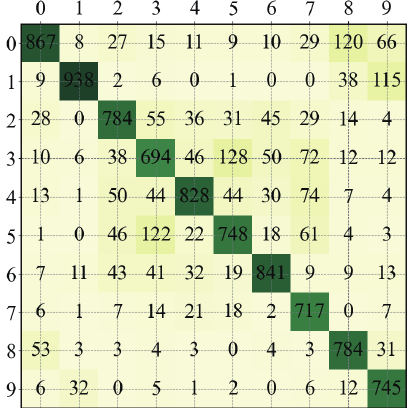}
    		\end{minipage}
		\label{Fig.cifar-10-100-confusematrix(D)}
    }
    
        	\subfigure[Remix \cite{Eccv/remix}]{
    		\begin{minipage}[b]{0.23\textwidth}
		 	\includegraphics[width=1\textwidth]{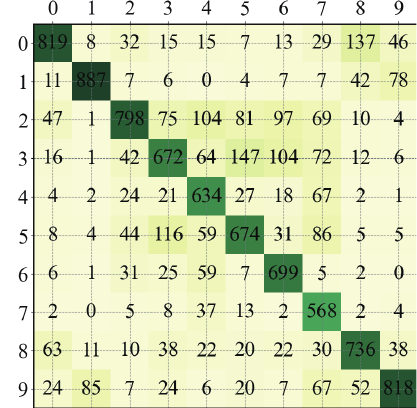}
    		\end{minipage}
		\label{Fig.cifar-10-100-confusematrix(E)}
    }
            	\subfigure[CB+DRW \cite{Cvpr/CB}]{
    		\begin{minipage}[b]{0.23\textwidth}
		 	\includegraphics[width=1\textwidth]{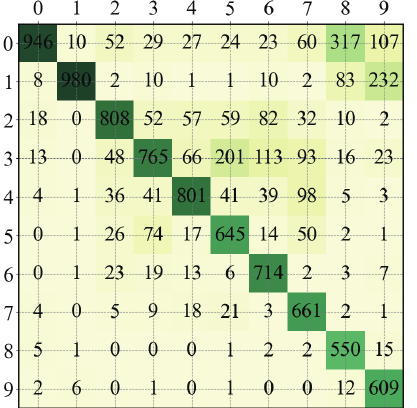}
    		\end{minipage}
		\label{Fig.cifar-10-100-confusematrix(F)}
    }
            	\subfigure[Logit Adjustment \cite{Aaai/logit_adjustment}]{
    		\begin{minipage}[b]{0.23\textwidth}
		 	\includegraphics[width=1\textwidth]{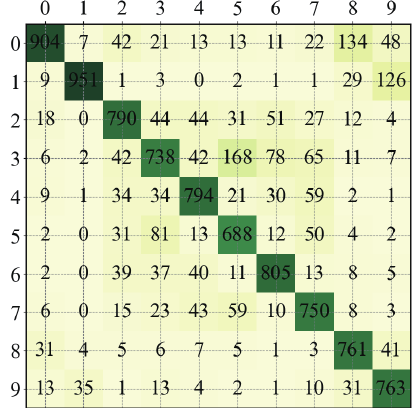}
    		\end{minipage}
		\label{Fig.cifar-10-100-confusematrix(G)}
    }
            	\subfigure[Ours]{
    		\begin{minipage}[b]{0.23\textwidth}
		 	\includegraphics[width=1\textwidth]{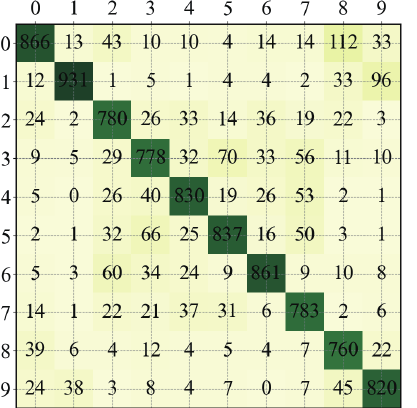}
    		\end{minipage}
		\label{Fig.cifar-10-100-confusematrix(H)}
    }
\topcaption{Additional confusion matrix comparisons on CIFAR-10-LT-100. The $x$-axis and $y$-axis indicate the ground truth and predicted labels, respectively. Deeper color indicates larger values. The disparity of these methods gradually gets appeared. Our method achieves the best accuracy as well as better non-diagonal distribution. Other methods exhibit an obvious bias towards the head or tail.}
\label{Fig.addcifar-10-100-confusematrix}
% \vspace{-0.5cm}
\end{figure}

\subsubsection{Visualized $\log$-confusion matrix on CIFAR-100-LT}

Furthermore, we visualize the distinguishing results of our methods on challenging CIFAR-100-LT for comprehensive comparisons. Specifically, we plot additional comparisons on CIFAR-100-LT-10 (Fig.\ref{Fig.addcifar-100-10-confusematrix}) and CIFAR-100-LT-100 (Fig.\ref{Fig.addcifar-100-100-confusematrix}). 

Fig.\ref{Fig.addcifar-100-10-confusematrix} \ref{Fig.addcifar-100-100-confusematrix} show that in the challenging LT scenarios, the diagonal values get decreased towards the tail. A proper confusion matrix should exhibit a balanced distribution of misclassification case and large enough diagonal values. The proposed method outperforms others both in the correct cases (the diagonal elements) and the misclassification case distribution. In general, the proposed method shows the best performance, especially in the tail classes (deeper color represents larger value).  The misclassification cases in other methods apparently concentrate on the upper triangular, which means the classifiers simply predict the tail samples as the head and vise versa.

\begin{figure}[h!]
% \vspace{-5pt}
	\centering
	\subfigure[ERM]{
		\begin{minipage}[b]{0.23\textwidth}
			\includegraphics[width=1\textwidth]{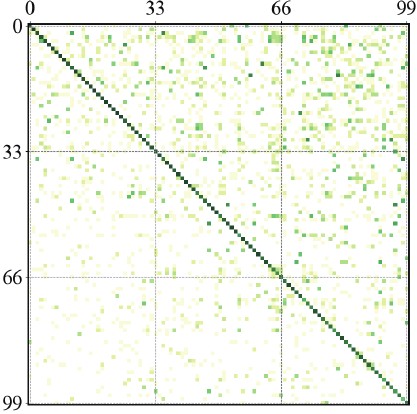}
		\end{minipage}
		\label{Fig.cifar-100-10-confusematrix(A)}
	}
	        	\subfigure[\textit{mixup} \cite{Iclr/mixup}]{
    		\begin{minipage}[b]{0.23\textwidth}
		 	\includegraphics[width=1\textwidth]{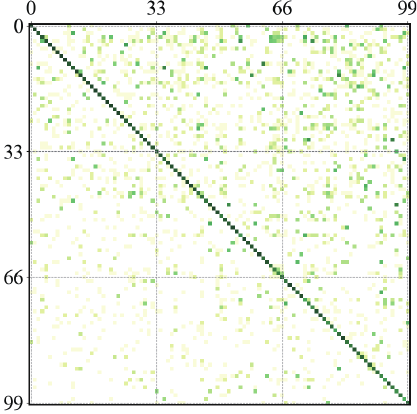}
    		\end{minipage}
		\label{Fig.cifar-100-10-confusematrix(B)}
    }
		\subfigure[LDAM+DRW \cite{Nips/LDAM}]{
		\begin{minipage}[b]{0.23\textwidth}
			\includegraphics[width=1\textwidth]{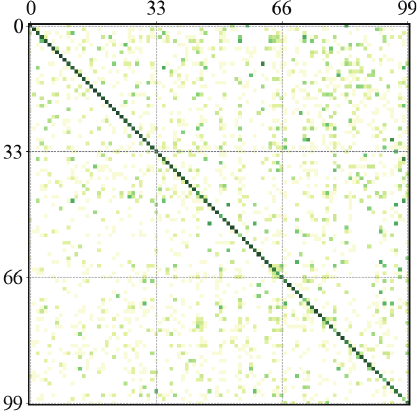} 
		\end{minipage}
		\label{Fig.cifar-100-10-confusematrix(C)}
	}
    	\subfigure[CDT \cite{Corr/CDT}]{
    		\begin{minipage}[b]{0.23\textwidth}
  		 	\includegraphics[width=1\textwidth]{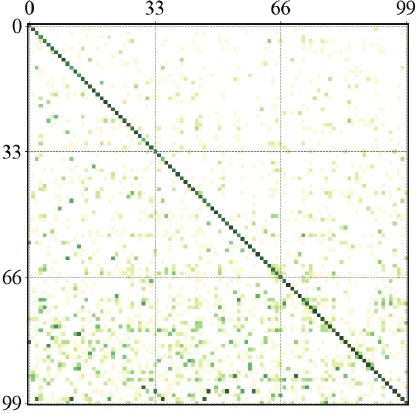}
    		\end{minipage}
		\label{Fig.cifar-100-10-confusematrix(D)}
    }
    
        	\subfigure[Remix \cite{Eccv/remix}]{
    		\begin{minipage}[b]{0.23\textwidth}
		 	\includegraphics[width=1\textwidth]{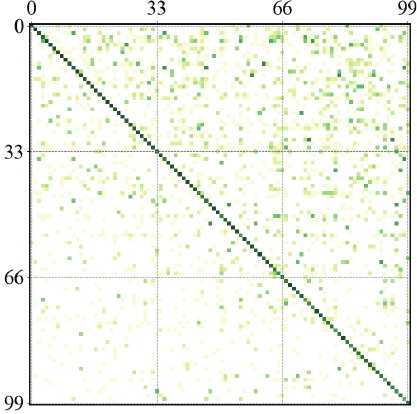}
    		\end{minipage}
		\label{Fig.cifar-100-10-confusematrix(E)}
    }
            	\subfigure[CB+DRW \cite{Cvpr/CB}]{
    		\begin{minipage}[b]{0.23\textwidth}
		 	\includegraphics[width=1\textwidth]{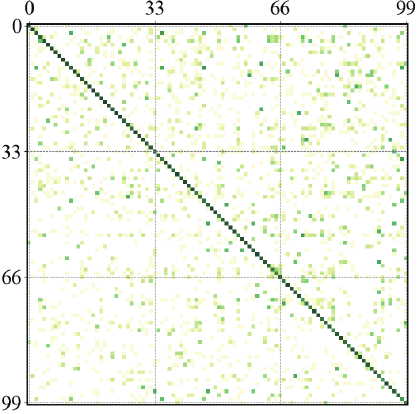}
    		\end{minipage}
		\label{Fig.cifar-100-10-confusematrix(F)}
    }
            	\subfigure[Logit Adjustment \cite{Aaai/logit_adjustment}]{
    		\begin{minipage}[b]{0.23\textwidth}
		 	\includegraphics[width=1\textwidth]{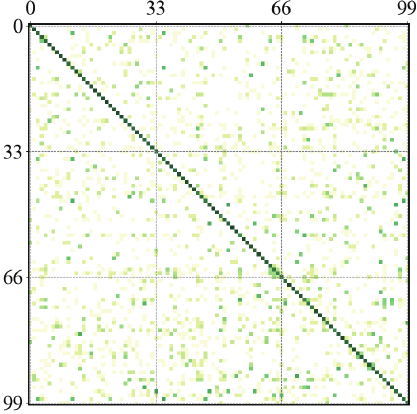}
    		\end{minipage}
		\label{Fig.cifar-100-10-confusematrix(G)}
    }
            	\subfigure[Ours]{
    		\begin{minipage}[b]{0.23\textwidth}
		 	\includegraphics[width=1\textwidth]{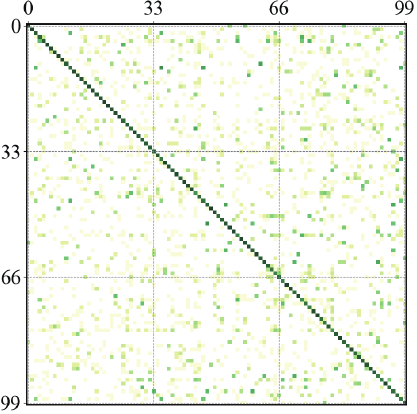}
    		\end{minipage}
		\label{Fig.cifar-100-10-confusematrix(H)}
    }
\topcaption{Additional $\log$-confusion matrix comparisons on CIFAR-100-LT-10. The $x$-axis and $y$-axis indicate the ground truth and predicted labels, respectively. $\log$ operation is adopted for clearer visualisation. Deeper color indicates larger values. The increased class number makes the misclassification distribution more clearly.}
\label{Fig.addcifar-100-10-confusematrix}
% \vspace{-0.5cm}
\end{figure}

\begin{figure}[h!]
% \vspace{-5pt}
	\centering
	\subfigure[ERM]{
		\begin{minipage}[b]{0.23\textwidth}
			\includegraphics[width=1\textwidth]{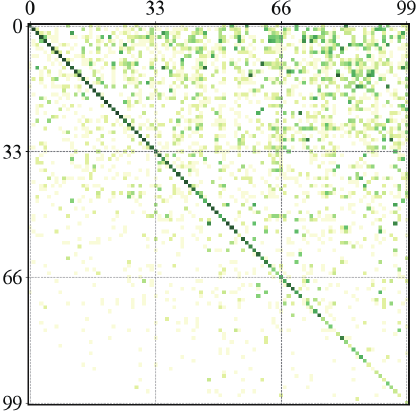}
		\end{minipage}
		\label{Fig.cifar-100-100-confusematrix(A)}
	}
	        	\subfigure[\textit{mixup} \cite{Iclr/mixup}]{
    		\begin{minipage}[b]{0.23\textwidth}
		 	\includegraphics[width=1\textwidth]{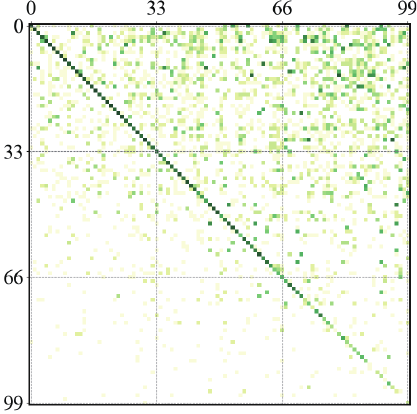}
    		\end{minipage}
		\label{Fig.cifar-100-100-confusematrix(B)}
    }
		\subfigure[LDAM+DRW \cite{Nips/LDAM}]{
		\begin{minipage}[b]{0.23\textwidth}
			\includegraphics[width=1\textwidth]{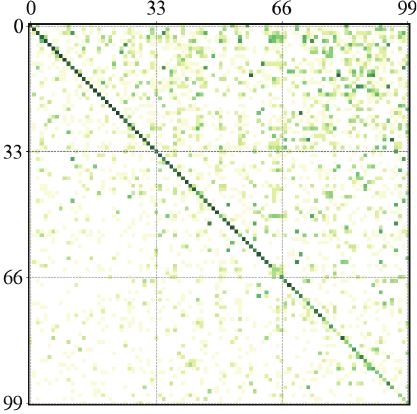} 
		\end{minipage}
		\label{Fig.cifar-100-100-confusematrix(C)}
	}
    	\subfigure[CDT \cite{Corr/CDT}]{
    		\begin{minipage}[b]{0.23\textwidth}
  		 	\includegraphics[width=1\textwidth]{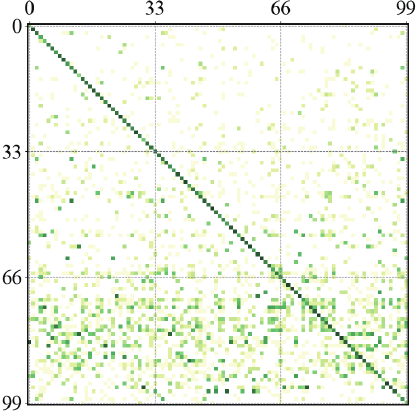}
    		\end{minipage}
		\label{Fig.cifar-100-100-confusematrix(D)}
    }
    
        	\subfigure[Remix \cite{Eccv/remix}]{
    		\begin{minipage}[b]{0.23\textwidth}
		 	\includegraphics[width=1\textwidth]{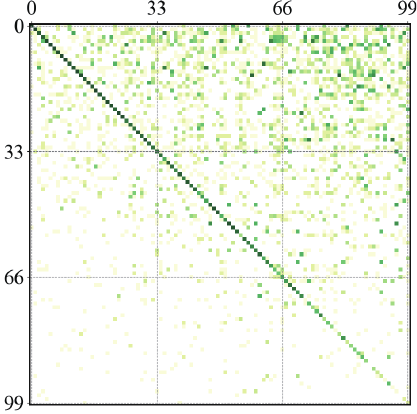}
    		\end{minipage}
		\label{Fig.cifar-100-100-confusematrix(E)}
    }
            	\subfigure[CB+DRW \cite{Cvpr/CB}]{
    		\begin{minipage}[b]{0.23\textwidth}
		 	\includegraphics[width=1\textwidth]{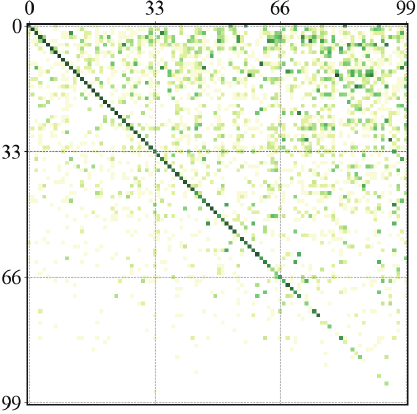}
    		\end{minipage}
		\label{Fig.cifar-100-100-confusematrix(F)}
    }
            	\subfigure[Logit Adjustment \cite{Aaai/logit_adjustment}]{
    		\begin{minipage}[b]{0.23\textwidth}
		 	\includegraphics[width=1\textwidth]{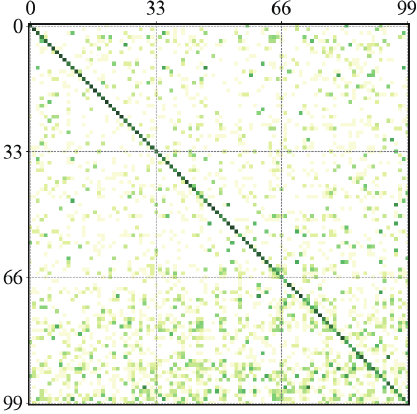}
    		\end{minipage}
		\label{Fig.cifar-100-100-confusematrix(G)}
    }
            	\subfigure[Ours]{
    		\begin{minipage}[b]{0.23\textwidth}
		 	\includegraphics[width=1\textwidth]{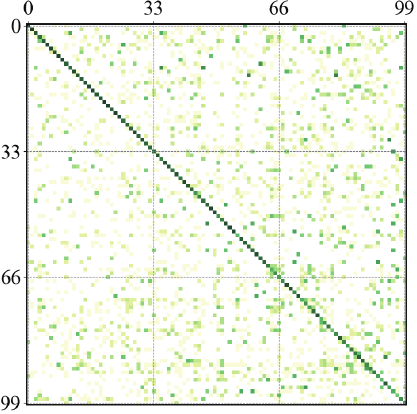}
    		\end{minipage}
		\label{Fig.cifar-100-100-confusematrix(H)}
    }
\topcaption{Additional $\log$-confusion matrix comparisons on CIFAR-100-LT-100. The $x$-axis and $y$-axis indicate the ground truth and predicted labels, respectively. $\log$ operation is adopted for clearer comparisons. Deeper color indicates larger values. Other methods show obviously poor generalization on the tail in the challenging CIFAR-100-LT-100, while the proposed method achieves significantly superior correct classification results and balanced misclassification case distribution.}
\label{Fig.addcifar-100-100-confusematrix}
% \vspace{-0.5cm}
\end{figure}

% \subsection{Results on test imbalance scenarios}
\subsection{Results on test imbalance scenarios}
Zhang et al. \cite{zhang2021test} have discussed the existing challenging and severe bias issue in imbalanced distributed test set. When the test set distribution is ideally known or well estimated, Bayias enables to overcome the bias issues in such scenarios by considering the priors. In contrast, CE loss implicitly requires train set and test set same distributed, while LA loss requires test set balanced distributed. LA and CE losses are powerless in other scenarios, while Bayias is flexible and capable with imbalanced distributed test sets, i.e., the constant term $\log C$ can change to be the test set distribution $\log \pi '$ to deal with such scenarios.

We sample test sets under different distribution with test imbalance factor $\rho'$ (i.e., long-tail ($\rho'>1$), balanced ($\rho'=1$), and reversed long-tail ($\rho'<1$)), corresponding to different imbalance factor $100,10,1,0.1,0.01$. The following comparisons in Tab.\ref{Tab.testimb} show the necessity of considering test distribution's influence, which reveals our motivation to deal with bias issues. When train and test sets follow the same distribution, Bayias is equal to CE and outperforms LA. When the distribution of train and test set is different severely (e.g., $\rho'=100$ and $0.01$), Bayias outperforms both CE and LA. Similar experiments and conclusions can be obtained from \cite{LADE} as well.

\begin{table*}[h!]
\setlength{\tabcolsep}{10pt}
% \vspace{-11pt}
% Besides baseline (ERM) and our methods, we divide all approaches into feature-wise and loss-wise groups.
\topcaption{Comparisons of CE, LA and Bayias loss on CIFAR-100-LT class-imbalanced test set. Note that the imbalanced test set is sampled by discarding some data. Therefore, the instance numbers of test samples vary in different imbalance factors, and longitudinal comparisons can reflect the effectiveness of Bayias. }
\resizebox{\textwidth}{!}{%
\begin{tabular}{l|ccccc|ccccc}
\toprule
Train set             & \multicolumn{5}{c|}{\textbf{CIFAR-100-LT-10}} & \multicolumn{5}{c}{\textbf{CIFAR-100-LT-100}} \\ \midrule
$\rho'$ & 100     & 10      & 1       & 0.1    & 0.01   & 100     & 10      & 1       & 0.1    & 0.01   \\ \midrule
CE                    & 69.73   & 63.67   & 55.70   & 48.48  & 43.13  & 63.54   & 52.74   & 38.32   & 25.00  & 14.03  \\
LA                    & 61.94   & 60.14   & 59.87   & 56.12  & 54.06  & 59.13   & 51.06   & 43.89   & 31.14  & 26.42  \\
Bayias                & \textbf{71.89} & \textbf{63.36} & \textbf{60.02} & \textbf{58.20} & \textbf{62.32} & \textbf{64.90} & \textbf{54.15} & \textbf{43.92} & \textbf{36.20} & \textbf{35.42} \\ \bottomrule
\end{tabular}%
}
\label{Tab.testimb}
\end{table*}

\subsection{{Comparisons with the two-stage method}} \label{Apdx.cmpMiSLAS}

We propose UniMix and Bayias to improve model \textit{calibration} by tackling the bias issues caused by imbalanced distributed train and test set. We improve the model calibration and accuracy simultaneously in an end-to-end manner. However, there are some methods to ameliorate calibration in post-hoc (e.g., temperature scaling \cite{Icml/Calibration-NN}) or two-stage (e.g., label aware smoothing \cite{DBLP:conf/cvpr/MiSLAS}) way. As we discussed above, such pipelines will be effective and we integrate our proposed methods with one of them for further comparison. The compared results are illustrated in Tab.\ref{Tab.cmpmislas}.

\begin{table*}[h!]
\setlength{\tabcolsep}{10pt}
\topcaption{Comparisons with state-of-the-art two-stage method on CIFAR-100-LT.}
\resizebox{\textwidth}{!}{%
\begin{tabular}{l|l|cc|cc|cc}
\toprule
\multicolumn{2}{l|}{Dataset}        & \multicolumn{2}{c|}{CIFAR-100-LT-10} & \multicolumn{2}{c|}{CIFAR-100-LT-50} & \multicolumn{2}{c}{CIFAR-100-LT-100} \\ \midrule
\multicolumn{2}{l|}{Metric (\%)}    & Accuracy          & ECE              & Accuracy          & ECE              & Accuracy          & ECE              \\ \midrule
\multirow{2}{*}{MiSLAS} & 1st-stage & 58.7              & 3.91             & 45.6              & 6.00             & 40.3              & 10.77            \\
                        & 2nd-stage & \textbf{63.2}     & \textbf{1.73}    & 52.3              & \textbf{2.25}    & 47.0              & 4.83             \\ \midrule
\multirow{2}{*}{Ours}   & 1st-stage & 61.1              & 7.79             & 51.8              & 4.91             & 47.6              & 3.24             \\
                        & 2nd-stage & 63.0              & 1.98             & \textbf{52.6}     & 3.45             & \textbf{48.3}     & \textbf{1.44}    \\ \bottomrule
\end{tabular}%
}
\label{Tab.cmpmislas}
\end{table*}

MiSLAS adopts \textit{mixup} for the 1st-stage and label aware smoothing (LAS) for 2nd-stage. We add the LAS in additional $10$ epochs following our vanilla pipeline to build a two-stage manner. Our method generally outperforms MiSLAS in the 1st-stage and can achieve on par with accuracy and lower ECE. In specific, MiSLAS mainly improves accuracy and ECE in the 2nd-stage, i.e., classifier learning, by a large margin, which indicates LAS's effectiveness. We adopt the LAS to replace Bayias compensated CE loss in the 2nd-stage and can further improve accuracy and ECE, especially in the more challenging CIFAR-100-LT-100. However, the improvement of the 2nd-stage is not significant compared to MiSLAS. Since the LAS and Bayias compensated CE loss are both modifications on standard CE loss, it is hard to implement both of them simultaneously. We suggest that the LAS plays similar roles as Bayias compensated CE loss does, which explains the limited performance improvement.

\setcounter{equation}{0}
\setcounter{table}{0} 
\setcounter{figure}{0}
\setcounter{corollary}{0}
\renewcommand{\thetable}{E\arabic{table}}
\renewcommand{\thefigure}{E\arabic{figure}}
\renewcommand{\theequation}{E.\arabic{equation}}

\section{Additional discussion}
\textbf{What is the difference between UniMix and traditional over-sample approaches?} We take representative over-sample approach SMOTE \cite{Jair/SMOTE} for comparisons. SMOTE is a classic over-sample method for imbalanced data learning, which constructs synthetic data via interpolation among a sample's neighbors in the same class. However, the authors in \cite{Nips/On_Mixup_Training} illustrate that such interpolation without labels is negative for classification calibration. In contrast, \textit{mixup} manners take elaborate-designed label mix strategies. The success of label smoothing trick \cite{Cvpr/Bag-tricks-cls} and \textit{mixup}s imply that the fusion of label may play an important role in classification accuracy and calibration. Different from other methods, our UniMix pays more attention to the tail when mixing the label and thus achieves remarkable performance gain.

\textbf{Why choose $\bm{\xi^*_{i,j} \sim \mathscr{U}(\pi_{y_i},\pi_{y_j},\alpha, \alpha)}$?} As we have discussed the motivation of UniMix Factor before, it is intuitive to set $\xi_{i,j} = \pi_{y_j} / (\pi_{y_i} + \pi_{y_j})$. Although satisfactory performance gain is obtained on CIFAR-LT in this manner, we fail on large-scaled dataset like ImageNet-LT. We suspect that the $\xi^*_{i,j}$ will be close to $0$ or $1$ when the datasets become extremely imbalanced and hence the effect of our mixing manner disappears. To improve the robustness and generalization, we transform the origin $Beta(\alpha, \alpha)$ distribution ($\alpha \leq 1$) to maximize the probability of $\xi_{i,j} = \pi_{y_j} / (\pi_{y_i} + \pi_{y_j})$ and its vicinity. We set $\alpha=0.5$ on CIFAR-LT and keep the same $\alpha$ as \textit{mixup} and Remix on ImageNet-LT and iNaturalist 2018.

\textbf{Why does Bayias work?} According to Bayesian theory, the estimated \textit{likelihood} is positive correlation to \textit{posterior} both in balance-distributed train and test set, with the same constant coefficient of \textit{prior} and evidence factor. However, though evidence factor is regraded as a constant in LT scenarios, the LT dataset suffers seriously skewed \textit{prior} of each category in the train set, which is extremely different from the balanced test set. Hence, the model based on maximizing \textit{posterior} probability needs the to compensate the different priors firstly. From the perspective of optimization, Deep Neural Network (DNN) is a non-convex model, minuscule bias will eventually cause disturbances in the optimization direction, resulting in parameters converging to another local optimal. In addition, $\log(\pi_i) + \log(C)$ is consistent with traditional balanced datasets (i.e., $\pi_i \equiv 1/C$), which turns to be a special case (i.e., $\log(1/C)+\log(C) \equiv 0$) of Bayias. Furthermore, our Bayias can compensate any discrepancy between train set and test set via directly setting the Bayias as $\log(\pi_{y_i}) - \log(\pi'_{y_i})$, where $\pi_{y_i}$ is the label prior in train set and $\pi'_{y_i}$ is the label prior in test set.

\end{document}